\PassOptionsToPackage{prologue,dvipsnames}{xcolor}
\documentclass{article}

\usepackage{microtype}
\usepackage{graphicx}
\usepackage{subfigure}
\usepackage{booktabs} 

\usepackage{hyperref}

\usepackage[accepted]{icml2024}

\usepackage{amsmath}
\usepackage{amssymb}
\usepackage{mathtools}
\usepackage{amsthm}
\usepackage{enumitem}
\setlist{nolistsep}

\usepackage{multirow}
\usepackage{pifont}
\newcommand*\colourcheck[1]{%
  \expandafter\newcommand\csname #1check\endcsname{\textcolor{#1}{\ding{52}}}%
}
\newcommand*\colourlargex[1]{%
  \expandafter\newcommand\csname #1largex\endcsname{\textcolor{#1}{\ding{54}}}%
}
\colourcheck{Green}
\colourlargex{Red}
\newcommand*{\xdash}[1][3em]{\rule[0.5ex]{#1}{0.55pt}}

\usepackage[capitalize,noabbrev]{cleveref}

\theoremstyle{plain}
\newtheorem{theorem}{Theorem}[section]

\newtheorem{lemma}[theorem]{Lemma}
\newtheorem{corollary}[theorem]{Corollary}
\theoremstyle{definition}
\newtheorem{definition}[theorem]{Definition}

\theoremstyle{remark}

\newif\ifverbose
\verbosefalse

\usepackage[textsize=tiny]{todonotes}

\DeclareMathOperator*{\argmin}{arg\,min}

\icmltitlerunning{Expanding ANN-to-SNN Conversion Beyond ReLU Network}

\begin{document}

\twocolumn[
\icmltitle{Sign Gradient Descent-based Neuronal Dynamics: \\ANN-to-SNN Conversion Beyond ReLU Network}



\icmlsetsymbol{equal}{*}

\begin{icmlauthorlist}
\icmlauthor{Hyunseok Oh}{snu}
\icmlauthor{Youngki Lee}{snu}
\end{icmlauthorlist}

\icmlaffiliation{snu}{Department of Computer Science \& Engineering, Seoul National University, Seoul, Republic of Korea}

\icmlcorrespondingauthor{Youngki Lee}{youngki.lee@gmail.com}

\icmlkeywords{Spiking neural network, ANN-to-SNN conversion, Sign descent method, Bio-inspired AI, Convex optimization}

\vskip 0.3in
]



\printAffiliationsAndNotice{}  

\begin{abstract}
Spiking neural network (SNN) is studied in multidisciplinary domains to (i) enable order-of-magnitudes energy-efficient AI inference and (ii) computationally simulate neuroscientific mechanisms.
The lack of discrete theory obstructs the practical application of SNN by limiting its performance and nonlinearity support. 
We present a new optimization-theoretic perspective of the discrete dynamics of spiking neurons. 
We prove that a discrete dynamical system of simple integrate-and-fire models approximates the subgradient method over unconstrained optimization problems. 
We practically extend our theory to introduce a novel sign gradient descent (signGD)-based neuronal dynamics that can (i) approximate diverse nonlinearities beyond ReLU and (ii) advance ANN-to-SNN conversion performance in low time steps.
Experiments on large-scale datasets show that our technique achieves (i) state-of-the-art performance in ANN-to-SNN conversion and (ii) is the first to convert new DNN architectures, e.g., ConvNext, MLP-Mixer, and ResMLP. We publicly share our source code at  \href{https://github.com/snuhcs/snn_signgd}{www.github.com/snuhcs/snn\_signgd}~.
\end{abstract}

\begin{figure}[ht]
    \begin{center} 
     \subfigure[Conceptual diagram of our technical contributions.]{\includegraphics[width=0.99\columnwidth]{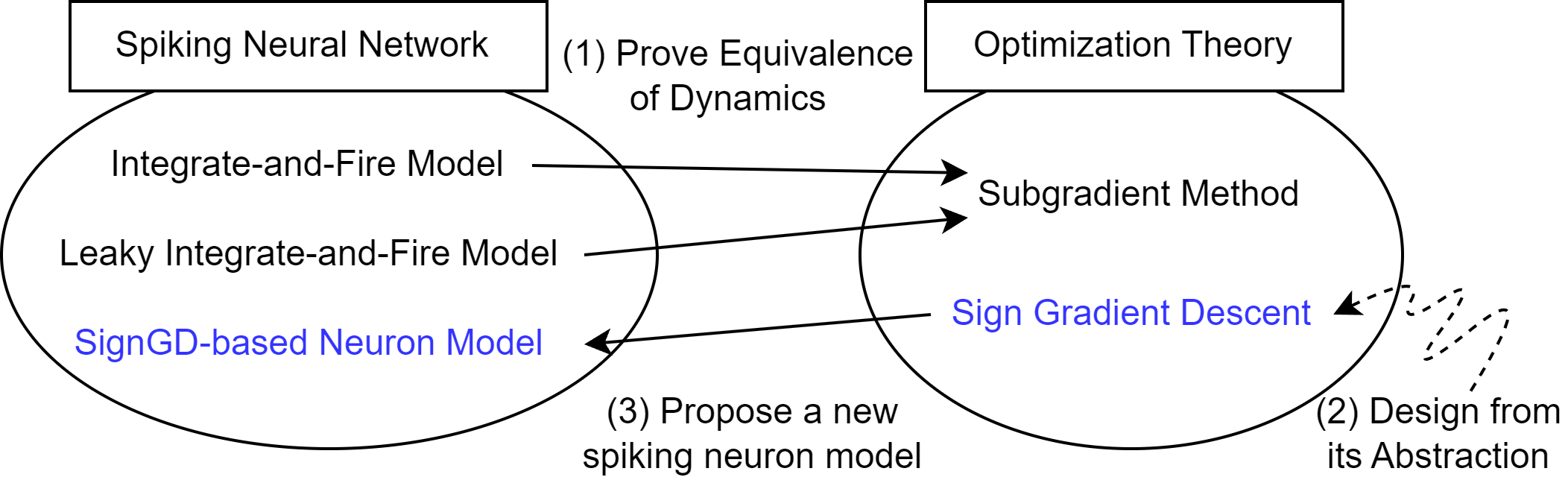}
     \label{fig:contributions}
     }     
     \vfill
     \subfigure[Our signGD-based spiking neuron enables the conversion of neural networks that use nonlinear operators other than ReLU.]{\includegraphics[width=0.99\columnwidth]{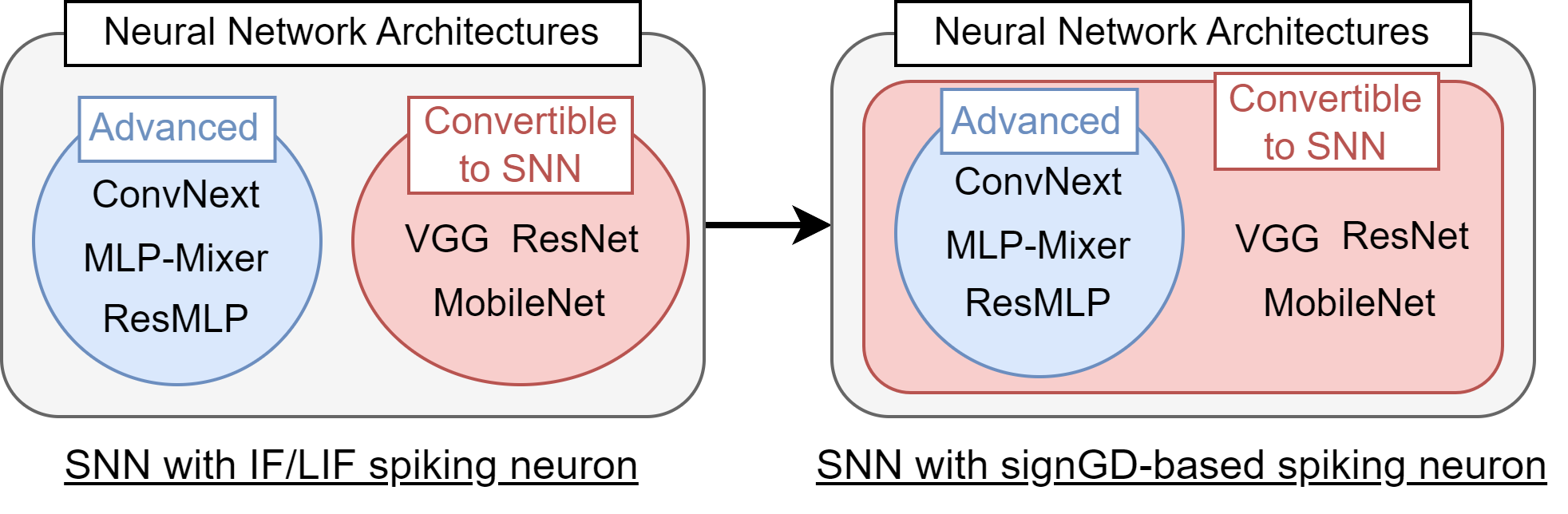}
     \label{fig:practical_contribution}
     }
    \end{center}
    \vskip -0.15in
    \caption{In this paper, we (i) mathematically connect the neuronal dynamics of integrate-and-fire models with the optimization dynamics of subgradient method, (ii) extend the theory to design a new spiking neuron model that can approximate arbitrary element-wise tensor operators, and (iii) use our neuron model to expand ANN-to-SNN conversion beyond ReLU networks (Fig.~\ref{fig:practical_contribution}).}
    \label{fig:teaser}
    \vskip -0.3in
\end{figure}
\section{Introduction}
Understanding how a biological neuron processes information has been a milestone of both neuroscience and efficient artificial intelligence~\cite{Christensen20212022roadmapneuromorphic, spinnaker, kasabov2014neucube}. Spiking neural network (SNN) is widely studied to get new insights on the information dynamics of the brain~\cite{ghosh2009snn,spinnaker}. SNN is a biologically plausible type of artificial neural network (ANN) in which its neuron models closely mimic neuroscientific mechanisms~\cite{Schuman2022OpportunitiesNeuromorphic}. 
In detail, the spiking neuron, the central processing unit of SNN, (i) processes non-linearity with an internal dynamical system and (ii) communicates the information in the form of a spike train, a time series of short binary electrical pulses.  
These key characteristics of SNN are practically leveraged to develop extremely efficient AI algorithms~\cite{Schuman2022OpportunitiesNeuromorphic}. For example, SNN inference is orders-of-magnitude more energy-efficient than the same-architecture DNNs, e.g., 35-560$\times$ less on VGG~\cite{bu2021optimalannsnn} and 280$\times$ less on YOLO~\cite{kim2020spikingyolo}. 
Discretization of spiking neuronal dynamics is thus pivotal to both (i) practically implement SNN for real-world applications and (ii) simulate brain mechanisms with SNN.
However, the theoretical understanding of SNN remains unclear~\cite{zhang2022theoreticallyprovable}, especially regarding its discrete neuronal dynamics.

The lack of discrete theory obstructs the practical application of SNN by constraining its inference accuracy and nonlinearity support.  
Building a high-performance SNN is categorized into two groups: (i) training SNN from scratch (Reviewed in Appendix~\ref{app:snn_training}.) and (ii) converting a pre-trained ANN model into SNN~\cite{cao2015spikingconv}. 
A successful SNN training strategy, surrogate gradient methods~\cite{Neftci2019SurrogateGL}, 
consumes immense computational resources since they unfold a DNN backwards the entire time-steps~\cite{li2021freelunch}. Yet, their best accuracy still falls behind DNNs with similar architecture, e.g., 4-6\% on ResNet~\cite{fang2021deepresidual} and 6-10\% on ViT~\cite{zhou2022spikformer}. 
On the other hand, ANN-to-SNN conversion techniques~\cite{cao2015spikingconv,han2020deeptsc} replace real-valued nonlinear operators of ANN with spiking neurons. However, the only known theoretical correspondence is between the clipped ReLU function and IF neuron~\cite{Rueckauer2017ConversionOC}. This leads to three major limitations in prior conversion approaches. First, ReLU should be the only nonlinear operator in a target ANN. Second, an accurate SNN inference mandates data-dependent normalization or calibration techniques~\cite{li2021freelunch}. Finally, low-latency techniques noticeably degrade the best accuracy of SNN by replacing ReLU with spike-aware  functions~\cite{bu2021optimalannsnn,Jiang2023SlipReLU}. 

In this light, we provide a new optimization-theoretic perspective of the discrete neuronal dynamics that can (i) explain 
the underlying principle of neuronal dynamics and (ii) extend to design a new spiking neuron that can compute various nonlinearities. 
We first prove that a discrete dynamical system of simple integrate-and-fire models is equivalent to a subgradient method over an unconstrained optimization problem with its solution as a spike-coded nonlinear function value. 
SNN inference is thus a neuron-wise first-order optimization process to approximate real-valued activations. 
This framework provides a way to study discrete neuronal dynamics by expressing it as an equivalent optimization algorithm, i.e., its optimizer form, and performing convergence analysis to derive the neuron's asymptotic behavior. 

Practically extending our theory, we present a novel signGD-based neuronal dynamics that can (i) approximate diverse nonlinearities beyond ReLU and (ii) achieve high accuracy in low time steps with converted SNNs. 
Specifically, we choose the sign gradient descent algorithm (signGD)~\cite{bernstein2018signsgd} as the optimizer form of our neuron, replacing the subgradient method. 
A spiking neuron's key characteristic, spike-based communication, constrains the space of approximable nonlinear function for the subgradient-based neuronal dynamics. 
In contrast, in the case of signGD-based neuronal dynamics, a binary spike conveys only the sign of the gradient of the objective function, widening the space of approximable nonlinear functions.
We generalize the learning rate schedule of the signGD-based optimizer form to formulate the new neuronal dynamics and the neural coding scheme.
We empirically validate with experiments that our signGD-based neuron can approximate unary nonlinearities, e.g., ReLU, LeakyReLU, and GELU~\cite{hendrycks2016gelu}, and n-ary nonlinearities, e.g., max pooling and layer normalization~\cite{ba2016layernorm}. 

To empirically verify the effectiveness of our signGD-based neuronal dynamics, we convert high-performance DNNs to SNNs 
with our proposed neurons and evaluate their performance. Experimental results on large-scale ImageNet~\cite{deng2009imagenet} and CIFAR~\cite{krizhevsky2009cifar} datasets show that our technique is (i) state-of-the-art in conversion techniques by precisely approximating the ANN performance in $\le 64$ time-steps, and (ii) first to convert complex DNN architectures, e.g., ConvNext~\cite{liu2022convnext}, MLP-Mixer~\cite{tolstikhin2021mlpmixer}, and ResMLP~\cite{Touvron2021ResMLP}. 
With our neuron, ImageNet top-1 accuracies of converted VGG16 and ResNet34 are  $>75\%$ in $T=64$, outperforming runner-ups by $3\%$.  ConvNext-B and RegNetX-3.2F reach $\approx 81\%$ in $T=256$ for the first time.

\section{Related Works}

\textbf{ANN-to-SNN Conversion.} 
Prior conversion techniques substitute the ReLU function of ANN with the IF neuron~\cite{roy2019towardsspikebasedmachine, Rueckauer2017ConversionOC}. 
Its limitation is that converted SNNs require a huge number of time steps for high accuracy~\cite{han2020deeptsc}, leading to large latency and energy consumption~\cite{liu2022spikeconverter}. Hence, subsequent works sought to \emph{accelerate} SNN inference, i.e., achieve higher accuracy in lower time steps. 
Data-dependent normalization~\cite{diehl2015thresholdbalancing,Rueckauer2017ConversionOC,Wang2022SignedNW}, calibration~\cite{li2021freelunch}, or neuron adaptation techniques~\cite{Hao2023ResidualMembranePotential} minimize the layer-wise empirical error between ANN and SNN activations~\cite{Deng2021OptimalConversionTheory}. 
Temporal coding-based techniques~\cite{Park2020TTFS,han2020deeptsc} embed information in latency of few spikes. 
Studies in SNN-aware nonlinearities substitute ReLU of an ANN architecture with piecewise-continuous functions, e.g., QCFS~\cite{bu2021optimalannsnn}, StepReLU~\cite{Wang2023StepReLU},  SlipReLU~\cite{Jiang2023SlipReLU}. 
Unlike these works, we theoretically explain integrate-and-fire models beyond IF neurons, propose novel neuronal dynamics to support diverse nonlinearities other than ReLU, and achieve the highest inference accuracy with our converted SNNs.

A few prior works customize neuronal dynamics to accelerate SNN or approximate different nonlinearities. \cite{liu2022spikeconverter} replaces event-driven computation with layer-wise computation. Instead of binary spikes, ternary spikes of $\{-1,0,1\}$ are used to approximate LeakyReLU~\cite{kim2020spikingyolo} or accelerate SNN~\cite{Wang2022SignedNW}. 
Time series of float instead of spikes are used to support max pooling~\cite{li2022quantizationframework} or accelerate SNN~\cite{Jiang2023SlipReLU}. Overall, these works sacrifice key characteristics of SNN, e.g., event-driven computation or binary spike train. In contrast, our signGD-based neuron computes in an event-driven manner with a binary spike train to support diverse nonlinearities.

\ifverbose{
\subsubsection{references}

(Spikformer)~\cite{zhou2022spikformer} ANN-to-SNN conversion and direct training. 
~\cite{hunsberger2015spikinglifconvert} There has recently been considerable effort to take deep ANNs and make them more biologically
plausible by introducing neural “spiking” [8, 9, 10, 11, 12, 13], such that connected nodes in the
network transmits information via instantaneous single bits (spikes), rather than transmitting realvalued activities. While one goal of this work is to better understand the brain by trying to reverse
engineer it [9], another goal is to build energy-efficient neuromorphic systems that use a similar
communication method for image categorization [12, 13].

(Deep Residual Learning in Spiking Neural Networks
)~\cite{fang2021deepresidual} ANN to SNN conversion (ANN2SNN) [20, 4, 46, 49, 12, 11, 6, 54, 33] and backpropagation with
surrogate gradient [40] are the two main methods to get deep SNNs. The ANN2SNN method first
trains an ANN with ReLU activation, then converts the ANN to an SNN by replacing ReLU with
spiking neurons and adding scaling operations like weight normalization and threshold balancing.

~\cite{zhang2022recentadvances}
The basic idea of DNNs-converted SNNs is that the average firing rate under rate encoding in SNNs can approximate the continuous activation value under the ReLU activation function in DNNs. In terms of performance, the DNNs-converted SNNs maintain the smallest gap with DNNs and can be implemented on large-scale network structures and datasets.

In ANNto-SNN conversion ~\cite{cao2015spikingconv}(Cao et al., 2015; Hunsberger & Eliasmith, 2015~\cite{hunsberger2015spikinglifconvert}; Rueckauer et al., 2017; Bu
et al., 2021~\cite{bu2021optimalannsnn}; Meng et al., 2022; Wang et al., 2022~\cite{Wang2022SignedNW}), the high-performance pre-trained ANN is
converted to SNN by replacing the ReLU activation layers with spiking neurons. 

(Towards spike-based machine intelligence with neuromorphic computing)~\cite{roy2019towardsspikebasedmachine} The idea of a conversion-based approach is to obtain an SNN that yields the same input–output mapping for a given task as that of a DLN.

~\cite{Rueckauer2017ConversionOC}A more straightforward approach is to take the parameters
of a pre-trained ANN and to map them to an equivalently accurate SNN. Early studies on ANN-to-SNN conversion began with the work of Perez-Carrasco et al. (2013), where CNN units were translated into biologically inspired spiking units with leaks and refractory periods, aiming for processing inputs from event-based sensors. 
Cao et al. (2015) suggested a close link between the transfer function of a spiking neuron, i.e., the relation between input current and output firing frequency to the activation of a rectified linear unit (ReLU), which is nowadays the standard model for the neurons in ANNs.
They report good performance error rates on conventional computer vision benchmarks, converting a class of CNNs that was restricted to having zero bias and only average-pooling layers. 
Their method was improved by Diehl et al. (2015), who achieved nearly loss-less conversion of ANNs for the MNIST (LeCun et al., 1998) classification task by using a weight normalization scheme. This technique rescales the weights to avoid approximation errors in SNNs due to either excessive or too little firing of the neurons. Hunsberger and Eliasmith (2016) introduced a conversion method where noise injection during training improves the robustness to approximation errors of the SNN with more realistic biological neuron models. Esser et al. (2016) demonstrated an approach that optimized CNNs for the TrueNorth platform which has binary weights and restricted connectivity. Zambrano and Bohte (2016) have developed a conversion method using spiking neurons that adapt their firing threshold to reduce the number of spikes needed to encode information.
These approaches achieve very good results on MNIST, but the SNN results are below state-of-the-art ANN results when scaling up to networks that can solve CIFAR-10 (Krizhevsky, 2009). One reason is that SNN implementations of many operators that are
crucial for improved ANN error rate, such as max-pooling layers, softmax activation functions, and batch-normalization, are nonexistent,
and thus SNNs can only approximately match the inference of an ANN. As a consequence, none of the previously proposed conversion approaches are general enough for full automatic conversion of arbitrary pre-trained ANNs taken from a Deep-Learning model zoo available, for example, in Caffe1.

~\cite{Rueckauer2017ConversionOC}In this work, we address some important shortcomings of existing ANN-to-SNN conversion methods. Through
mathematical analysis of the approximation of the output firing rate of a spiking neuron to the equivalent analog activation value, we were able to derive a theoretical measure of the error introduced in the previous conversion process. On the basis of this novel theory, we propose modifications to the spiking neuron model that significantly improve the performance of deep SNNs. By developing spiking implementations of max pooling layers, softmax activation, neuron biases, and batch normalization (Ioffe and Szegedy, 2015), we extend the suite of CNNs that can be converted. In particular, we demonstrate for the first time that GoogLeNet Inception-V3 can be converted to an equivalent-accurate SNN. Further, we show that the conversion to spiking networks is synergistic with ANN network compression techniques such as parameter quantization and the use of low-precision activations.

(A free lunch)~\cite{li2021freelunch} Unlike training from scratch, ANN-to-SNN conversion
methods, such as data-based normalization (Diehl et al.,
2015; Rueckauer et al., 2016) or threshold balancing (Diehl
et al., 2015; 2016), adapt to more complex situations (Tavanaei et al., 2019). The major bottleneck of these methods
is how to balance accuracy and inference latency as they
require more than 2k time steps to get accurate results. 

~\cite{fang2021deepresidual} Some recent conversion methods have achieved near loss-less accuracy with VGG-16 and ResNet [12,11, 6, 33].

The converted
SNN requires large time steps to accurately approximate ReLU activation, which causes large latency
(Han et al., 2020)~\cite{Han2020RMPSNNRM}. 

 ~\cite{Rueckauer2017ConversionOC} However, the converted SNN needs a longer time to rival the original ANN in precision
as the conversion is based on rate-coding [46], which increases the SNN’s latency and restricts the
practical application. 
Previous ANN2SNN methods noticed the distinction between plain feedforward ANNs and residual
ANNs, and made specific normalization for conversion. Hu et al. [17] were the first to apply the
residual structure in ANN2SNN with scaled shortcuts in SNN to match the activations of the original
ANN. Sengupta et al. [49] proposed Spike-Norm to balance SNN’s threshold and verified their
method by converting VGG and ResNet to SNNs

~\cite{zhang2022recentadvances}
Moreover, existing DNNs-converted SNNs algorithms also suffer from long simulation periods. From the perspective of model compression, the conversion process is an extreme quantization of activation values. The binary neural networks (BNNs) [Rastegari et al., 2016; Chen et al., 2018] in DNNs have a similar concept. However, the connection and difference between the BNNs and SNNs, and the possible impact of the additional temporal dimension in SNNs are not clearly elaborated. The threshold fring properties of SNNs may make them more receptive to compression algorithms. Therefore, the combination with compression algorithms such as weight quantization and pruning also needs to be explored so that the computational effciency advantage of SNNs can be further developed [Chen et al., 2021].

 As a consequence, some neurons are difficult to transmit information with short simulation sequences. Thus the converted SNNs usually require huge simulation length to achieve high accuracy (Deng et al., 2020) due to the trade-off between simulation length and accuracy (Rueckauer et al., 2017). This dilemma can be partially relieved by applying the threshold balance on the channel level Kim et al. (2019) and adjusting threshold values according to the input and output frequencies (Han et al., 2020; Han & Roy, 2020). However, as far as we know, it remains unclear how the gap between ANN and SNN formulates and how the simulation length and voltage threshold affect the conversion loss from layers to the whole network. In addition, the accuracy of converted SNN is not satisfactory when the simulation length is as short as tens.
The performance of the converted SNN is determined by the source ANN performance and the conversion error. we analyze that minimizing the conversion error is equivalent to minimizing output errors caused by different activation functions of the ANN and SNN on each layer. Here we further analyze how to modify the activation function so that the layer-wise error can be minimized.

~\cite{roy2019towardsspikebasedmachine}The output value of a nonlinear neuron—using, for example, a hyperbolic tangent (tanh) or a normalized exponential (softmax) function—can take both positive and negative values, whereas the rate of a spiking neuron can only be positive. Thus, negative values will always be discarded, leading to a decline in accuracy of the converted SNNs. Another problem with conversion is obtaining the optimal firing rate at each layer without any drastic performance loss. The inference time for SNNs that are converted from DLNs turns out be very large (of the order of a few thousand time steps), leading to increased latency as well as degraded energy efficiency.

we consider two extreme cases where (1) the threshold Vth is so large that the simulation time T is not long enough for the neurons to fire a spike or (2) the threshold Vth is so small that the neuron spikes every time and accumulates very large membrane potential after simulation (upper bound of Eqn.7). For these two cases, the remaining potential contains most of the information from ANN and it is almost impossible to convert from the source ANN to the target SNN. To eliminate these two cases, we apply the threshold ReLU instead of the regular ReLU and set the threshold voltage Vth in the SNN as the threshold yth for ReLU in the ANN.
How to minimize the layer-wise squared difference: We can see that when the input is equal, hl and h0l actually have a systematic bias that can be further optimized by either shifting hl or h0l. Thus the total conversion error can be approximately estimated as
 which explicitly indicates how the low threshold value and long simulation time decrease the conversion error. The optimal shift is affected by both the distribution of activation values and the level of overfitting in the source ANN and target SNN.

SpikingYOLO~\cite{kim2020spikingyolo} introduces channel-wise normalization (abbreviated to channel-norm) to enable fast and efficient information transmission in deep SNNs. Our method normalizes the weights by the maximum possible activation (the 99.9th percentile) in a channel-wise manner instead of the conventional layer-wise manner. 
Using signed neurons with IBT, leakyReLU can be implemented accurately in SNNs and can directly be mapped to the current neuromorphic architecture with minimum overhead. over 2,000 times better than Tiny-YOLO for 32-bit FL and INT operations.
SpikingYOLO consumes approximately 280 times less energy than Tiny YOLO when ran on TrueNorth. 

TSC~\cite{han2020deeptsc}: It uses a novel time-based coding scheme (TSC) and TSC spiking neuron model. We also propose a threshold balancing technique which alleviates the ANN-SNN conversion accuracy loss and significantly improved the latency and scalability of TSC-SNNs to deep architectures. Input and output of the TSC spiking neurons are encoded temporally using sparse spiking events over certain time period. Hence, inference in TSC-SNN is carried out over multiple feed-forward passes or time-steps (also known as inference latency), where each pass entails sparse spike-based computations.

~\cite{li2021freelunch}
Rueckauer et al. (2017)
suggest using percentile threshold, which avoids picking the
outlier in the activation distribution. Spike-Norm (Sengupta
et al., 2018) tests architectures like VGG-16 and ResNet20. In this work, we further extend the source architecture
to MobileNet and RegNet. RMP (Han et al., 2020) and
TSC (Han & Roy, 2020) achieves near-to-origin accuracy
by adjusting the threshold according to the input and output
spike frequency. Deng & Gu (2021) decompose the conversion loss into each layer and reduce it via shifting bias. Low
latency converted SNN is an on-going research challenge
since it still requires a considerable amount of simulation
length. At the same time, most SNN conversion work does
not address the BN layers in low latency settings.
All of them fail to convert ANN with BN layers in low latency time steps (≤ 256), which may significantly increase the latency especially for resource-limited devices. We think the simple copy-paste of parameters without any dedicated calibration on SNN will inevitably result in activation mismatch. In this work, we aim to obtain an SNN in extremely low latency (less than 256 time steps) and in extremely low cost. We choose to utilize a pre-trained ANN and convert it to SNN. Unlike previous conversion work which simply transplants the weights parameters to the SNN, in this work we show that the activation transplating is much more important. In order to accomplish this, we propose SNN calibration, a new technology family by calibrating the parameters in SNN to match the activation after conversion, and thus significantly narrow the gap in activation distribution between the source ANN and calibrated SNN.
we propose SNN calibration, a new technology family by calibrating the parameters in SNN to match the activation after conversion, and thus significantly narrow the gap in activation distribution between the source ANN and calibrated SNN.
We first use the derivation in Deng & Gu (2021) to deduce the relationship between ¯s(`) and ¯s(`+1). 
The conversion loss comes from two aspects, namely the flooring error and the clipping error. 
we use Minimization of Mean Squared Error (MMSE) to obtain the threshold V(`)th under different simulation length T. 
Besides adaptive threshold, we further reduce the conversion error by calibrating the parameters of SNN. We first analyze how conversion errors accumulate through layers, and then present a set of layer-wise algorithms to calibrate different types of SNN parameters, including bias, weights and initial membrane potential. We introduce Light Pipeline and Advanced Pipeline, which can be chose by users according to their memory and computation budgets for layer-wise calibration in practice. The light pipeline achieves fast calibration with less memory and computation consumption by only adjusting bias in SNNs. With a little effort, the light pipeline can outperform stateof-the-art methods by a large margin. We also propose an Advanced Pipeline that achieves best results by calibrating the weights as well as the initial membrane potential in a fine-grained way.
To make SNN run on corresponding hardware, we convert the AvgPool layer by treating the AvgPool layer as a convolutional layer with specific values. There is no corresponding module in SNN for Batch Normalization (BN) layers. Rueckauer et al. (2017) propose to
absorb the BN parameters to the weight and bias. To correct the bias and membrane potential, we sample one batch of unlabeled data
(128 training images). 

SpikeConverter~\cite{liu2022spikeconverter}: existing ANN-converted SNN methods are still far from applicable due to the following reasons. 1) Converted SNN still suffers from accuracy drop compared to the source ANN. ANN-converted SNNs typically exploit the firing rate of the spike train to serve as the equivalence of the activation value. However, the firing rate has far worse resolution than the activation values in ANN, leading to accuracy drop. 2) The converted SNNs need an enormous number of time steps for better information representation, which directly deteriorate the energy efficiency of SNNs. Although SNN is supposed to perform the event-driven asynchronous calculation that spikes happen at any time in the time window, the practical hardware operates on a clock-driven synchronous pattern that segment the time window into time steps and process the spikes in batches, resulting in the inference latency and energy consumption directly proportional to the number of time steps. In fact, converted ResNet on ImageNet requires up to 4096 time steps (Han, Srinivasan, and Roy 2020) while the energy consumption of AlexNet-converted SNN with 500 time steps is nearly 5 ∼ 10× higher than that of the source ANN (Singh et al. 2020).
In previous works, the hard reset mechanism is usually used to reset the membrane potential of the neuron to a fixed value when it fires a spike. However, such a method degrades the real-valued information contained in the membrane potential into a boolean value when the voltage exceeds the threshold. To eliminate such information degradation, we adopt the soft reset scheme, which only subtracts the threshold from the membrane potential rather than reset it directly. 
we can build precise relationship between the inputs spike trains, output spike trainand the membrane potential of a leaky integrate and fire neuron. we have the ideal conversion identical relation; an excellent counterpart of the activation value in ANN for both input and output. Such equivalence ensures the homogeneous representation of both input and output, which means that the information can be transmitted totally in the form of spike trains without transforming to other modalities.
The input voltage can be negative due to negative weights, but the output spike will not respond to it until the membrane potential is accumulate to positive values again. Therefore, we propose the temporal separation scheme that separates the neural calculation into two phases. The accumulating phase collects input spikes without firing spikes and the generating phase fires the output spike according to the accumulated membrane potential. To implement temporal separation, we propose the inverse-leaky integrate-and-fire neuron to realize the calculation
and the pipeline mechanism to minimize the delay.  the output spike train will only contain consecutive spikes. If the neuron fails to fire a spike in a certain spike, it will never spike any more since the membrane potential is already below the threshold and nonincreasing. In such context, we propose our inverse-leaky integrateand- fire (iLIF) neuron whose damping coefficient k is larger than 1. 
Inspired by the fact that different processing units are allocated to work on different layers to minimize memory accessing costs, we propose the inter-layer direct delivery and inter-sample pipelining to minimize the time delay of a single sample and maximize the productivity for multiple samples.
While T is usually determined by accuracy and delay requirements, we look for the optimal threshold voltage and k by minimizing the conversion error. We draw the conclusion that k = 2 gives us the optimal conversion.

(Optimal ANN-SNN Conversion )~\cite{bu2021optimalannsnn} To obtain high-performance SNNs with ultra-low latency (e.g., 4 time-steps), we list the critical errors in ANN-SNN conversion and provide solutions for each error. We go deeper into the errors in the ANN-SNN conversion and ascribe them to clipping error, quantization error, and unevenness error. We find that unevenness error, which is caused by the changes in the timing of arrival spikes and has been neglected in previous works, can induce more spikes or fewer spikes as expected. 
We propose the quantization clip-floor-shift activation function to replace the ReLU activation function in source ANNs, which better approximates the activation function of SNNs. We prove that the expected conversion error between SNNs and ANNs is zero, indicating
that we can achieve high-performance converted SNN at ultra-low time-steps.
The key idea of ANN-SNN conversion is to map the activation value of an analog neuron in ANN to the firing rate (or average postsynaptic potential) of a spiking neuron in SNN. 
Clipping error.  Considering that nearly 99.9
Quantization error (flooring error). The output spikes are discrete events, thus discrete with quantization resolution. When mapping, there exists avoidable quantization error.
Unevenness error. (AAAI'22 SpikeConverter paper에도 제시된 에러: negative weight 때문에 membrane potential timing 안좋으면 예상했던거보다 많이 혹은 적게 firing되는거) Unevenness error is caused by the unevenness of input spikes. If the timing of arrival spikes changes, the output firing rates may change, which causes conversion error. There are two situations: more spikes as expected or fewer spikes as expected.  Specifically, the unevenness error will degenerate to the quantization error if vl(T) is in the range of [0, θl]. (개소리)
It is natural to think that if the commonly used ReLU activation function is substituted by a clip-floor function with a given quantization
steps L, the conversion error at time-steps T = L will be eliminated. Thus the performance degradation problem at low latency will be solved. We proposed the quantization clip-floor activation function to train ANNs.
If the time-steps T of the converted SNN is the same as the quantization steps L of the source ANN, the conversion error will be zero. There is no guarantee that the conversion error is zero when T is not equal to L. We propose the quantization clip-floor-shift activation function to train ANNs.  there exists a hyperparameter vector φ that controls the shift of the activation function. 
Training an ANN with quantization clip-floor-shift activation instead of ReLU is also a tough problem. To direct train the ANN, we use the straight-through estimator (Bengio et al., 2013) for thederivative of the floor function.

(Bridging Gaps)~\cite{hao2022bridginggaps} We introduce offset spike to measure the degree of deviation between the actual firing rate and the desired firing rate in SNNs. Then, we demonstrate that the offset spike of being one accounts for the main part in each layer and is the main reason of conversion error.
ANN-SNN conversion errors can be divided into clipping error, quantization error (flooring error), and unevenness error (deviation error) (Bu et al., 2022b). In previous works (Han et al., 2020; Li et al., 2021; Meng et al., 2022b), those errors are eliminated (or reduced) separately, and thus far no method to eliminate the unevenness error (deviation error) has been identified. Since we find that the essential cause of most conversion errors comes from the remaining term. 
We train the source ANN with the QCFS activation function (equation 7) and then convert it to an SNN

\subsubsection{soft reset}
Soft reset~\cite{Han2020RMPSNNRM}

~\cite{li2021freelunch}
Recently, many methods have been proposed to reduce the
conversion loss and simulation length. The soft-reset also
called the reset-by-subtraction mechanism, is the most common technique to address the potential reset’s information
loss (Rueckauer et al., 2016; Han & Roy, 2020). Our IF neuron model also adopts this strategy.
\subsubsection{max pooling}
(Quantization Framework for Fast Spiking Neural Networks )~\cite{li2022quantization} we empirically show that the inference latency of the SNN and the activation bit-width of the ANN are correlated after ANN-to-SNN conversion, so a fast SNN can be built by using a quantized ANN.
Standard quantization methods (post-training quantization and quantization-aware training) have increasingly matured. These two standard methods are not suitable for this research, as they fail to achieve competitive accuracy in extremely low bit precision such as 2 bits. Hence, the first obstacle is to choose a more effective quantization method than the standard post-training quantization and quantization-aware training. 
The ANN quantization technique chosen in this paper is based on LSQ (Esser et al., 2019). LSQ defines the gradients of the quantization step size to prevent activations from being too close to quantization transmission points. It can enable network quantization down to 2 bits while minimizing the accuracy loss introduced by quantization.
Firstly, many quantization techniques including the applied LSQ leave the output layer in floating-point to render better accuracy, e.g., 4
Secondly, the integrate-and-fire mechanism in spiking neurons corresponds to rounding down rather than rounding to nearest which is generally used in ANN quantization. Here we stick to using rounding to nearest during quantization and compensate for it in the SNN by pre-charging the membrane potential (Hwang et al., 2021).
We add a mechanism for generating negative spikes in spiking neurons to compensate for the incorrectly emitted positive spikes: A negative spike will be generated when the membrane potential is smaller than zero and the total spike count generated by this neuron is greater than zero.
Event-Based Max Pooling: a max pooling output spike zt is generated only when Mt changes, which keeps the event-based nature of SNNs. Access to the spike count and the calculation of the difference
are uncomplicated and can be implemented in PyNN

\subsubsection{encoding}
(Recent Advances and New Frontiers in Spiking Neural Networks )~\cite{zhang2022recentadvances} Currently, the specifc method of neural encoding has not yet been concluded. Populations of neurons with different encodings may coexist and cooperate, thus providing a suffcient perception of information. Neural encoding methods may behave differently in different brain regions.  Furthermore, many SNNs algorithms only pay attention to rate coding, ignoring the spike trains’ temporal structure. It may cause that the advantages of SNNs in temporal information processing have not been well exploited. Therefore, the design of the algorithm suitable for temporal coding with high information density may be the new direction for future exploration.

(Theoretically Provable Spiking Neural Networks)~\cite{zhang2022theoreticallyprovable} About Neural Encoding. The actual input data (e.g., image or video) should usually be preconverted into a spiking version before fed up to SNNs. The conversion procedure is called neural
encoding, as shown in Figure 1. There are two main categories of neural encoding approaches: temporal encoding and rate-based encoding; the former encodes input data by exploiting the distance
between time instances that fire spikes [40], and the latter encodes input data as a count sequence
of the fired spikes within temporal windows [18]. The rate-based encoding is the simplest and most
popular scheme in SNNs. The representative techniques are usually encoded by a Poisson distribution or recorded by a dynamic vision sensor [3, 27]. Recent years have witnessed a lot of efforts on
the information capacity of neural encoding, specially rate-based encoding, from empirical [6, 17]
and theoretical [24, 35, 36] sides. Throughout this paper, we adopt rate-based encoding as the default and focus on the firing rates of SNNs, generalizing the computational powers concerning the
spike count.
About Firing Rates. When we investigate the dynamics and neural computation of SNNs, the
firing rates or equally the number of firing spikes are the key measure of network activities for investigating neural computation and model dynamics because of the close relation between firing
rates and network function (including neural input, connectivity, spiking function, and firing process) [1, 2]. There are great efforts to use firing rates in SNNs for some real-world tasks, such as
vision [13, 33, 44] and speech recognition [30, 39]. Besides, Barrett et al. [5] and Chou et al. [8]
showed that the averaged firing rate can approximate the optimal solutions of some quadratic programs within polynomial complexity. This work employs an “instantaneous” firing rate rather than
the averaged firing rate or the total number of firing spikes used in previous studies. This manner
provides a feasible way to construct the discrete dynamical systems using the IFR functions, based
on which we can develop in-depth understandings of SNNs from spatial and temporal aspects.
\fi

\ifverbose{
\subsubsection{references}

(STDP-based spiking deep convolutional neural networks
for object recognition) Primate’s visual system solves the object recognition task through hierarchical processing along the
ventral pathway of the visual cortex [10]. Through
this hierarchy, the visual preference of neurons
gradually increases from oriented bars in primary
visual cortex (V1) to complex objects in inferotemporal cortex (IT), where neural activity provides
a robust, invariant, and linearly-separable object
representation [10, 9]. Despite the extensive feedback connections in the visual cortex, the first feedforward wave of spikes in IT (∼ 100−150 ms poststimulus presentation) appears to be sufficient for
crude object recognition [54, 19, 33].
The computing units of DCNNs send floatingpoint values to each other which correspond to their
activation level, while, biological neurons communicate to each other by sending electrical impulses
(i.e., spikes). The amplitude and duration of all
spikes are almost the same, so they are fully characterized by their emission time. Interestingly, mean
spike rates are very low in the primate visual systems (perhaps only a few of hertz [50]). Hence, neurons appear to fire a spike only when they have to
send an important message, and some information
can be encoded in their spike times. Such spiketime coding leads to a fast and extremely energyefficient neural computation in the brain (the whole
human brain consumes only about 10-20 Watts of
energy [34]).
The current top-performing DCNNs are trained
with the supervised back-propagation algorithm
which has no biological root.
Although it works
well in terms of accuracy, the convergence is rather
slow because of the credit assignment problem [45].
Furthermore, given that DCNNs typically have millions of free parameters, millions of labeled examples are needed to avoid over-fitting. However, primates, especially humans, can learn from far fewer
examples while most of the time no label is available. They may be able to do so thanks to spiketiming-dependent plasticity (STDP), an unsupervised learning mechanism which occurs in mammalian visual cortex [38, 18, 37]. According to
STDP, synapses through which a presynaptic spike
arrived before (respectively after) a postsynaptic
one are reinforced (respectively depressed).
In this paper we proposed a STDP-based spiking deep neural network (SDNN) with a spiketime neural coding. The network is comprised of
a temporal-coding layer followed by a cascade of
consecutive convolutional (feature extractor) and
pooling layers. The first layer converts the input
image into an asynchronous spike train, where the
visual information is encoded in the temporal order of the spikes. Neurons in convolutional layers
integrate input spikes, and emit a spike right after
reaching their threshold. These layers are equipped
with STDP to learn visual features. Pooling layers
provide translation invariance and also compact the
visual information [48]. Through the network, visual features get larger and more complex, where
neurons in the last convolutional layer learn and
detect object prototypes. At the end, a classifier
detects the category of the input image based on
the activity of neurons in the last pooling layer with
global receptive fields.

(Event-driven Backpropagation) The first category consists of recurrent neural network (RNN)-like learning algorithms. These
algorithms treat spiking neural networks as binary-output recurrent neural networks and handle the
discontinuities of membrane potential at spike times with continuous surrogate derivatives [17]. They
typically train deep SNNs with surrogate gradients based on the idea of backpropagation through
time (BPTT) algorithm [18, 19, 20, 21, 22, 23, 24, 25, 26, 27, 28]. While competitive accuracies are
reported on the MNIST, CIFAR-10, and even ImageNet datasets [29, 30, 31], the gradient information
is propagated each time step, whether or not a spike is emitted (as shown in Fig. 1). Therefore,
these approaches do not follow the event-driven nature of spiking neural networks, which lose the
asynchronous characteristic of SNNs and consume much power when trained on neuromorphic
hardware.
The second category is event-driven algorithms, which propagate gradient information through spikes.
Precise spiking timing acts an important role in this situation, and they are extensively used in such
algorithms [32, 33, 34, 35, 36, 37, 38, 39]. Classical examples include SpikeProp [32] and its variants [33, 40, 41]. These algorithms approximate the derivative of spike timing to membrane potential
as the negative inverse of the time derivative of membrane potential function. This approximation
is actually mathematically correct without preconditions [42]. Some other works apply non-leaky
integrate-and-fire neurons to stabilize the training process [35, 38, 43]. Most of these works restrict
each neuron to fire at most once, which inspires [44] to take the spike time as the state of a neuron,
and model the relation of neurons by this spike time. As a result, the SNN is trained similarly to an
ANN. Among the methods trained in an event-driven fashion (not modelling the relation of spike
time to train like ANNs), the state-of-the-art model is TSSL-BP [39]. However, they use RNN-like
surrogate gradients (a sigmoid function) to assist training. Hence, it is still challenging to train SNNs
in a pure event-driven fashion.

(Event-driven Backpropagation) Event-driven learning v.s. RNN-like learning: In both forward and backward computation of
event-driven learning, information is only carried by spikes in SNNs. Specifically, in backward
computation, gradient information is propagated through spikes [32, 33, 41, 35] (shown in Fig. 1a-b).
On the other side, in RNN-like learning, information is not only carried by spikes in backward
computation. Especially, gradient information can be propagated through a neuron that does not emit
a spike in backward computation (shown in Fig. 1c-d). This gradient propagation is achieved by a
surrogate function [12, 17, 18, 23, 47], which is a function of the membrane potential at the current
time step ut, and the firing threshold θ.
Time-based gradient v.s. activation-based gradient: Time-based gradients represent the (reverse)
direction that the timing of a spike should move, that is, to move leftward or rightward on the time
axis [32]. In backward propagation, the derivative of the firing time of a spike to the corresponding
membrane potential ∂t
∂u is often approximated as −1
∂u
∂t
[32, 33], denoting how the change of membrane
potential will change the spike firing time (Fig. 1b). On the other side, activation-based approaches
replace the Heaviside neuron activation function Θ(·) (spike st = Θ(ut − θ)) in forward propagation
with derivable functions σ(·) in backward propagation, whether there are spikes in the current time
step [18, 26, 31, 21]. Therefore, activation-based approaches essentially regard SNNs as binary
RNNs and train them with approximated gradients, where the gradients indicate whether the values
in the network (including the binary spikes) should be larger or smaller (Fig. 1d).
As a result, time-based gradients are event-driven by nature, since the temporal gradient could only be
carried by spikes. Meanwhile, activation-based gradients are more suitable for the RNN-like training
scheme since the diversity of surrogate gradients largely relies on the fact that ut ̸= θ in discrete time
steps [17], which no longer holds in continuous time simulation. If we want to apply activation-based
gradients to event-driven learning, there should only be one value ∂s
∂u when the membrane potential
reaches the threshold. 
the sign of the gradient gets wrong in propagation between layers. Thus, the commonly
used double-exponential spike response kernel is incompatible with the time-based gradient in
event-driven learning.
A smoother gradient assigning approach. Inspired by the above gradient inconsistency as well
as the invariance of gradient sum, we propose a new gradient backpropagation approach here

~\cite{Rueckauer2017ConversionOC} In order to bridge the gap betweenDeep Learning continuousvalued
networks and neuromorphic spiking networks, it is
necessary to develop methods that yield deep Spiking Neural
Networks (SNNs) with equivalent error rates as their continuousvalued
counterparts. Successful approaches include direct
training of SNNs using backpropagation (Lee et al., 2016),
the SNN classifier layers using stochastic gradient descent
(Stromatias et al., 2017), or modifying the transfer function of
the ANNs during training so that the network parameters can be
mapped better to the SNN (O’Connor et al., 2013; Esser et al.,
2015; Hunsberger and Eliasmith, 2016). The largest architecture
trained by Hunsberger and Eliasmith (2016) in this way is
based on AlexNet (Krizhevsky et al., 2012). While the results
are promising, these novel methods have yet to mature to the
state where training spiking architectures of the size of VGG-16
becomes possible, and the same state-of-the-art error rate as the
equivalent ANN is achieved.

SG methods provide an alternative approach to overcoming the difficulties associated with
the discontinuous nonlinearity. Moreover, they hold opportunities to reduce the potentially high
algorithmic complexity associated with training SNNs. Their defining characteristic is that instead
of changing the model definition as in the smoothed approaches, a SG is introduced. In the following we make two distinctions. We first consider SGs which constitute a continuous relaxation
of the non-smooth spiking nonlinearity for purposes of numerical optimization (Fig. 4). Such
SGs do not explicitly change the optimization algorithm itself and can be used, for instance,
in combination with BPTT. Further, we also consider SGs with more profound changes that
explicitly affect locality of the underlying optimization algorithms themselves to improve the
computational and/or memory access overhead of the learning process. One example of this
approach that we will discuss involves replacing the global loss by a number of local loss
functions. Finally, the use of SGs allows to efficiently train SNNs end-to-end without the need
to specify which coding scheme is to be used in the hidden layers.

(Spikformer) In the area of direct training, SNNs are unfolded over the simulation time
steps and trained in a way of backpropagation through time (Lee et al., 2016; Shrestha & Orchard,
2018). Because the event-triggered mechanism in spiking neurons is non-differentiable, the surrogate
gradient is used for backpropagation (Lee et al., 2020; Neftci et al., 2019)Xiao et al. (2021a) adopts
implicit differentiation on the equilibrium state to train SNN. Various models from ANNs have
been ported to SNNs. However, the study of self-attention on SNN is currently blank. Yao et al.
(2021) proposed temporal attention to reduce the redundant time step. Zhang et al. (2022a;b) both
use ANN-Transformer to process spike data, although they have ’Spiking Transformer’ in the title.
Mueller et al. (2021) provides a ANN-SNN conversion Transformer, but remains vanilla self-attention
which does not conform the characteristic of SNN. In this paper, we will explore the feasibility of
implementing self-attention and Transformer in SNNs.

(Deep Residual Learning in Spiking Neural Networks
) The backpropagation methods can be classified into two categories [26]. The
method in the first category computes the gradient by unfolding the network over the simulation timesteps [31, 19, 58, 50, 30, 40], which is similar to the idea of backpropagation through time (BPTT). As
the gradient with respect to the threshold-triggered firing is non-differentiable, the surrogate gradient
is often used. The SNN trained by the surrogate method is not limited to rate-coding, and can also be
applied on temporal tasks, e.g., classifying neuromorphic datasets [58, 8, 16]. The second method
computes the gradients of the timings of existing spikes with respect to the membrane potential at the
spike timing [5, 39, 24, 65, 63].

(A Free lunch) For training-based SNN, there are several supervised learning algorithms divided into (1) synaptic plasticity and (2)
surrogate gradient. Synaptic plasticity methods are based on time-sensitivity and update the connection weight via
the two neurons’ firing time interval (Kheradpisheh et al.,
2018; Iyer & Chua, 2020; LI & LI, 2019). They are more
suitable for the neuromorphic image (Amir et al., 2017) or
rate coding from static images. On the other hand, surrogate
gradient (spiking-based backpropagation) methods use a
soft relaxed function to replace the hard step function and
train SNN like RNN (Wu et al., 2018; Shrestha & Orchard,
2018). They suffer from the computationally expensive
and slow during the training process on complex network
architecture(Rathi et al., 2019).

(Towards spike-based machine intelligence with neuromorphic computing) In a spike-based approach, SNNs are trained using timing information  and therefore offer the obvious advantages of sparsity and efficiency in overall spiking dynamics. Research efforts have been directed towards integrating global backpropagation-like spike-based error gradient descent to enable supervised learning in multi-layer SNNs. Most works that rely on backpropagation estimate a differentiable approximate function for the spiking neuronal functionality so that gradient descent can be performed (Fig. 4a).
SNNs already have a computational advantage as a result of binary spike-based processing. Furthermore, the stochasticity in neuronal dynamics of LIF neurons can improve the robustness of a network to external noise.

(Towards spike-based machine intelligence with neuromorphic computing) A major restriction in the use of SNNs with such sensors is the lack of appropriate training algorithms that can efficiently utilize the timing information of the spiking neurons. Practically, in terms of accuracy, SNNs are still behind their second-generation deep-learning counterparts in most learning tasks.
Another restriction on SNNs is spike-based data availability. The performance of SNN training algorithms is evaluated on existing static-image datasets, for example CIFAR or ImageNet. Such static-frame-based data are then converted to spike trains using appropriate encoding techniques, such as rate coding or rank-order coding.

(Attention Spiking Neural Networks) We do not intend to shift the meta-operator of existing SNNs, e.g., replacing convolution (or fully connected) with self-attention, but try to apply the attention as an auxiliary unit in a simple and lightweight way to easily integrate with existing SNN architectures for improving representation power, like attention CNNs. Challenges in adapting attention to SNNs arise from three aspects. Firstly, we must keep the neuromorphic computing characteristic of SNNs, which is the basis of SNN’s energy efficiency. Thus, implementing the attention while retaining SNN’s event-driven is the primary consideration. Secondly, SNNs are used to process various applications, such as sequential event streams and static images. We need to diverse attention SNN design to cope with different scenarios. Thirdly, binary spiking activity makes deep SNNs suffer from spike degradation [15] and gradient vanishing [14], collectively referred to as the degradation problem, i.e., an accuracy drop would occur on both the training and test sets when the network deepens.To adapt attention SNNs to a variety of application scenarios, we merge multidimensional attention with SNN (MA-SNN), including temporal, channel, and spatial dimensions, to learn ’when’, ’what’ and ’where’ to attend, respectively. 
We first aggregate spatial-channel information of a feature block at each time step by using both average-pooling and maxpooling operations, generating two different temporal context descriptors, which denote average-pooled features and max-pooled features respectively. Then, we transform both average-pooled and max-pooled features to a TA weight vector by a shared MLP network.
It is well known that each channel of feature maps corresponds to a certain visual pattern, and CA focuses on ”what” are salient semantic attributes for the given input.  Interestingly, we find that another key role of attention is the suppression of minor features, which is usually ignored in attention CNN but crucial for the efficiency of the SNN.
We adopt the SA part of CBAM [28] as our SA function. AvgPool, MaxPool is the 2-D SA attention weights, and a 7x7 convolution operation.
MA can be integrated into existing residual SNN architectures without constraints, and we consistently exploit attention to optimize membrane potential of spiking neurons. In this paper, we adopt the MS-Res-SNN [14] as the basic residual block.

(Self-Supervised Learning of Event-Based Optical Flow with Spiking Neural Networks) we use a representation consisting only of per-pixel and per-polarity event counts. This representation gets populated with consecutive, non-overlapping partitions of the event
stream each containing a fixed number of events.
We use the contrast maximization proxy loss for motion compensation [16] to learn to estimate optical flow from the continuous event stream in a self-supervised fashion.
We compare various spiking neuron models from literature on the task of event-based optical flow estimation. we introduce an adaptive threshold to make up the adaptive LIF (ALIF) model. A second state variable T acts as a low-pass filter over the output spikes, adapting the firing threshold based on the neuron’s activity.
Instead of postsynaptic adaptivity, we can keep a trace of presynaptic activity and use that to regularize neuron firing, giving the presynaptic LIF (PLIF) model.

(GLIF: A Unified Gated Leaky Integrate-and-Fire
Neuron for Spiking Neural Networks) We propose the gated LIF model (GLIF) that fuses different bio-features in the aforementioned
three neuronal behaviors to possess more response characteristics. As illustrated in Fig. 1, GLIF
controls the fusion of different bio-features through gating units Gα, Gβ, and Gγ for those three
neuronal behaviors that are membrane potential leakage, integration accumulation, and spike initiation,
respectively. In each gating unit, a gating factor is computed from a Sigmoid function σ(x) over
a learnable gating parameter x to determine the proportion of each bio-feature, thus guiding the
fusion of different bio-features. Therefore, GLIF can simultaneously contain different bio-features,
possessing more response characteristics. In addition, when the gating factor is 0 or 1, GLIF can also
support single bio-feature. As a result, GLIF can cover other different LIF models and be viewed as a
super spiking neuron, greatly enlarging the representation space of spiking neurons.
Furthermore, we introduce the channel-wise parametric method to GLIF. This method makes all
membrane-related parameters in GLIF learnable and shares the same GLIF parameters channelwisely in SNNs. Combining with learnable gating factors in GLIF, on the one hand, this method
makes different channels in SNNs have completely different spiking neurons, leveraging the larger
representation space of GLIF neurons to increase the neuronal dynamic diversity of spiking neurons.
Meanwhile, the heterogeneity of spiking neurons and the expressive ability of SNNs are also increased.
On the other hand, the spike neurons in SNNs are constantly changing during training, which is
similar to the neuronal maturation during development [20, 21, 22], thus the adaptivity of spiking
neurons being enhanced.
}\fi

\textbf{Theoretical understandings of SNN.} The theoretical basis of SNN is an unclear but widely investigated area~\cite{zhang2022theoreticallyprovable}. 
SNN can behave as a computational model of Turing machine~\cite{Maass1996Lowerbound,Maass1996ThirdGen}. 
Chaos theory analyzes singularities and asymptotic behavior of SNN~\cite{cessac2008discretespikeneurontheory}.
\citeauthor{mancoo2020understandingconvexopt} shows that a continuous LIF network holistically solves quadratic programming with convex constraints as a gradient flow. Bifurcation theory shows that the LIF network is a bifurcation dynamical system highly sensitive to decay factor~\cite{zhang2021bifurcation}.
Prior theoretical results on continuous SNN applies inexactly to discrete-time dynamics due to discretization errors~\cite{roy2010discretizationerror,mancoo2020understandingconvexopt} accumulating through time~\cite{niesen2004globaldiscretizationerror}. In contrast, we show that discrete neuronal dynamics of integrate-and-fire models approximate a subgradient method.
Furthermore, we practically extend our theory to achieve state-of-the-art performance on SNN inference and ANN-to-SNN conversion.

\ifverbose{
several works theoretically analyze SNN's computational capacity 
(Towards spike-based machine intelligence with neuromorphic computing)~\cite{roy2019towardsspikebasedmachine} A complementary body of work in the SNN domain is that of liquid state machines (LSMs)68. LSMs use unstructured, randomly connected recurrent networks paired with a simple linear readout. Such frameworks with spiking dynamics have shown a surprising degree of success for a variety of sequential recognition tasks69–71, but implementing them for complex and large-scale tasks remains an open problem.

(Theoretically Provable Spiking Neural Networks)~\cite{zhang2022theoreticallyprovable}
Some researchers [19, 20, 21, 31] focused on the approximation universality of SNNs, in which some
typical SNNs can simulate the standard computational models such as Turing machines, random
access machines, threshold circuits, sequence-to-sequence mapping, etc. There are also efforts on
the computational efficiency of SNNs for some specific issues, such as the convergence in the limit
results and computational complexity of SNNs for the sparse coding problem [34, 35] and temporal
quadratic programming [5, 8], respectively.
Amazingly, a recent study [44] theoretically proved that, contrary to previous beliefs, typical SNNs
can hardly work well on spatio-temporal data, because they in nature are bifurcation dynamical
systems with fixed eigenvalues in which many patterns inherently cannot be learned. They also
suggested that adding self-connection structure can enhance the representation ability of SNNs on
spatio-temporal systems that fully connect the spiking neurons in the same layer and solves adaptive
eigenvalues of discrete dynamical systems. In this paper, we theoretically investigate the approximation ability and computational efficiency of
the self-connection spike neural networks (scSNNs). Our theoretical results show that equipped with
self connections, scSNNs can approximate discrete dynamical systems using polynomial number
of parameters within polynomial time complexities. Our main contributions are summarized as
follows:
• We prove that the proposed scSNNs are universal approximators in Theorem 1.
• As for spatial approximation, we prove that a broad range of radial functions can be well
approximated by scSNNs with polynomial spiking neurons in Theorem 2.
• As for temporal approximation, we prove that multivariate spike flows can be approximated
by scSNNs within polynomial time in Theorem 3 and verify this conclusion in simulation
experiments.

(bifurcation spiking neural network)~\cite{zhang2021bifurcation} This work investigates the dynamical properties of the LIF-modeling SNNs, especially the
influence of hyper-parameters on the model dynamics. As a result, we declare that the LIF model is
a bifurcation dynamical system, which means that its topology depends sensitively on the control
rate hyper-parameters. This result sheds three significant insights: (1) The performance of SNNs
is sensitive to the setting of control rates, which is consistent with the facts. (2) It is necessary to
enable diverse and learnable control rates, corresponding to the eigenvalues of bifurcation dynamical
systems, for achieving adaptive systems. This claim argues the conventional manners that the control
rates are neatly preset as a negative fixed value. (3) The role of control rates cannot be replaced by
learnable connection parameters and other hyper-parameters.
However, training control rates is a very tricky challenge. Since control rates and connection
weights are entangled during the training process, the approaches (Hunter et al., 2012; Lorenzo et al.,
2017) of turning hyper-parameters in conventional neural networks cannot be directly used to solve
this issue. An alternative way is to sample the control rates from a certain pre-defined distribution
and find the optimal ones by alternating optimization. Nevertheless, this method usually succeeds on
an apposite distribution and larger computation and storage.
To tackle the challenges above and improve the performance of SNNs, we propose the Bifurcation
Spiking Neural Network (BSNN). By exploiting the bifurcation theory, we convert the issue of
learning a group of adaptive control rates into a new problem of learning a collection of apposite
eigenvalues. So BSNN overcomes the obstacle that controls rates interact with connection weights,
leading to a robust control rate setting and achieves a laudable performance with considerably less
computation and storage than the alternating optimization approaches. The experiments conducted
on four benchmark data sets demonstrate the effectiveness of BSNN, showing that its performance
surpasses existing SNNs and is robust to the setting of control rates.
Our main contributions are summarized as follows:
• We provide a theoretical framework for studying the dynamical properties of spiking neural
models, e.g., we show the LIF model is a bifurcation dynamical system in Section 4.
• We point out the fact that the control rate hyper-parameter, rather than other ones, is sensitive
to the performance of SNNs with LIF neurons.
• We present the BSNN with supervised implementation, and then theoretically show that the
dynamical system led by BSNN has adaptive eigenvalues, leaving a robust setting of the
control rate hyper-parameters

(Increasing Liquid State Machine Performance with
Edge-of-Chaos Dynamics Organized by
Astrocyte-modulated Plasticity)~\cite{ivanov2021lsmedgeofchaos}
The liquid state machine (LSM) combines low training complexity and biological plausibility, which has made it an attractive machine learning framework for
edge and neuromorphic computing paradigms. Originally proposed as a model of
brain computation, the LSM tunes its internal weights without backpropagation
of gradients, which results in lower performance compared to multi-layer neural
networks. Recent findings in neuroscience suggest that astrocytes, a long-neglected
non-neuronal brain cell, modulate synaptic plasticity and brain dynamics, tuning
brain networks to the vicinity of the computationally optimal critical phase transition between order and chaos. Inspired by this disruptive understanding of how
brain networks self-tune, we propose the neuron-astrocyte liquid state machine
(NALSM)1
that addresses under-performance through self-organized near-critical
dynamics. Similar to its biological counterpart, the astrocyte model integrates
neuronal activity and provides global feedback to spike-timing-dependent plasticity
(STDP), which self-organizes NALSM dynamics around a critical branching factor
that is associated with the edge-of-chaos.
With the recent rise of neuromorphic [1–4] and edge computing [5, 6], the liquid state machine (LSM)
learning framework [7] has become an attractive alternative [8–11] to deep neural networks owing to
its compatibility with energy-efficient neuromorphic hardware [12–14] and inherently low training
complexity. Originally proposed as a biologically plausible model of learning, LSMs avoid training
via backpropagation by using a sparse, recurrent, spiking neural network (liquid) with fixed synaptic
connection weights to project inputs into a high dimensional space from which a single neural layer
can learn the correct outputs. Yet, these advantages over deep networks come at the expense of 1)
sub-par accuracy and 2) extensive data-specific hand-tuning of liquid weights. Interestingly, these
two limitations have been targeted by several studies that tackle one [15, 16] or the other [17, 18],
but not both. This has limited the widespread use of LSMs in real-world applications [8]. In that
sense, there is an unmet need for a unified, brain-inspired approach that is directly applicable to the
emerging neuromorphic and edge computing technologies, facilitating them to go mainstream.

(Expressivity of Spiking Neural Networks)~\cite{singh2023expressivity} This article studies the expressive power of spiking neural networks where infor-
mation is encoded in the firing time of neurons. The implementation of spiking
neural networks on neuromorphic hardware presents a promising choice for future
energy-efficient AI applications. However, there exist very few results that com-
pare the computational power of spiking neurons to arbitrary threshold circuits and
sigmoidal neurons. Additionally, it has also been shown that a network of spiking
neurons is capable of approximating any continuous function. By using the Spike
Response Model as a mathematical model of a spiking neuron and assuming a
linear response function, we prove that the mapping generated by a network of
spiking neurons is continuous piecewise linear. We also show that a spiking neural
network can emulate the output of any multi-layer (ReLU) neural network. Further-
more, we show that the maximum number of linear regions generated by a spiking
neuron scales exponentially with respect to the input dimension, a characteristic
that distinguishes it significantly from an artificial (ReLU) neuron. Our results
further extend the understanding of the approximation properties of spiking neural
networks and open up new avenues where spiking neural networks can be deployed
instead of artificial neural networks without any performance loss

(A discrete time neural network model with spiking
neurons
Rigorous results on the spontaneous dynamics.)~\cite{cessac2008discretespikeneurontheory} This paper is the first one of a series trying to address some of these questions in
the context of BMS model. The goal the present article, is to pose the mathematical
framework used for subsequent developments. In section 2 we present the BMS model
and provide elementary mathematical results on the system dynamics. We show that
the presence of a sharp threshold for the model definition of neuron firing induces singularities responsible for a weak form of initial conditions sensitivity. This effect is
different from the usual notion of chaos since it arises punctually, whenever a trajectory
intersects a zero Lebesgue measure set, called the singularity set. Similar effects are encountered in billiards [16] or in Self-Organized Criticality [2],[3],[4]. Applying methods
from dynamical systems theory we derive rigorous results describing the asymptotic
dynamics in section 3. Although we show that the dynamics is generically periodic, the
presence of a singularity set has strong effects. In particular the number of periodic
orbits and the transients growth exponentially as the distance between the attractor
and the singularity set tends to zero. This has a strong impact on the numerics and
there is a dynamical regime numerically indistinguishable from chaos. Moreover, these
effects become prominent when perturbing the dynamics or when the infinite size limit
is considered. In this context we discuss the existence of a Markov partition allowing
to encode symbolically the dynamics with “spike trains”. In section 4 we indeed show
that there is a one to one correspondence between the membrane potential dynamics
and the sequences of spiking patterns (“raster plots”). This opens up the possibility
to use methods from ergodic theory and statistical mechanics (thermodynamic formalism) to analyse spiking sequences. This aspect will be the central topic of another
paper. As an example, we briefly analyze the case of random synapses and inputs on
the dynamics and compare our analysis to the results obtained by BMS in [51],[50]. We
exhibit numerically a sharp transition between a neural death regime where all neurons
are asymptotically silent, and a phase with long transient having the appearance of a
chaotic dynamics. This transition occurs for example when the variance of the synaptic
weights increases. A further increase leads to a periodic dynamics with small period. In
the discussion section we briefly comment some extensions (effect of Brownian noise,
use of Gibbs measure to characterize the statistics of spikes) that will be developed in forthcoming papers.

~\cite{lu2022linearlif} Structural equivalence is mainly reflected in the structures
of the ReLU-AN model and LIF model. The parameters
of the two models should have a mapping relationship
represented by a transformation function R. R can be
described as a binary relation satisfying reflexive (xRx),
symmetrical xRy ⇒ yRx), and transitive properties ((xRy∧
yRz) ⇒ xRz).We will present a perfect parametermapping
between Linear LIF/SNN (model A) and ReLU/ANN
(model B) in Section 4.4.
• Behavioral equivalence focuses on the functional
equivalence of the two models, requiring that model A can complete the functions of model B, and vice versa. We
define behavioral equivalence as: “Model A and Model B
have the same output if run under identical experimental
conditions. Given a parameter mapping rule, there always
exists a small error bound ε that
FA(x) − FB(x)
≤ ε can
be guaranteed for any valid input x, where FA,FB denotes
the function of model A and model B.”
In this section, we demonstrate the equivalence of LIF/SNN
and ReLU-AN/DNN model and the advantages of the Linear
LIF model compared to the Reset-to-Zero LIF model. As shown
in Figure 5, it mainly includes: 1. Verify the structural and
functional equivalence of LIF/SNN and ReLU/DNN through
simulation. 2. Reduce the simulation error by increasing
the sampling frequency and coding time, demonstrating a
convergence toward ideal conditions.
The simulation experiment in this section is mainly divided
into two parts:
1. Simulation 1: Proof of structural equivalence
(a) Compare the Linear LIF model with the ReLU-AN model
(with bias) when the input signal frequencies are the same.
2. Simulation 2: Prove of behavioral equivalence
(a) Compare the Linear LIF model and the ReLU-AN model
(with bias) under the condition that the two input signal
frequencies are different.
b) Compare the Linear LIF model and the ReLU-AN model
(with bias) under the condition that the input signal
frequencies are different.
(c) Compare LIFNN and DNN (without bias) based on
face/motor data set and MNIST and CIFAR10 data set.
}\fi
\section{Preliminaries}

Integrate-and-fire models are simplified phenomenological models of biological neuronal dynamics~\cite{gerstner2014neuronaldynamcis}. It consists of two components: (i) a time-evolution of membrane potential (Integration) and (ii) a firing mechanism to create a spike (Thresholding). Its continuous neuronal dynamics is a differential equation with a thresholding criterion (See Appendix~\ref{sec:continuous_lif_dynamics}.).
For computational tractability, the general one-dimensional integrate-and-fire model is discretized as follows.
\begin{align}
u_{pre}(t) &= u(t-1) +  f\big(u(t-1)\big)
    + \frac{R}{\tau_m}I(t) \quad \label{lif:1}\\
    s(t) &= \mathbb{H}(u_{pre}(t) - \theta_{th}) \label{lif:2}
\end{align}
where time $t \in \mathbb{N}$, dynamics function $f(u): \mathbb{R} \to \mathbb{R}$,  pre-firing potential $u_{pre}(t)$, $s(t) \in \{0,1\}$ a spike, post-firing potential $u(t)$, heaviside step function $\mathbb{H}$, influx current $I(t)$, threshold $\theta_{th}$, membrane resistance $R$, and membrane constant $\tau_m$.
The potential $u_{pre}$ resets to $u$ after the spike $s(t)$ fires based on its pre-defined reset mechanism. 
\begin{align}
    &u(t) = 
    u_{pre}(t) - \theta_{th}s(t) & \text{(reset-by-subtraction)}\label{lif:3}
    \\
    &u(t) =  u_{pre}(t) (1-s(t)) & \text{(reset-to-zero)}\nonumber
\end{align}
Discretized reset mechanisms  are categorized into two: (i) \emph{reset-to-zero} discards the leftover potential, and (ii) \emph{reset-by-subtraction}~\cite{Han2020RMPSNNRM} retains the leftover potential after the reset. We focus on the reset-by-subtraction mechanism since it is easier to theoretically analyze and more performant in practical applications~\cite{han2020deeptsc}.

Integrate-and-fire (IF) neuron has the simplest neuronal dynamics defined as $f(u) = 0, \tau_{m} = 1$ in equation~\eqref{lif:1}.
\begin{equation}
    u_{pre}(t) = u(t-1) 
    + R\ I(t) \label{if:1}
\end{equation}
Leaky-Integrate-and-fire(LIF) neuron introduces linear leakage $f(u) = -\frac{u - u_{rest}}{\tau_{m}}$ into the dynamics of equation~\eqref{lif:1}.
\begin{equation}
    u_{pre}(t) = u(t-1) -\frac{u(t-1) - u_{rest}}{\tau_{m}}
    + \frac{R}{\tau_{m}} I(t) \quad \label{lif:4}
\end{equation}

Neural coding schemes interpret the information representation of SNN by encoding a real value into a spike train and vice versa. Rate coding decodes an activation value as a ratio of spike events over time steps, i.e., 
$y = \text{(\# of spikes)} /\text{(total \# of time-steps)}$.
Importantly, it can be equally defined as a moving average of spikes,
\begin{equation}
y(t) = y(t-1) \cdot (t-1) / t + s(t) \cdot (1  / t) \label{eq:rate}
\end{equation}
Another coding scheme used in ANN-to-SNN conversion literatures is phase coding~\cite{liu2022spikeconverter, Li2021Bistable, Kim2018phasecoding}.
Phase coding encodes information in the phase of spikes, which correlates with 
internal oscillation rhythms~\cite{guo2021neuralcoding}. Phase coding-based techniques in SNN assign the weight $W_i = (1/2)^i$ to the phase $i$, similar to binary digits~\cite{Kim2018phasecoding}. To simplify the theoretical analysis, we generalize the phase coding into an arbitrary base $\tau \in \mathbb{R}^+$ of weight $W_i = \frac{1}{\tau} \big( \frac{\tau - 1}{\tau}\big)^{i-1}$ and an infinite-length period. We define it as exponential moving average (EMA) coding since its streaming update over a spike train can be defined as follows.
\begin{equation}
    y(t) =  y(t-1) \cdot (\tau - 1)/\tau +  s(t) \cdot 1/\tau
    \label{def:ema_coding}
\end{equation}

\begin{figure}[ht]
\vskip -0.2in
\begin{center}
\centerline{\includegraphics[width=0.99\columnwidth]{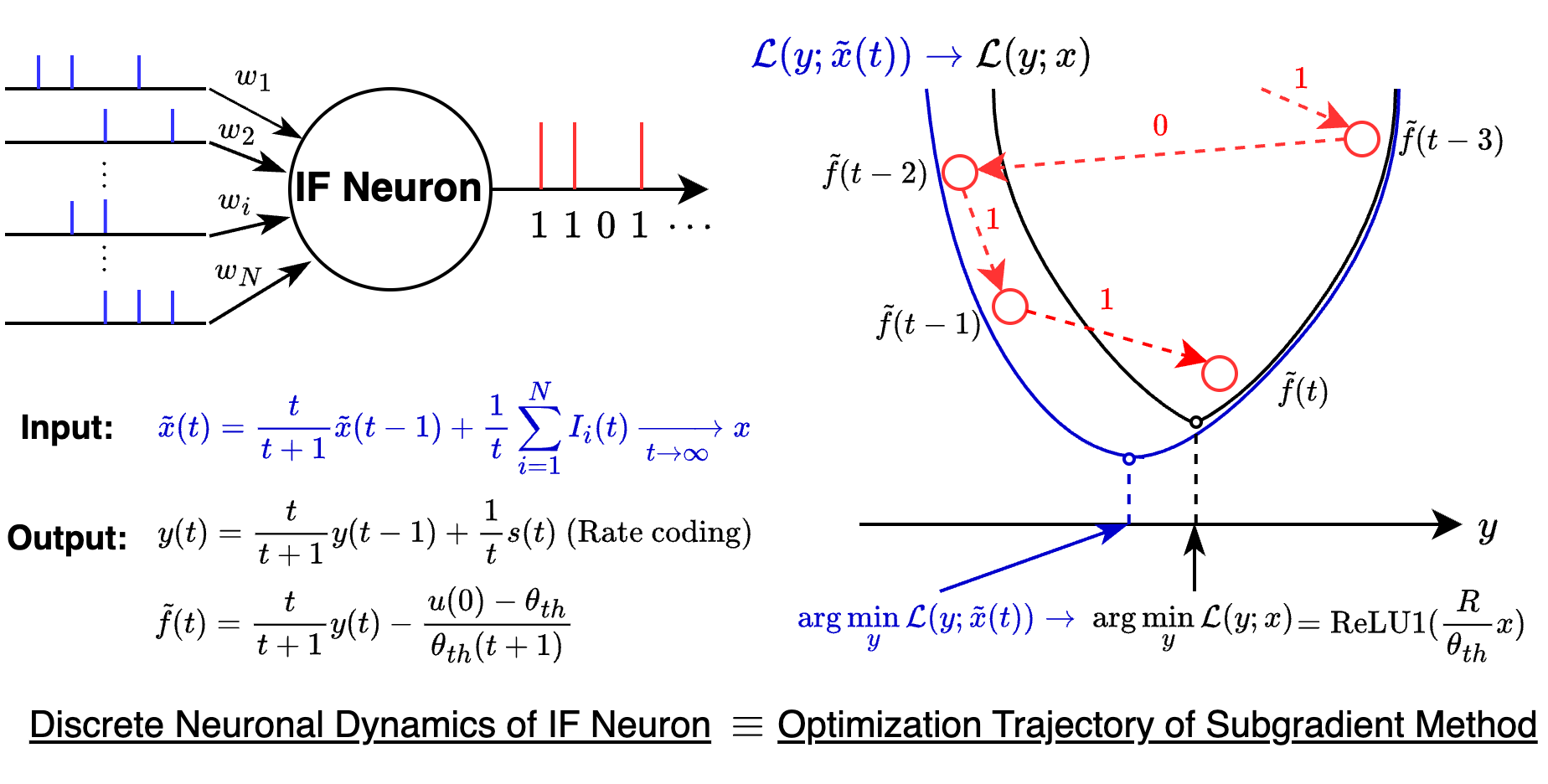}}
\vskip -0.15in
\caption{Mathematical equivalence of discrete neuronal dynamics of IF neuron (left) and subgradient method over an unconstrained convex optimization problem (right), described in Theorem~\ref{thm:if_rate}.}
\label{fig:dynamics_equivalence}
\end{center}
\vskip -0.4in
\end{figure}
\section{Optimizer Model of Neuronal Dynamics}
In this section, we show that neuronal dynamics of integrate-and-fire models is a first-order iterative optimization process approximating a spike-coded nonlinear function value.

\subsection{Theoretical Analysis}
We begin by showing that the IF neuron with rate-coded input behaves as a subgradient method with diminishing step sizes to approximate a clipped ReLU, as illustrated in Figure~\ref{fig:dynamics_equivalence}. 
Due to limited space, we  list detailed proofs of all the following theorems in the Appendix~\ref{app:proofs}. We denote a ReLU clipped at $x = 1$ as ReLU1(x). 

\begin{theorem}
\label{thm:if_rate}
Dynamical system of IF neuron~(Eq. \ref{lif:2},\ref{lif:3}, \ref{if:1}) with rate-coded input $\tilde{x}(t) = \frac{1}{t}\sum_{i=1}^{t} I(i)$ and output $y(t)= \frac{1}{t}\sum_{i=1}^{t} s(i)$ is equivalent to the subgradient method  over an optimization problem $\min_{y \in \mathbb{R}} \mathcal{L}(y;x)$, approximated with $x \leftarrow \tilde{x}(t+1)$ as,
\begin{align}
&\tilde{f}(t) = \tilde{f}(t-1) - \frac{1}{t+1} \cdot \tilde{g}\big(\tilde{f}(t-1); \tilde{x}(t)\big)\label{eq:main_if_subgradient} 
\end{align}
\begin{align}
&\mathcal{L}(y; x) = h\big( \frac{R}{\theta_{th}} x - y\big) + \frac{1}{2} y^2 \quad\text{(objective fn.)}\nonumber\\
&\tilde{f}(t) = \frac{t}{t+1} y(t) - \frac{u(0) - \theta_{th}}{\theta_{th}(t+1)}~ \text{($t$-th approx.)}\label{eq:main_if_tthapprox}
\end{align}
where $\tilde{g}(y;x)$ is a subgradient of $\mathcal{L}(y;x)$, $h(x) = \text{ReLU}(x)$. Solution of the problem is $\text{ReLU1}(\frac{R}{\theta_{th}} x)$. (Proof at~\ref{thm:if_dynamics})
\end{theorem}

We now demonstrate that our framework can also predict the behavior of unknown combinations of neuronal dynamics and neural coding, beyond the known combination. We show that the LIF neuron with EMA-coded input behaves as a subgradient method with a constant step size. 
\begin{theorem}
\label{thm:lif_rate}
Dynamical system of LIF neuron~(Eq. \ref{lif:2},\ref{lif:3}, \ref{lif:4}) with EMA-coded input $\tilde{x}(t)$ and output $y(t)$~(Eq. \ref{def:ema_coding}), if $\tau_m = \tau$, is equivalent to subgradient method over an optimization problem $\min_{y\in \mathbb{R}} \mathcal{L}(y;x)$ with $x  \leftarrow \tilde{x}(t + 1)$,
\begin{align*}
&\tilde{f}(t+1) = \tilde{f}(t) - \frac{1}{\tau} \cdot \tilde{g}\big(\tilde{f}(t); \tilde{x}(t+1)\big) \\
&  \mathcal{L}(y; x) = \frac{1}{2} y^2 +\frac{u_{rest}}{\theta_{th}(\tau-1)} y +h\left(\frac{Rx - \theta_{th}}{\theta_{th}(\tau - 1)}x - y\right)\\
& \tilde{f}(t) = y(t)  - \frac{1}{\tau \theta_{th}}(\frac{\tau-1}{\tau})^{t}u(0) - \frac{u_{rest}\sum_{i = 0}^{t}(\frac{\tau - 1}{\tau})^{t - i}}{\theta_{th} \tau(\tau-1)} 
\end{align*}
where $\tilde{g}(y;x)$ is a subgradient of $\mathcal{L}(y;x)$, $h(x) = \text{ReLU}(x)$. If $u_{rest} = 0$, the solution to the problem is $\text{ReLU1}(\frac{R}{\theta_{th}(\tau - 1)}x - \frac{1}{\tau-1})$. (Proof at~\ref{thm:lif_dynamics})
\end{theorem}

\subsubsection{Convergence Analysis}
Our framework provides a way to understand the asymptotic behavior of spiking neurons by theoretically analyzing its corresponding \emph{optimizer form}. To show this, we conduct a convergence analysis on the optimizer form of an IF neuron to reveal its asymptotic properties. Note that the spike-based information representation discrepates the subgradient estimate $\tilde{g}(y;\tilde{x}(t))$ and the true subgradient $\tilde{g}(y;x)$. 
Thus, the convergence property of spike train input determines the convergence property of the neuron output.

\begin{theorem}
\label{thm:convergence_if}
Let $f^*(x) = \argmin_{y \in \mathbb{R}}\mathcal{L}(y;x)$ be the minimizer of $ \mathcal{L}(y;x) = \text{ReLU}(x - y) + \frac{1}{2} y^2$ over a true input $x \in \mathbb{R}$, and its $t$-th approximation $\tilde{f}(t)$ be  defined as  equation~\ref{eq:main_if_subgradient}  and rate-coded input $\tilde{x}(t) = \frac{1}{t}\sum_{i=1}^{t} I(i)$. Denote $h(i) = \big(1 - \sqrt{\frac{i-1}{i+1}}\big)$. If $\Vert \tilde{f}(t) \Vert < M \in  \mathbb{R}^+$, then the error $\Vert \tilde{f}(t) - f^*(x) \Vert$ is upper bounded as follows. (Proof at~\ref{thm:if_neuron_convergence_analysis})
\begin{align}
&\Vert \tilde{f}(t) - f^*(x)\Vert^2 
\le \frac{\Vert \tilde{f}(0) - f^*(x)\Vert^2}{t+1}
  \\
&+\underbrace{\frac{M + 1}{t+1}\sum_{i = 1}^{t} h(i)}_{\text{Nondifferentiability Error}}
+\underbrace{\frac{4}{t+1}\sum_{i = 1}^{t}  \min(\Vert x - \tilde{x}(i)\Vert, 1)}_{\text{Input Error}}\nonumber
\end{align}
\end{theorem}

An assumption $\Vert \tilde{f}(t) \Vert < M$ over the optimization trajectory is rarely violated since $\tilde{f}(t)$ is attracted towards the bounded value $\text{ReLU1}
\big(\tilde{x}(t+1)\big)$, and $\tilde{x}(t) \xrightarrow{t \to \infty} x$. Furthermore, the upper bound can be tighter since the nondifferentiability error term $h(i)$ is non-zero only when the $i$-th update crosses the singularity $\tilde{x}(t+1)$. 
We now derive the convergence of neuron output in diverse input setups.
\begin{corollary}
\label{cor:convergence_results}
    Let $\Vert \tilde{f}(t) \Vert < M \in \mathbb{R}^+$, then (Proof at~\ref{cor:conv_guarantee_input})
    \begin{itemize}
    \item (Exact Input) If $\tilde{x}(t) = x$, then $\Vert \tilde{f}(t) - f^*\Vert \to 0$ as $t \to \infty$. 
    \item (Deterministic Input) If $\Vert \tilde{x}(t) - x \Vert = \mathcal{O}(\frac{1}{t})$, then $\Vert \tilde{f}(t) - f^*\Vert \to 0$ as $t \to \infty$. 
    \item (Stochastic Input) If $\mathbb{E}[\Vert \tilde{x}(t) - x \Vert] = \mathcal{O}(\frac{1}{t})$, then $\mathbb{E}[\Vert \tilde{f}(t) - f^*\Vert] \to 0$ as $t \to \infty$. 
    \end{itemize}
\end{corollary}

\textbf{Empirical Validation.}
To empirically validate our theoretical findings, we conduct toy experiments on a single neuron and compare the output of a neuron and its optimizer form. We use spikingjelly's implementation~\cite{spikingjelly} of spiking neurons and Poisson encoding.
Figure~\ref{fig:toy_if_rate} in Appendix shows that the time-evolution of IF neuron output transformed with Equation~\ref{eq:main_if_tthapprox} exactly equals its optimizer form in theorem~\ref{thm:if_rate}. As   Corollary~\ref{cor:convergence_results} states, the decoded output $\tilde{f}(t)$ of IF neuron converges to $\text{ReLU1}(\frac{R}{\theta_{th}}x)$ as the input spike train $\tilde{x}(t)$ converges to $x$. We also verify in Figure~\ref{fig:toy_lif} that the time-evolution of LIF neuron output matches the output of its subgradient method-based optimizer form in theorem~\ref{thm:lif_rate}. Both results show that the asymptotic behavior of spiking neuron outputs follows the known convergence properties of the subgradient method~\cite{boyd2003subgradientmethod}. Since the IF neuron with rate-coded input has a nonsummable diminishing step size schedule $\eta(t) = \frac{1}{t+1}$ for its optimizer form, the output converges. In contrast, the output only converges up to an error bound since the LIF neuron has a constant step size for its optimizer form.

\begin{figure}[ht]
    \vskip -0.05in
    \begin{center} 
     \subfigure[Sequential stages of neuronal dynamics]{\includegraphics[width=0.37\columnwidth]{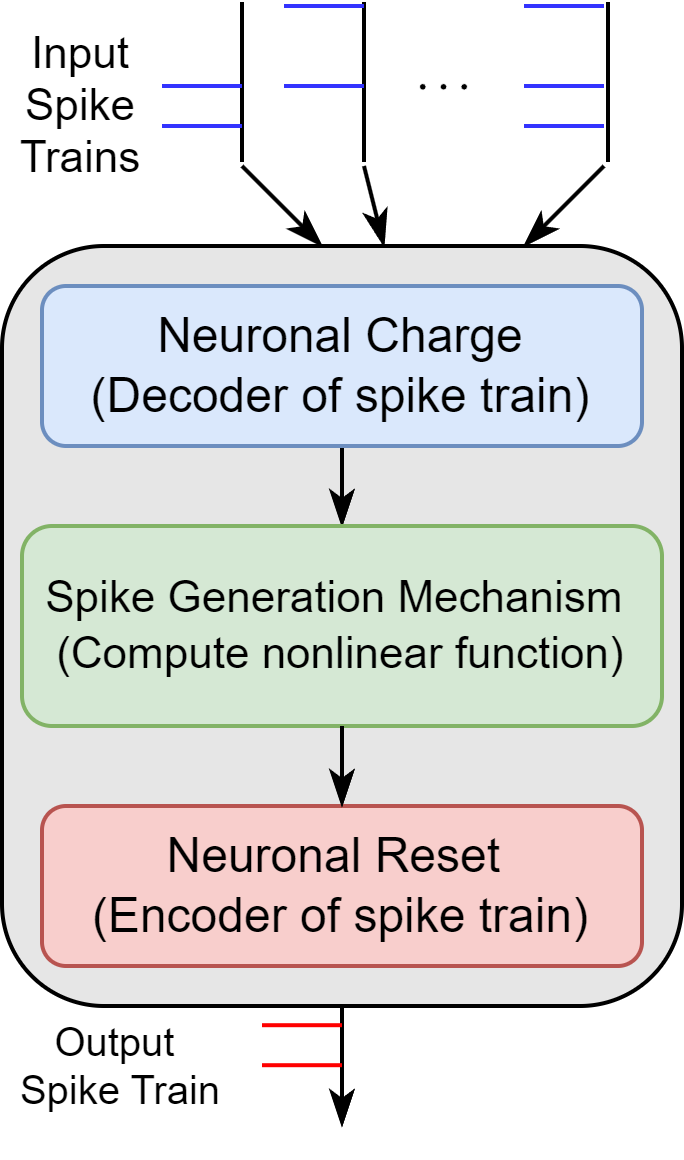}
     \label{fig:dynamics_stages}}     
     \hfill
     \subfigure[Optimization-theoretic interpretation of dynamics stages in Fig.~\ref{fig:dynamics_stages}]{\includegraphics[width=0.57\columnwidth]{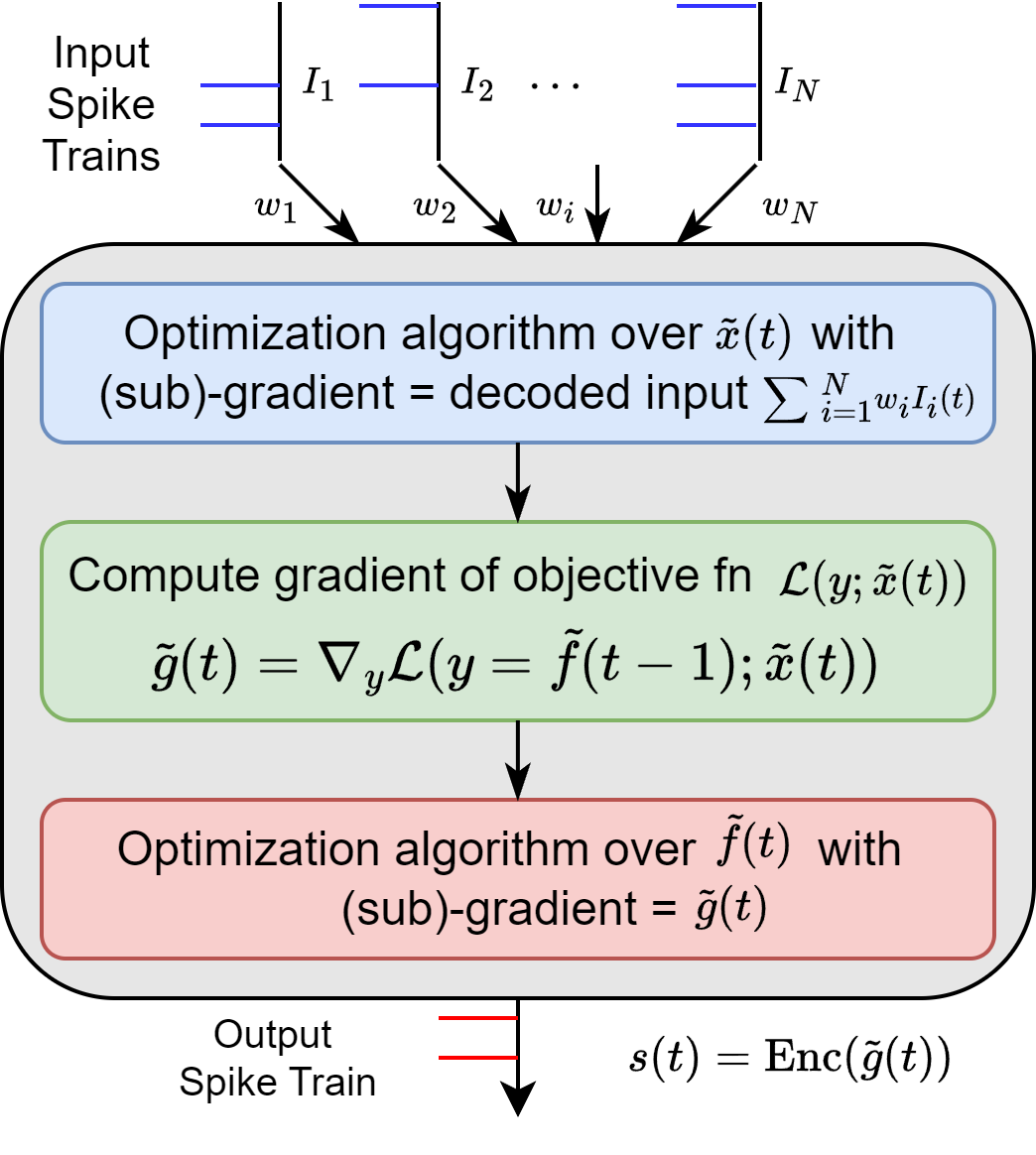}
     \label{fig:dynamics_interpretation}} 
    \vskip -0.1in
     \vfill
     \subfigure[A spike train carries iterative updates (gradients) of a presynaptic neuron's output value to postsynaptic neurons, thereby synchronizing an activation value.]{\includegraphics[width=0.80\columnwidth]{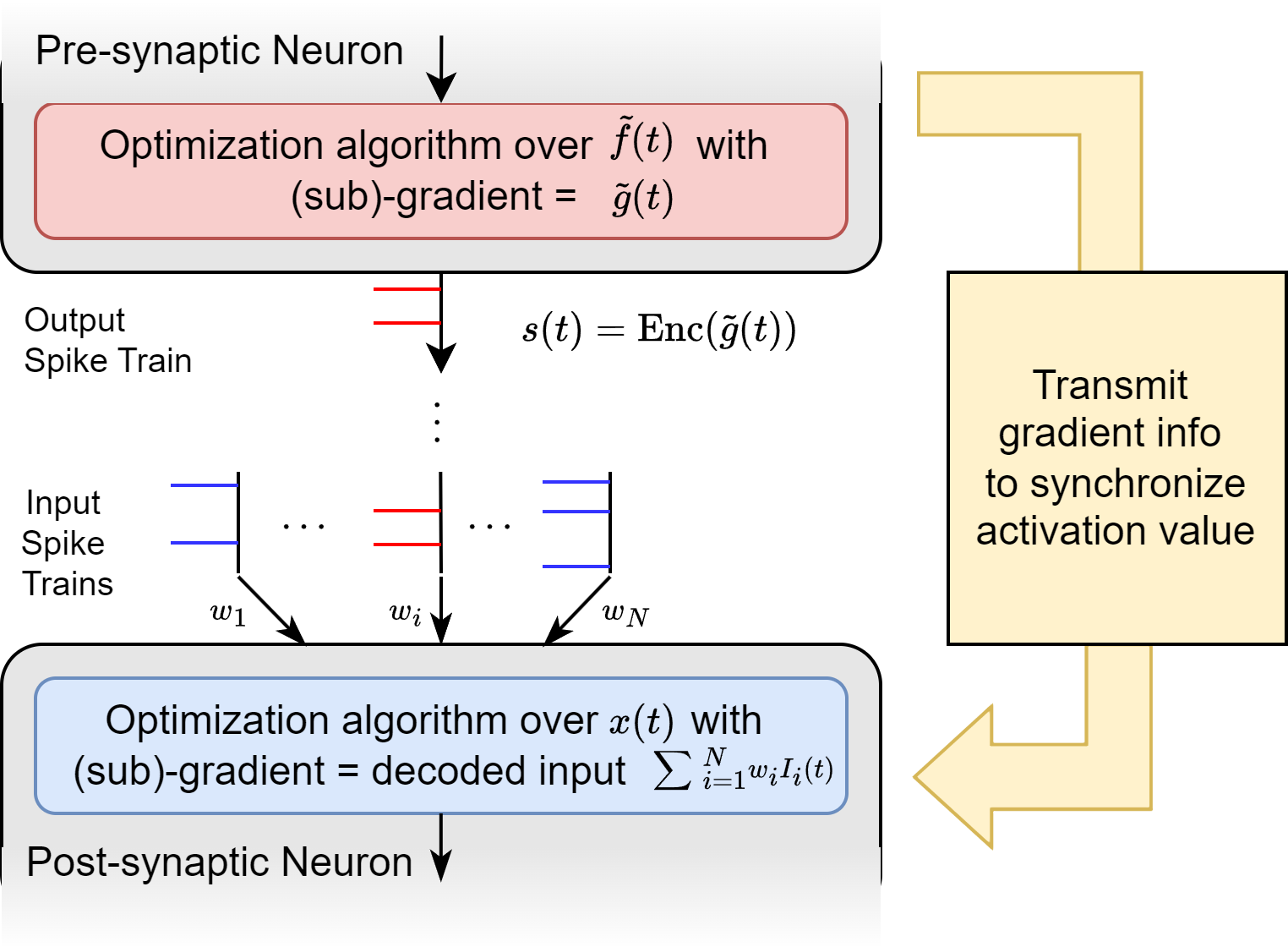}
     \label{fig:spike_train_interpretation}}
    \end{center}
    \vskip -0.15in
    \caption{Our interpretation of SNN's computational characteristics: neuronal dynamics (\ref{fig:dynamics_stages}-\ref{fig:dynamics_interpretation}), spike train (\ref{fig:spike_train_interpretation})
    \label{fig:interpretation}.}
    \vskip -0.2in
\end{figure}

\subsection{Interpretation}
Our framework provides an optimization-theoretic interpretation of SNN's key characteristics in four-fold. First, neuronal dynamics, i.e., a dynamical system of spiking neurons, can be analyzed using its corresponding \emph{optimizer form}.
Second, in an end-to-end SNN inference, each spiking neuron computes a first-order iterative algorithm to approximate a nonlinear function value. The integration phase spike-decodes pre-synaptic currents to approximate an input activation, and the thresholding phase estimates a subgradient from the approximated input.
Third, the role of the spike train is to deliver gradient information of a neuron's iterative updates to post-synaptic neurons. 
Finally, the neural coding scheme and neuronal dynamics jointly determine the learning rate schedule of its optimizer form. 

\begin{figure}[ht]
\vskip -0.0in
\begin{center}
\centerline{\includegraphics[width=0.75\columnwidth]{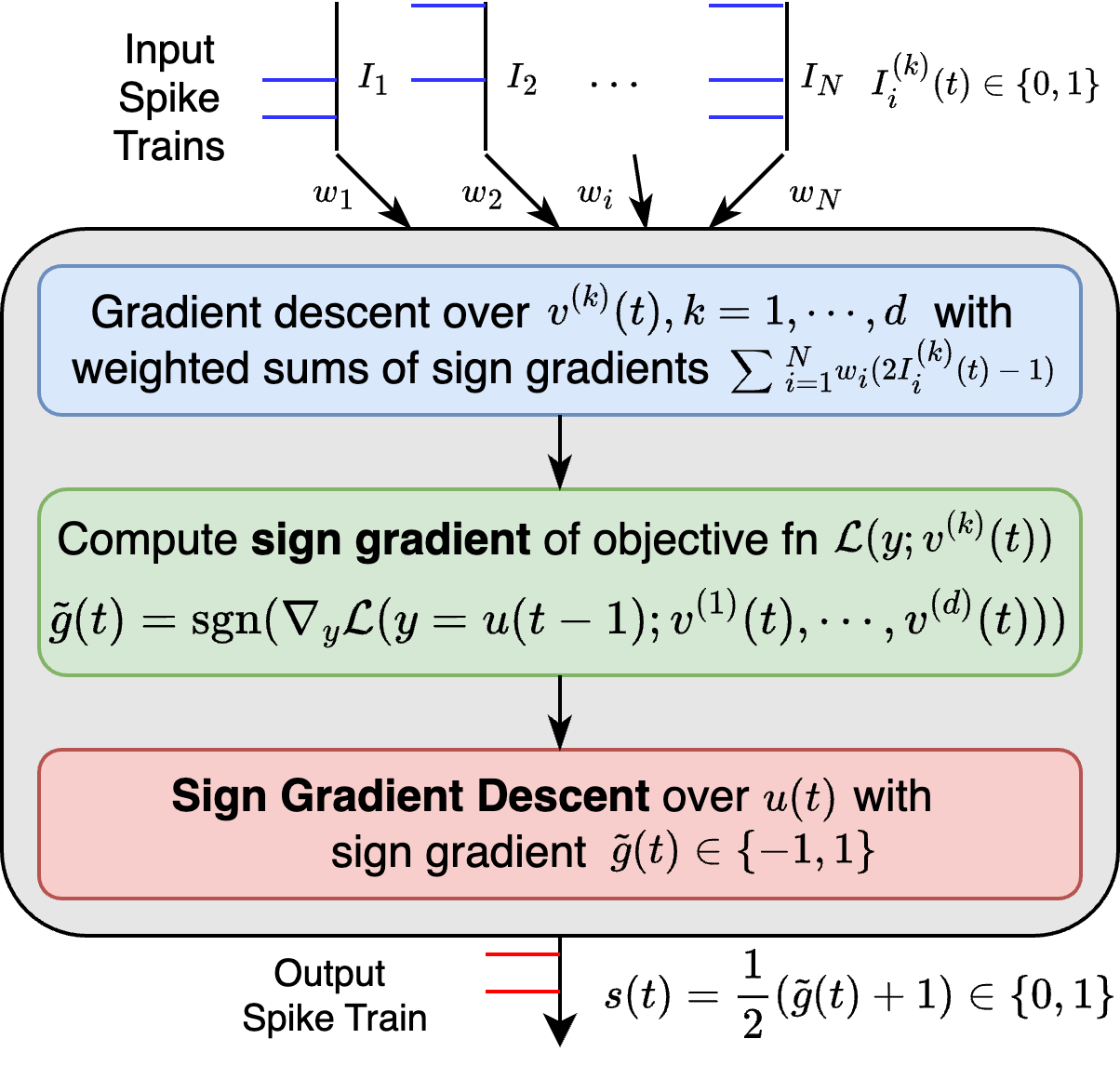}}
\vskip -0.15in
\caption{sign gradient descent(signGD)-based neuronal dynamics, a design optimization-theoretically extended from Fig.~\ref{fig:dynamics_interpretation}.}
\label{fig:signgd_based_neuron_concept}
\end{center}
\vskip -0.45in
\end{figure}

\section{SignGD-based Neuronal Dynamics}

Spiking neurons, with their optimizer form as the subgradient method, are difficult to approximate arbitrary nonlinear functions. Our framework interprets that a spike train delivers (sub)-gradient information. In detail, a binary spike transfers an activation update step between spiking neurons, and the update step is gradient information $\nabla_y\mathcal{L}(y;x)$ of a certain objective function $\mathcal{L}(y;x)$. However, binarity of a spike $s(t)$ constrains the space of (sub)-gradient $\nabla_y\mathcal{L}(y;x)$ to a bivariate function of $s(t)$ and $y$, thereby limiting the space of objective function $\mathcal{L}(y;x)$. For example, in Theorem~\ref{thm:if_rate}, $\mathcal{L}(y;x)$ is a form of (ReLU + regularizer) since the subgradient $\nabla_y\mathcal{L}(y;x) = y - s(t) =  y -\mathbb{H}(\frac{R}{\theta_{th}} \tilde{x}(t) - y)$. To support arbitrary nonlinear functions, the binarity of a spike should not  constrain the space of objective functions. At the same time, the spike-based computational characteristic should be preserved.  This implies that the dynamics of the optimizer form of a new spiking neuron should be different from the subgradient method. 

We thus present a new sign gradient descent (signGD) -based neuronal dynamics that can (i) easily approximate a larger space of nonlinear functions and (ii) accelerate SNN inference. Our key approaches are two-fold: First, we apply sign gradient descent (signGD)~\cite{bernstein2018signsgd} instead of the subgradient method to design neuronal dynamics. Second, we generalize the learning rate schedule of the optimizer form of neuronal dynamics. 

signGD~\cite{bernstein2018signsgd} is a distributed optimization method that cuts down the network communication cost of gradient tensor. Each node transmits only the sign of the gradient instead of the exact (or compressed) gradient to the parameter server.
We apply signGD to design new neuronal dynamics and a coding scheme. 
It preserves the SNN's key characteristic of binary spike-based communication since the optimizer form has a binary update of $\{-1, 1\}$, which easily transforms into a $\{0,1\}$ spike.
Thus, given arbitrary objective function, a binary spike can represent an activation update step, which is the \emph{sign} of the (sub)-gradient.

The learning rate schedule is important for the convergence speed of signGD since the learning rate solely determines the update size. It is difficult to determine optimal step sizes for SNNs, which perform neuron-wise convex optimization, different from a single convex objective. For instance, schedules with fast decay make approximation faster for neurons in frontal SNN layers but slower for lateral SNN layers,  since neurons in lateral SNN layers receive accurate gradients in later time steps. Thus, learning rate scheduling should consider the approximation speed of entire SNN neurons end-to-end. 
We hence generalize the learning rate schedule of the optimizer form of signGD-based neuron and empirically search local optimum. Specifically, we define our signGD-based neuron and coding scheme as follows.
\begin{definition}
\label{def:signed_schedule_coding}
\textbf{(Signed schedule coding)} Let a spike train $s(t) \in \{0,1\}$, a step size schedule $\eta(t) \in \mathbb{R}^{+}$ for $t \in \mathbb{N}$ and $y(0) = 0$. Signed schedule coding with $\eta(t)$ decodes an activation $y(t)$ from $s(t)$ as
\begin{equation}
y(t) = y(t-1) - \eta(t) \big(2\cdot s(t)- 1\big)
\end{equation}
\end{definition}
This coding scheme interprets a binary spike $s(t) \in \{0,1\}$ as a sign information $\{-1, 1\}$ with a simple mathematical formula $2s(t) - 1$. We list three spike encoding algorithms for our signed schedule coding scheme in Appendix~\ref{sec:signsgd_spike_encoding}.
\begin{definition}
\label{def:signgd_dynamics}
\textbf{(SignGD-based neuronal dynamics)} Let a smooth 
objective function $\mathcal{L}(y;x_1,\cdots,x_d): \mathbb{R} \times \mathbb{R}^d \to \mathbb{R}$, positive coefficients $\alpha_{i}(t) \in \mathbb{R}^{+}$, $\beta_{i}(t) \in \mathbb{R}^{+}$, and step size schedule $\eta(t) \in \mathbb{R}^{+}$ for $i=1,2$ and $t \in \mathbb{N}$. Suppose an influx current $I^{(k)}(t)$ of $k$-th operand in time $t$ is a sum of weighted spike trains, i.e., for real weights $W_{i}$,
\begin{equation}
    I^{(k)}(t) = \sum_{j=1}^N W_jI_j^{(k)}(t), \quad I_{j}^{(k)}: \mathbb{N} \to \{0,1\} 
\end{equation}
We denote $W = \sum_{i=1}^{N}W_i$. Then the signGD-based neuronal dynamics over two internal variables $u(t) \in \mathbb{R}$ and  $v(t) = \big(v^{(1)}(t), \cdots, v^{(k)}(t),\cdots, $ 
$ v^{(d)}(t)\big) \in \mathbb{R}^d$ is 
\begin{align}
v^{(k)}(t+1) &= \alpha_1(t)v^{(k)}(t) - \alpha_2(t+1) (2I^{(k)}(t+1) - W)\nonumber\\
s(t) = \mathbb{H}&\bigg(\nabla_y\mathcal{L}\big(\frac{\eta(t-1)}{\beta_2(t-1)}u(t);\frac{\eta(t)}{\alpha_2(t)}v(t)\big)\bigg) \label{eq:signsgd_fire}\\
u(t+1) &= \beta_1(t)u(t)  
-\beta_2(t) (2s(t) - 1)\label{eq:signsgd_reset}
\end{align}
\end{definition}
The dynamics coefficients $\alpha_i(t)$ and $\beta_i(t)$ with $i = 1,2$ are time-scheduled, and remains constant irrespective of the input current $I^{(k)}(t)$ or any internal variables $u(t)$ or $v^{(k)}(t)$. Hence, we can pre-compute and load these values up to a specific time-step $T$.
The term $W$  translates a weighted sum of $\{0,1\}$ spikes into a weighted sum of sign gradients $\{-1,1\}$. In practice, $W$ is pre-computed in three steps: (i) stimulate all neurons to spike for a single step and record the influx current $I^+$, (ii) depress all neurons to not spike for a single step and record the influx current $I-$, (iii) use $I^+$ and $I^-$ to compute $W$ neuron-wise. Figure~\ref{fig:signgd_based_neuron_concept} illustrates the optimization-theoretic abstraction of our spiking neuronal dynamics. Below, we show that our signGD-based neuronal dynamics is equivalent to the signGD algorithm.

\begin{theorem}
\label{thm:signsgd}
Let an input activation $\tilde{x}(t)$ be signed schedule coded over an $N$ weighted input spike trains, i.e., 
\begin{align*}
    &\tilde{x}^{(k)}(t) = \tilde{x}^{(k)}(t-1 ) - \sum_{i=1}^N W_i\bigg(  \eta(t)(2I_i^{(k)}(t) -1)\bigg)
\end{align*}
Output $\tilde{f}(t)  = \tilde{f}(t-1) - \eta(t) \big(2\cdot s(t)- 1\big)$ is signed schedule coded with $s(t)$ of signGD-based neuronal dynamics.\\
If $\alpha_1$, $\alpha_2$, $\beta_1$, $\beta_2$ and $\eta$ satisfies $\eta(1) = \alpha_2(1) = \beta_2(1)$ and
\begin{align}
\frac{\eta(t)}{\eta(t-1)} = \frac{\beta_1(t)\beta_2(t)}{\beta_2(t-1)} = \frac{\alpha_2(t)}{\alpha_1(t-1)\alpha_2(t-1)} \label{eq:main_signsgd_coeff_condition}
\end{align}
Then the dynamical system of $\tilde{f}(t)$ 
is equivalent to the sign gradient descent method formulated as, (Proof at~\ref{thm:signgd_dynamics})
\begin{align}
\tilde{f}&(t) = \tilde{f}(t-1) - \eta(t) \text{sgn} (\nabla_y\mathcal{L}(\tilde{f}(t-1); \tilde{x}(t))) \label{eq:main_signsgd_subg}
\end{align}
\end{theorem}
We now demonstrate that our neuron can support spike train-based evaluation of novel nonlinear functions.
\subsection{Single-operand Nonlinearities}
We first approximate single-operand nonlinear functions: ReLU, LeakyReLU, and GELU.
We utilize a simple  objective function $\mathcal{L}(y;x) = \frac{1}{2}\Vert y - f(x)\Vert^2$, which is smooth and convex over $y$. 
Note that the choice of nonlinear function only affects the firing mechanism (Eq.~\ref{eq:signsgd_fire}) of our signGD-based neuronal dynamics. Also, we approximate the exact ReLU instead of the clipped ReLU1.
Figure~\ref{fig:unary_nonlinearity_signgd} in Appendix shows toy experiments validating that our signGD-based neuron approximates the unary nonlinearities accurately, and the learning rate schedule affects the convergence speed.
\begin{corollary}
\label{cor:unary_signsgd}
SignGD-based neuronal dynamics~(Def.~\ref{def:signgd_dynamics}) satisfying Eq.~\ref{eq:main_signsgd_coeff_condition} and $\eta(t) = \alpha_2(t) = \beta_2(t)$ is equivalent to signGD (Eq.~\ref{eq:main_signsgd_subg}) if $s(t)$ (Eq.~\ref{eq:signsgd_fire}) satisfies, (Proof at~\ref{thm:signgd_unary})
\begin{itemize}
\item \textbf{(ReLU)} $\mathcal{L}(y;x) = \frac{1}{2}\Vert y - \text{ReLU(x)} \Vert^2$ and $s(t) = \mathbb{H}\big(v(t)\big)\mathbb{H}\big(u(t) - \beta_1(t)v(t)\big) + \mathbb{H}\big(-v(t)\big)\mathbb{H}\big(u(t)\big)$
\item \textbf{(Sigmoid approximation of GELU)}~\cite{hendrycks2016gelu}) $\mathcal{L}(y;x) = \frac{1}{2}\Vert y - \frac{x}{1+e^{-1.702x}} \Vert^2$ and $s(t) = \mathbb{H}\big((1 + (e^{-1.702})^{v(t)}) u(t) - \beta_1(t) v(t) \big)$
\item \textbf{(LeakyReLU)} $\mathcal{L}(y;x) = \frac{1}{2}\Vert y - \text{LeakyReLU(x, $\delta$)} \Vert^2$, where $\delta$ is the negative slope, and $s(t) = \mathbb{H}(v(t))\mathbb{H}\big(u(t) -\beta_1(t)v(t)\big) + \mathbb{H}(-v(t))\mathbb{H}\big(u(t) - \delta \beta_1(t)v(t)\big) $
\end{itemize}
\end{corollary}

\subsection{Multi-operand Nonlinearities}
\label{sec:multioperand_nonlinearities}
In this section, we approximate two tensor operators, max pooling and layer normalization~\cite{ba2016layernorm}, with binary-input signGD-based neuronal dynamics ($d=2$). We also empirically validate the multi-operand nonlinearity approximation with our neuron in Appendix~\ref{sec:layernorm_speed} and Figure~\ref{fig:binary_nonlinearity_signgd}.
\begin{figure}[ht]
\vskip -0.0in
\begin{center}
\centerline{\includegraphics[width=0.85\columnwidth]{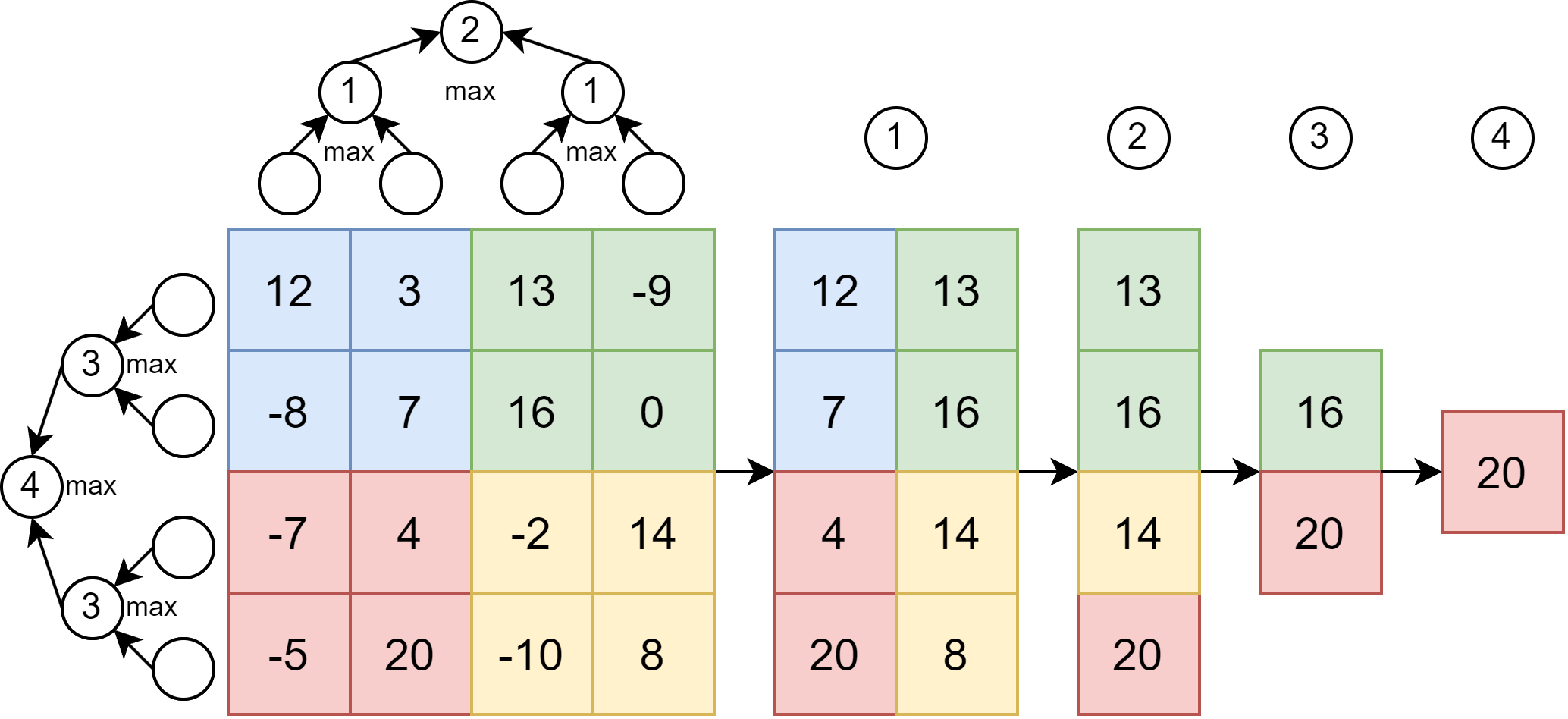}}
\vskip -0.1in
\caption{An example of decomposing max pooling of 4x4 window with binary-input maximum operators.}
\label{fig:maxpool_decomposition}
\end{center}
\vskip -0.4in
\end{figure}
\subsubsection{Max Pooling}
The max pooling operator downsamples a feature map with a maximum value for each patch in CNN~\cite{He2015ResNet,Simonyan2014VGG,tan2019efficientnet}.
Max pooling approximation in SNN has been a long-standing problem ~\cite{gaurav2022spikingmaxpool} since it should anticipate the maximum of spike trains ahead of time. A common strategy is to replace it with average pooling~\cite{Sengupta2018SNNVGGResNet, li2021freelunch}, which degrades the performance~\cite{Rueckauer2017ConversionOC}. Other works compute the instantaneous maximum  instead of exact maximum~\cite{gaurav2022spikingmaxpool,guo2019overheadfreemp}, use TTFS coding~\cite{stanojevic2023ttfsmp}, or sacrifice SNN's characteristics, e.g., spike-based~\cite{li2022quantizationframework} or event-driven computation~\cite{lu2022linearlif}.

 Our signGD-based neuron can approximate the exact maximum  in an event-driven manner, thereby supporting the max pooling operation.
If a neuron can evaluate the maximum over two input spike trains, then we can construct a binary computational tree over row and column dimension of the pooling kernel, as in Figure~\ref{fig:maxpool_decomposition}. 
If a kernel size is $K_r\times K_c$, the tree depth of max neurons is $\lceil\log_2 K_r\rceil \times \lceil\log_2 K_r\rceil$. 
We approximate maximum function over two activations with a binary-input signGD-based neuron defined as follows.
\begin{corollary}
\label{cor:maxpool_signsgd}
SignGD (Eq.~\ref{eq:main_signsgd_subg}) with $\mathcal{L}\big(y;x_1,x_2\big) = \frac{1}{2}\Vert y - \max(x_1, x_2) \Vert^2$ is equivalent to signGD-based neuronal dynamics~(Def.~\ref{def:signgd_dynamics})
satisfying Eq.~\ref{eq:main_signsgd_coeff_condition}, $\alpha_2(t) = \beta_2(t)$, 
\begin{equation*}
\begin{split}
s(t) =~&\mathbb{H}(v^{(1)}(t) - v^{(2)}(t))(u(t) -\beta_1(t) v^{(1)}(t))\\
& + \mathbb{H}(v^{(2)}(t) - v^{(1)}(t))(u(t) - \beta_1(t)v^{(2)}(t))
\end{split}
\end{equation*}    
\end{corollary}

\subsubsection{Layer Normalization}
Normalization techniques stabilize and accelerate the training of DNN models~\cite{huang2023normalization}. Batch normalization~\cite{ioffe2015batchnorm} is supported for SNN inference since batch statistics are fixed in the inference phase; hence, it becomes an affine operator~\cite{li2021freelunch}. It is not the case for layer normalization~\cite{ba2016layernorm}, an operator employed at high-performance DNNs, e.g., Transformer~\cite{vaswani2017transformer},  ConvNext~\cite{liu2022convnext}, and MLP-Mixer~\cite{tolstikhin2021mlpmixer}. Layer normalization in SNN should evaluate data instance-wise channel statistics across multiple spike trains ahead of time. This requires event-driven spike-based computation of intricate nonlinearities, e.g., variance and inverse square root, which was infeasible with prior spike neurons.

\begin{figure}[ht]
\vskip -0.0in
\begin{center}
\centerline{\includegraphics[width=0.9\columnwidth]{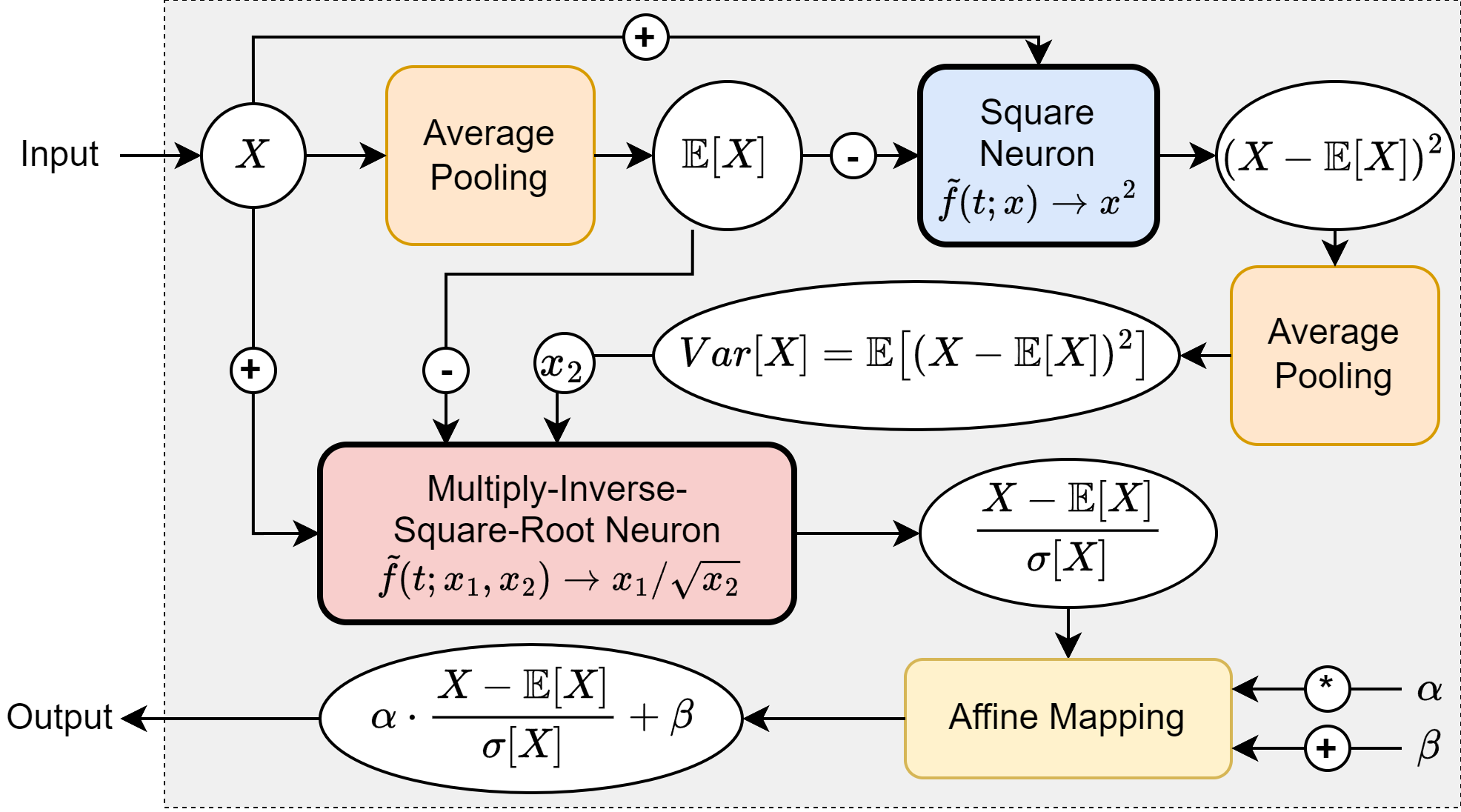}}
\vskip -0.1in
\caption{Decomposing layer normalization~\cite{ba2016layernorm} into operators that can be approximated with signGD-based spiking neurons (Corollary~\ref{cor:square_signsgd},~\ref{cor:sqrt_inverse_mul_signsgd}) and affine connections.}
\label{fig:layer_norm_decomposition}
\end{center}
\vskip -0.4in
\end{figure}

To approximate layer normalization with our signGD-based neuron, we decompose the operator into a computational graph of linear operators and two nonlinearities, as in Figure~\ref{fig:layer_norm_decomposition}. Linear operators can be implemented in SNN or be fused to other linear layers. Two key nonlinearities are square function $h_1(x) = x^2$ and multiply-inverse-sqrt function $h_2(x_1, x_2) = \frac{x_1}{\sqrt{x2}}$. We derive signGD-based neurons that can approximate these nonlinearities as follows. 
\begin{corollary}
\label{cor:square_signsgd}
SignGD (Eq.~\ref{eq:main_signsgd_subg}) with $\mathcal{L}(y;x) =  \frac{1}{2}\Vert y - x^2 \Vert^2$ is equivalent to the signGD-based neuronal dynamics~(Def.~\ref{def:signgd_dynamics})
satisfying Eq.~\ref{eq:main_signsgd_coeff_condition}, $\eta(t) =\beta_2(t) = \alpha_2(t)$ and $s(t) = \mathbb{H}\big(u(t) - \beta_1(t)v(t)^2\big)$ (Proof at~\ref{thm:signgd_square})
\end{corollary}
\begin{corollary}
\label{cor:sqrt_inverse_mul_signsgd}
SignGD (Eq.~\ref{eq:main_signsgd_subg}) with $ \mathcal{L}(y;x_1,x_2) = \Vert y - \frac{x_1}{\sqrt{x_2}} \Vert^2$ is equivalent to the signGD-based neuron~(Def.~\ref{def:signgd_dynamics})
satisfying Eq.~\ref{eq:main_signsgd_coeff_condition}, $\eta(t) = \alpha_2(t) = \beta_2(t)$, (Proof at~\ref{thm:signgd_misr})
\begin{equation*}
\begin{split}
s(t) = &~\mathbb{H}\big(u(t)\big)\mathbb{H}\big(v^{(1)}(t)\big)\mathbb{H}\big(v^{(2)}(t)u(t)^2 - v^{(1)}(t)^2\big) 
 \\
+~ \mathbb{H}&\big(-u(t)\big)\mathbb{H}\big(-v^{(1)}(t)\big)\mathbb{H}\big(v^{(1)}(t)^2 - v^{(2)}(t)u(t)^2\big)  
\\
+~ \mathbb{H}&(u(t))\mathbb{H}(-v^{(1)}(t))
\end{split}
\end{equation*}    
\end{corollary}

\begin{table*}[t]
\vskip -0.1in
\caption{Comparing ANN-to-SNN conversion performance on ImageNet~\cite{deng2009imagenet} models. \emph{Exact Arch} means the trained ANN architecture is identical to its original paper, without any customization or tailoring. \emph{No Spike-aware Activation Func} means ReLU functions in ANN architecture are not replaced with spike-aware functions before training, e.g., QCFS~\cite{bu2021optimalannsnn}, SlipReLU~\cite{Jiang2023SlipReLU}. For our signGD-based neuron, we use $\eta(t) = \frac{5.0}{t+1}$ for ConvNext and MLP-Mixer, and $\eta(t) = 0.15 \cdot 0.965^t$ for the others. Results of RTS~\cite{Deng2021OptimalConversionTheory} are from \cite{li2021freelunch}. }
\label{tab:imagenet_comparison}
\vskip 0.1in
\begin{center}
\begin{small}
\begin{tabular}{lccccccc}
\toprule
\multirow{2}{*}{Methods} & Exact& No Spike-aware &\multirow{2}{*}{ANN Acc.} & \multicolumn{4}{c}{Simulation time-steps} \\
 &Arch.& Activation Func. & &   T = 32 &T = 64 & T = 128 &T = 256 \\
\midrule
\multicolumn{8}{c}{ResNet-34~\cite{He2015ResNet} ImageNet}\\
\midrule
TSC~\cite{han2020deeptsc}    & \Redlargex & \Greencheck & 70.20 &  - & - & - & 55.65  \\
RTS~\cite{Deng2021OptimalConversionTheory}
&\Redlargex & \Redlargex & 75.66 &  0.09 & 0.12 & 3.19 & 47.11 \\
SNNC-AP~\cite{li2021freelunch}      
& \Redlargex & \Greencheck &75.66 &  64.54 & 71.12 & 73.45 & 74.61   \\
QCFS~\cite{bu2021optimalannsnn}      
& \Redlargex & \Redlargex &74.32&  \textbf{69.37} & 72.35 & 73.15 & 73.37  \\
SlipReLU~\cite{Jiang2023SlipReLU}  
&\Redlargex & \Redlargex & 75.08 &  66.61 & 72.71 & 74.01 & - \\
SRP~\cite{Hao2023ResidualMembranePotential}       
&\Redlargex & \Redlargex & 74.32 &  68.40 & 68.61 & - & -  \\
Ours (with Max Pooling)       
&\Greencheck & \Greencheck & 73.30 &  58.09 & 72.38 & 73.31 & 73.29  \\
Ours (without Max Pooling)       
&\Redlargex & \Greencheck & 75.65 & 59.85 & \textbf{75.34} & \textbf{75.67} & \textbf{75.65} \\
\midrule
\multicolumn{8}{c}{VGG-16~\cite{Simonyan2014VGG} ImageNet}\\
\midrule
TSC~\cite{han2020deeptsc}    
& \Redlargex & \Greencheck &73.49 &  - & - & - & 69.71  \\
RTS~\cite{Deng2021OptimalConversionTheory}
& \Redlargex & \Redlargex &75.36 &  0.114 & 0.118 & 0.122 & 1.81 \\
SNNC-AP~\cite{li2021freelunch}      
& \Redlargex & \Greencheck &75.36 &  63.64 & 70.69 & 73.32 & 74.23   \\
SNM~\cite{Wang2022SignedNW}      
&\Redlargex & \Greencheck &73.18 &  64.78 & 71.50 & 72.86 & -   \\
QCFS~\cite{bu2021optimalannsnn}      
& \Redlargex & \Redlargex &74.29 &  68.47 & 72.85 & 73.97 & 74.22 \\
SlipReLU~\cite{Jiang2023SlipReLU}  
& \Redlargex & \Redlargex &71.99 &  67.48 & 71.25 & 72.02 & - \\
SRP~\cite{Hao2023ResidualMembranePotential}       
& \Redlargex & \Redlargex &74.29 &  \textbf{69.35} & 69.43 & - & - \\
Ours (with Max Pooling)       
&\Greencheck & \Greencheck & 73.36 &  38.08 & 67.04 & 71.33 & 71.50  \\
Ours (without Max Pooling)       
&\Redlargex & \Greencheck & 75.35 & 69.16 & \textbf{75.32} & \textbf{75.31} & \textbf{75.34}   \\
\midrule
\multicolumn{8}{c}{RegNetX~\cite{Radosavovic2020RegNetX} ImageNet}\\
\midrule
RTS~\cite{Deng2021OptimalConversionTheory}
& \Redlargex & \Redlargex & 80.02 &  0.218 & 3.542 & 48.60 & 71.22 \\
SNNC-AP~\cite{li2021freelunch}      
& \Redlargex & \Greencheck & 80.02 &  \textbf{55.70} & 70.96 & 75.78 & 77.50   \\
Ours (RegNetX-3.2GF) &\Greencheck & \Greencheck & 81.19 &  26.85 & \textbf{77.74} & \textbf{80.93} & \textbf{80.99}  \\
\midrule 
\multicolumn{8}{c}{New DNN architectures converted with Our signGD-based neuron} \\
\midrule
ResMLP-S24~\cite{Touvron2021ResMLP}&\Greencheck & \Greencheck & 80.76 & \textbf{72.94} & \textbf{76.91} & \textbf{77.99} & 78.04 \\
ConvNext-B~\cite{liu2022convnext}&\Greencheck & \Greencheck & 84.06 & 0.11 & 5.07 & 72.60 & \textbf{81.07}\\
MLP-Mixer-B32~\cite{tolstikhin2021mlpmixer}&\Greencheck & \Greencheck & 76.59 & 0.11 & 0.35 & 50.06 & 72.97\\
\bottomrule
\end{tabular}
\end{small}
\end{center}
\vskip -0.28in
\end{table*}
\section{Evaluations}
We demonstrate the practical effectiveness of our signGD-based neuronal dynamics in four-fold. First, we validate our support for diverse nonlinearities by converting new DNN architectures. Second, we compare the accuracy of converted ANNs with existing conversion techniques. 
Third, we verify our design choices through ablation studies. Finally, we visualize the effect of our technique on SNN inference speed.
We also conduct an energy consumption analysis of our technique in Appendix~\ref{sec:energy_consumption}. 
For a fair comparison, we generalize the learning rate schedule of the subgradient method-based neuron, which is formulated in Appendix~\ref{sec:lr_schedule_subgradient}. 
We detail our conversion technique and its spikingjelly~\cite{spikingjelly} implementation in Appendix~\ref{app:implementation_details}.

\subsection{Diversifying DNN Architecture Support}
To verify that our signGD-based neuronal dynamics can approximate diverse nonlinearities, we convert large-scale DNN architectures that prior works fail to convert. We experiment with the latest non-transformer architectures: MLP-Mixer~\cite{tolstikhin2021mlpmixer}, ResMLP~\cite{Touvron2021ResMLP}, ConvNext~\cite{liu2022convnext}, ResNet34, and VGG16 with max pooling. Note that GELU conversion should be supported for ResMLP, MLP-Mixer, and ConvNext. Layer normalization should be converted for MLP-Mixer and ConvNext. Table~\ref{tab:imagenet_comparison} shows that our signGD-based neuron enables conversion of these architectures for the first time. Converted ResNet and VGG with max pooling layer  achieves $>70\%$ accuracy in $T \ge 128$. MLP-Mixer and ConvNext take comparably more time-steps slow due to the layer normalization, as expected in Appendix~\ref{sec:layernorm_speed}. ResMLP approximates its ANN accuracy the fastest among all DNN models, achieving $72\%$ in $T = 32$ and $78\%$ in $T = 128$. 

\subsection{Comparison with Prior Works}
To show the practicality of our signGD-based neuron, we convert DNN to SNN with our neuron and compare its performance with prior techniques.
We use ImageNet~\cite{deng2009imagenet} and CIFAR~\cite{krizhevsky2009cifar} datasets.
For ImageNet models, we use pretrained weights of \cite{li2021freelunch} for a fair comparison. Table~\ref{tab:imagenet_comparison} shows that our converted SNNs accurately approximate the ANN performance on $T\ge 64$, achieving the best performance among all conversion techniques.
In $T\le32$, spike-aware ANN activation works~\cite{bu2021optimalannsnn,Jiang2023SlipReLU,Hao2023ResidualMembranePotential} achieve higher accuracy but significantly sacrifices accuracy in later time-steps $T \ge 64$ with $1.97-5.71\%$. Such an accurate SNN acceleration is possible since our neuron approximates the true ReLU instead of clipped ReLU and does not fine-tune ANN activations for early convergence. Also, for CIFAR models, Tables~\ref{tab:cifar10} and \ref{tab:cifar100} show that our converted SNN outperforms all the existing works in $T\ge 64$, and also in $T \le 32$ except for spike-aware ANN activation techniques. 
Our SNN performance approaches the ANN in $T\ge32$ with a small gap of $<1\%$. Note that prior studies used specialized techniques for low time-steps $T \le 32$, e.g., data-driven calibration~\cite{li2021freelunch} or spike-aware activation functions~\cite{bu2021optimalannsnn, Jiang2023SlipReLU}. We do not apply these methods to isolate and directly compare the effect of neuron models, and minimize the influence of specific conversion techniques.

 \begin{figure}[!ht]
\vskip -0.1in
\begin{center}
    \subfigure[Effect of neuronal dynamics and normalization. Exponential  schedule $\eta(t) = 0.95 * (0.15)^t$.]{
        \includegraphics[width=0.45\linewidth]{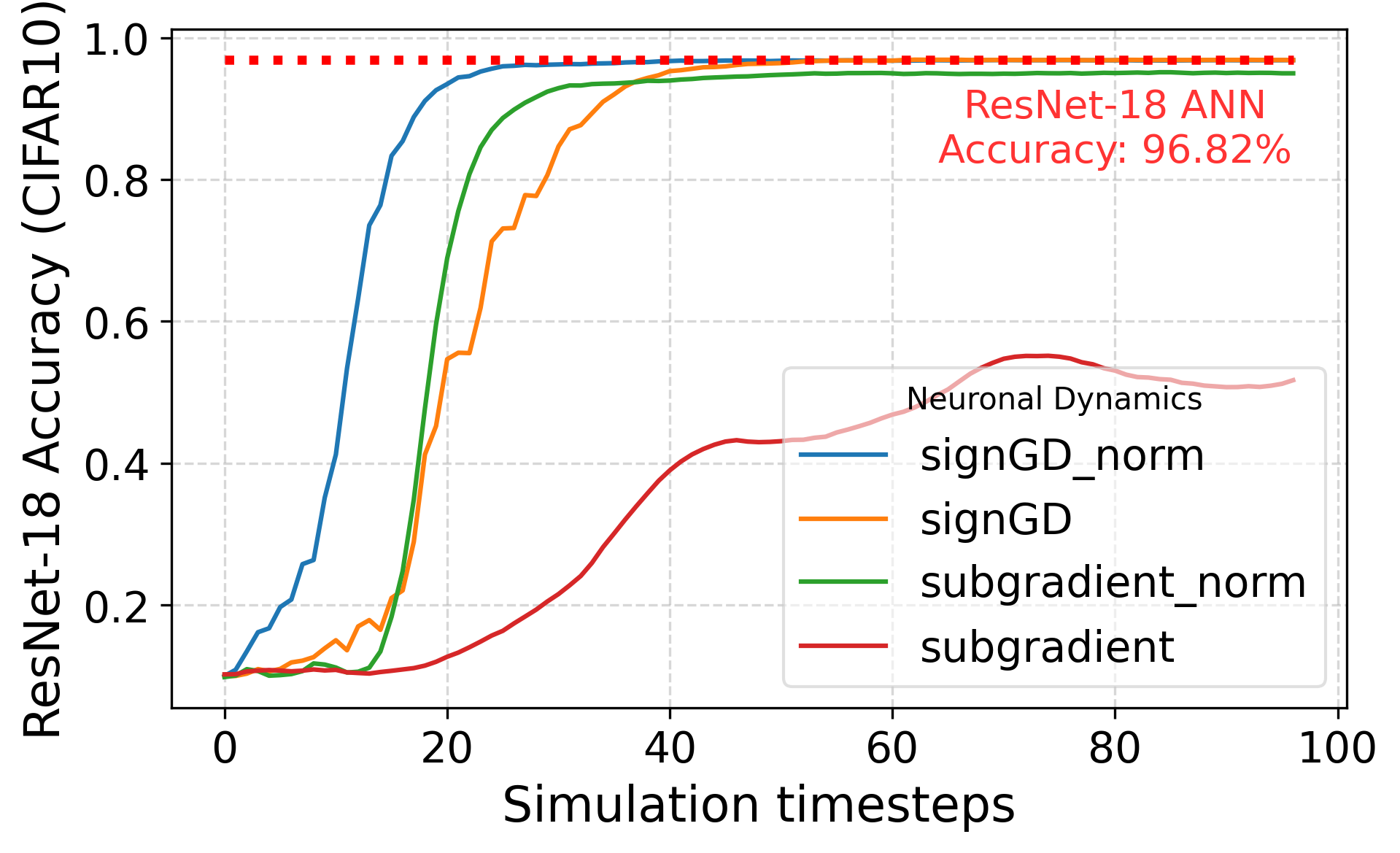}
        \label{fig:design_choice}
    }\hfill
    \subfigure[Effect of input encoding through timestep. SignGD-based neuron with $\eta(t) = \frac{1}{t+1}$.]{
        \includegraphics[width=0.45\linewidth]{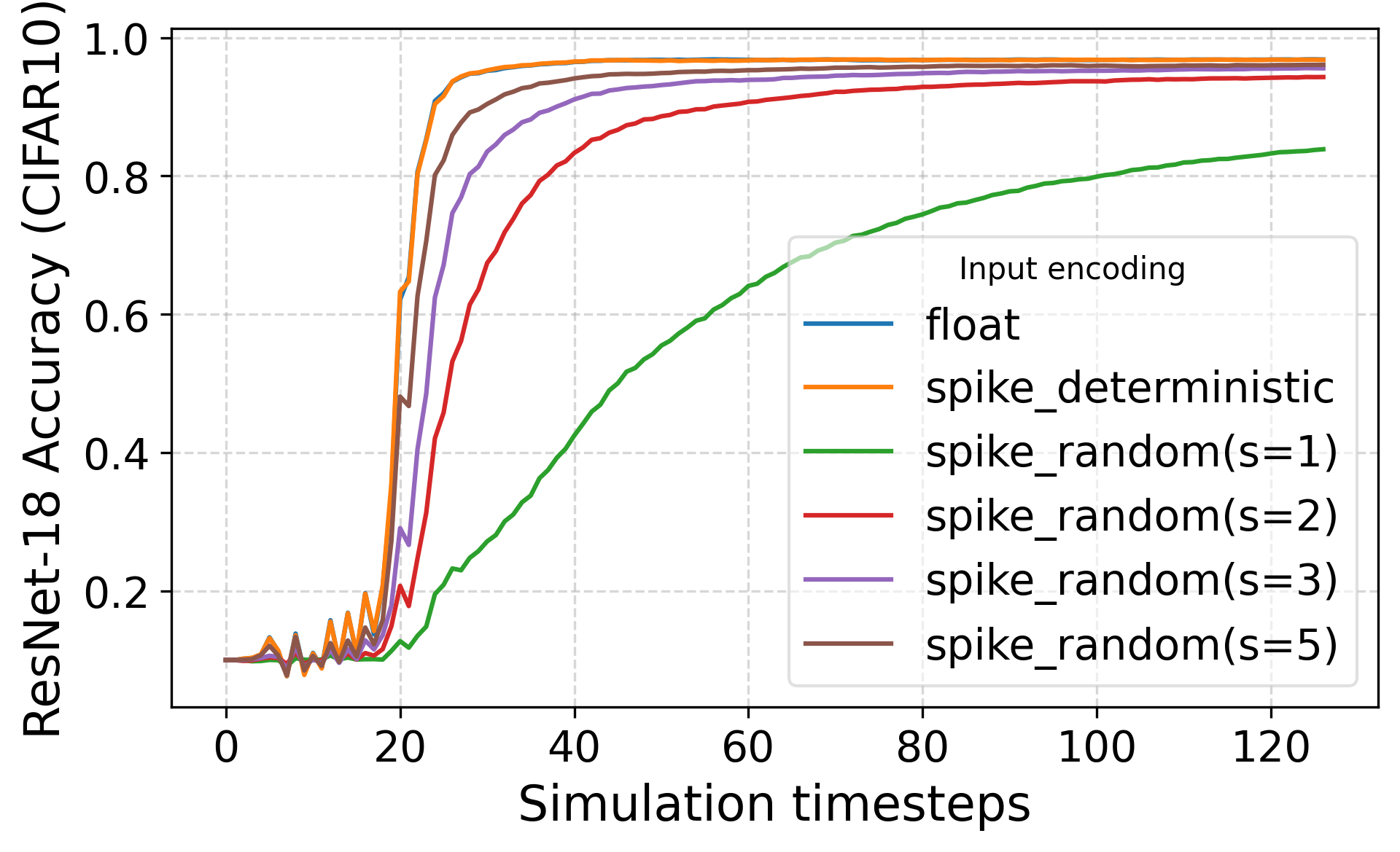}
        \label{fig:ablation_inputencoding}
    }
    \vfill
    \subfigure[Effect of LR on signGD-based neuron, inverse  schedule.]{
        \includegraphics[width=0.45\linewidth]{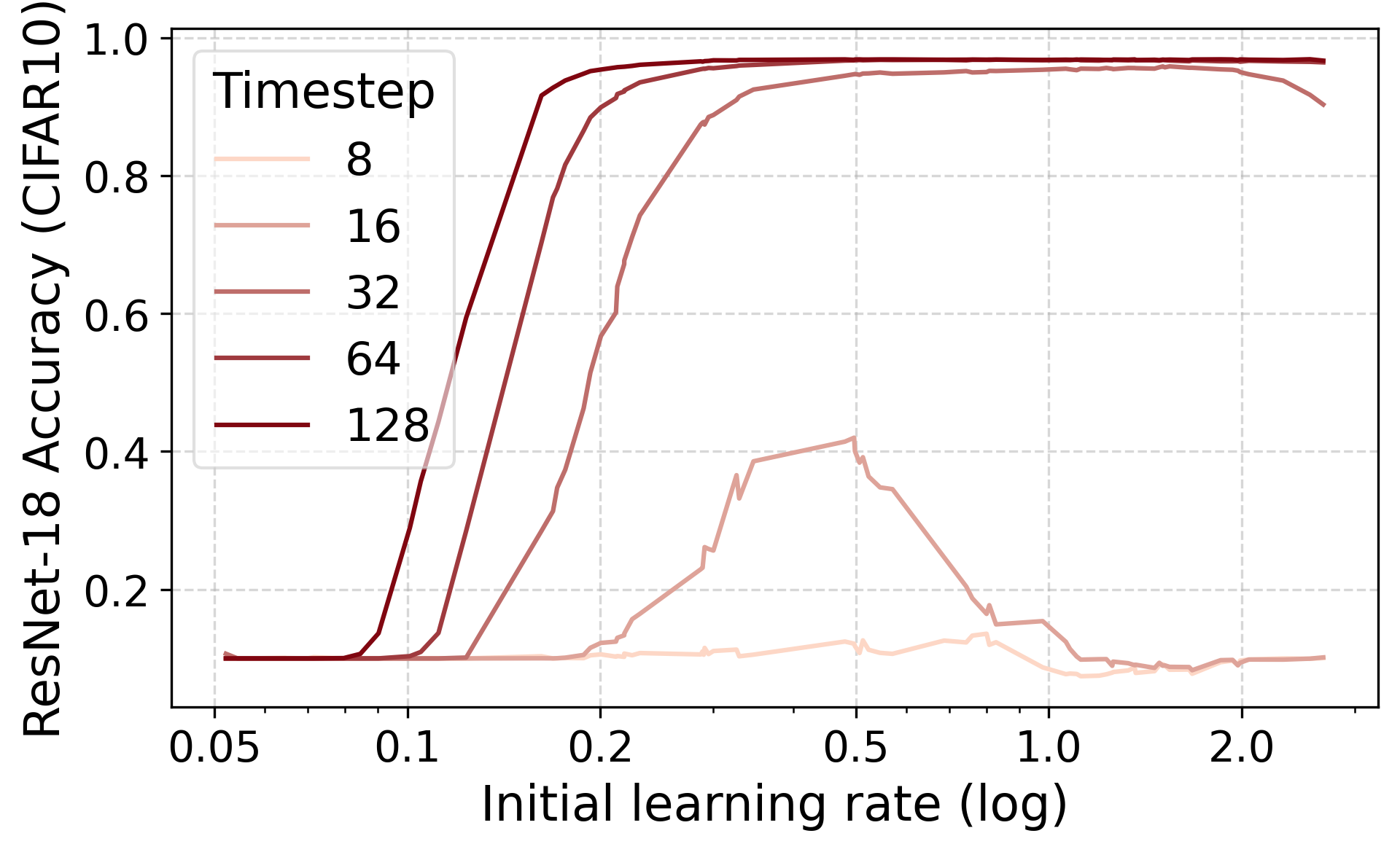}
        \label{fig:ablation_initlr_signsgd}
    }
    \hfill
    \subfigure[Effect of LR on subgradient-based neuron,  inverse schedule.]{
        \includegraphics[width=0.45\linewidth]{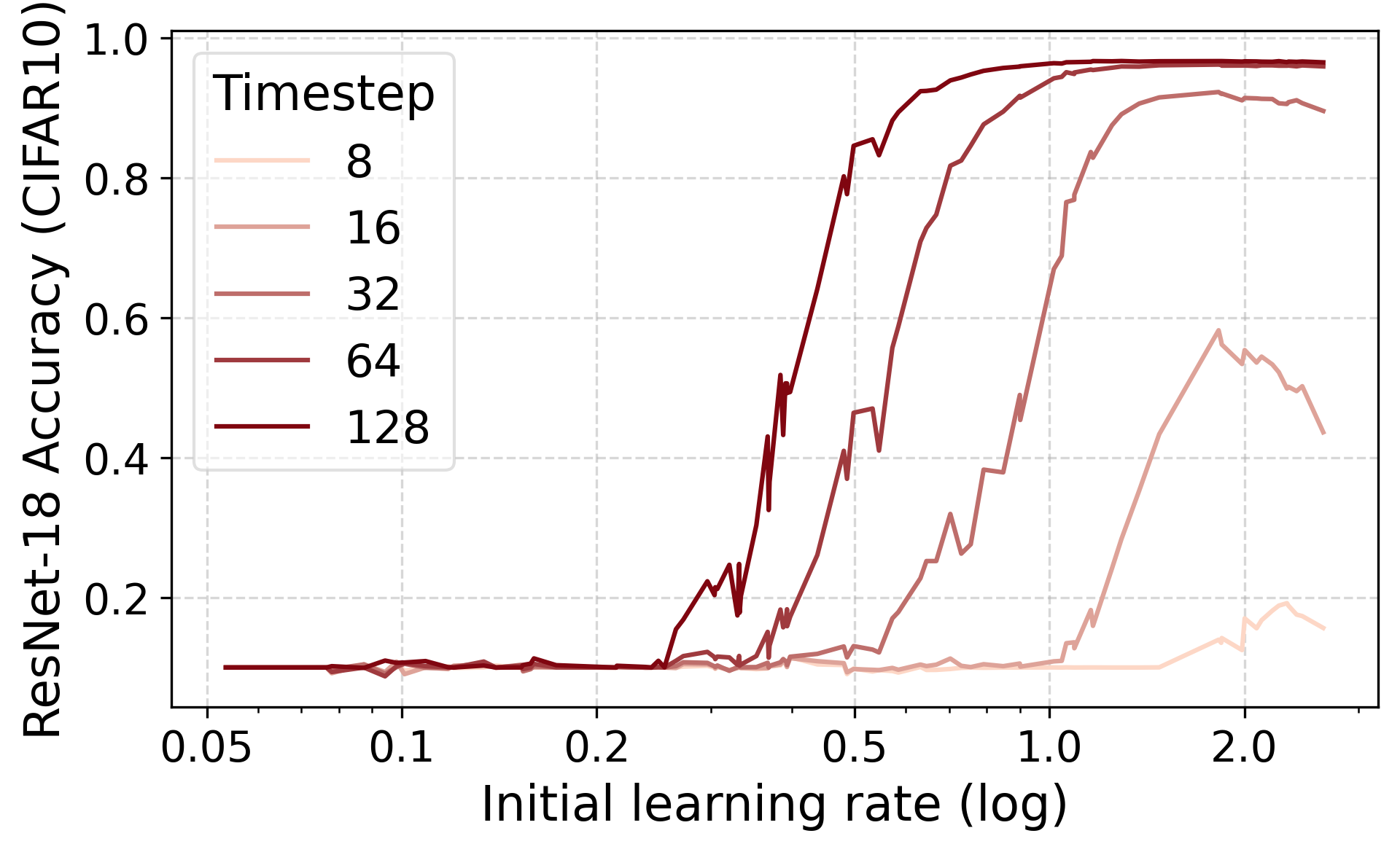}
        \label{fig:ablation_initlr_subgradient}
    }
    \vfill    
    \subfigure[Effect of LR and decay factor on signGD-based neuron, exponential schedule.]{
        \includegraphics[width=0.45\linewidth]{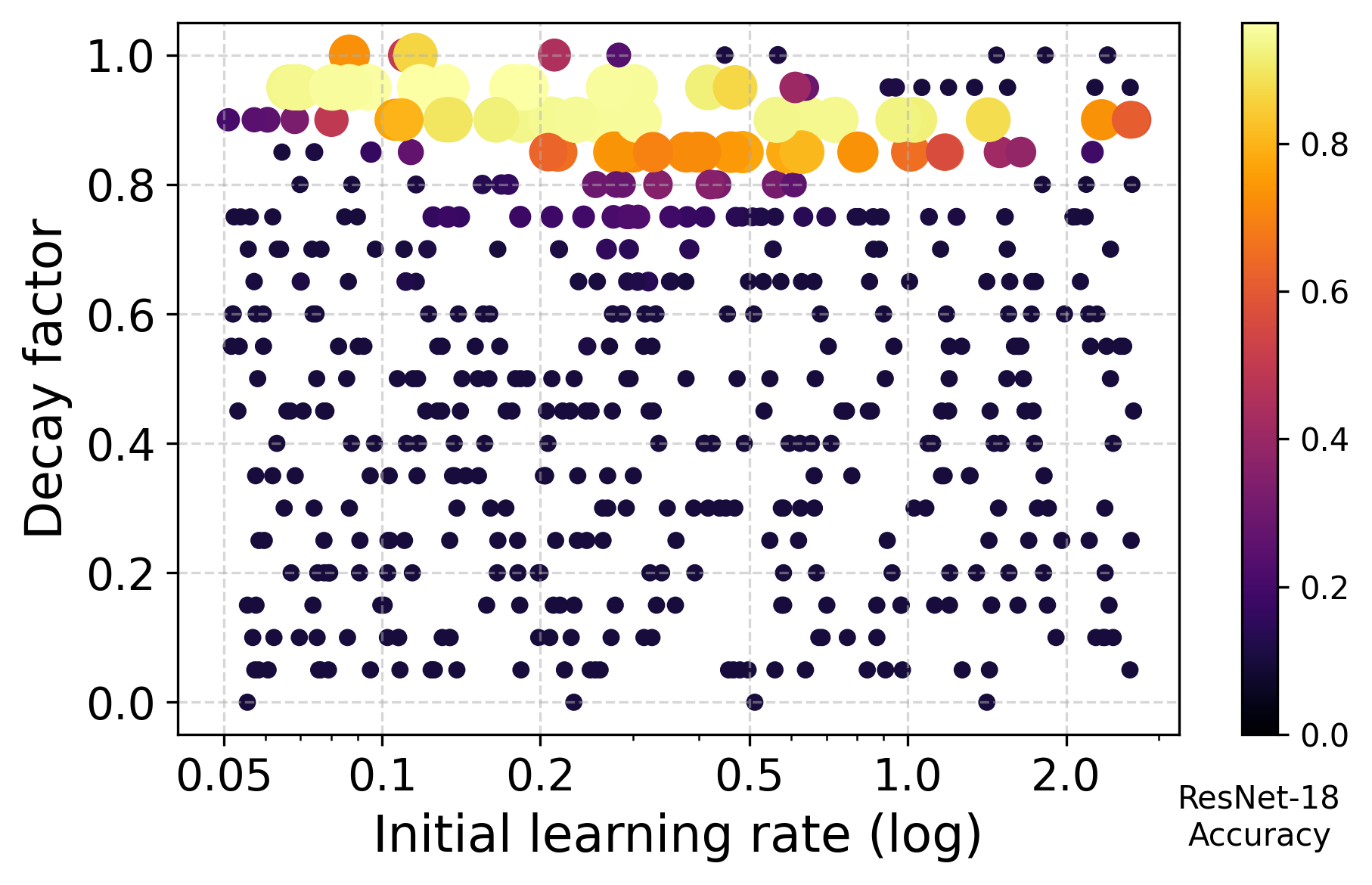}
        \label{fig:ablation_decayfactor_exp_signsgd}
    }\hfill    
    \subfigure[Effect of LR and decay factor on subgradient-based neuron, exponential schedule.]{
        \includegraphics[width=0.45\linewidth]{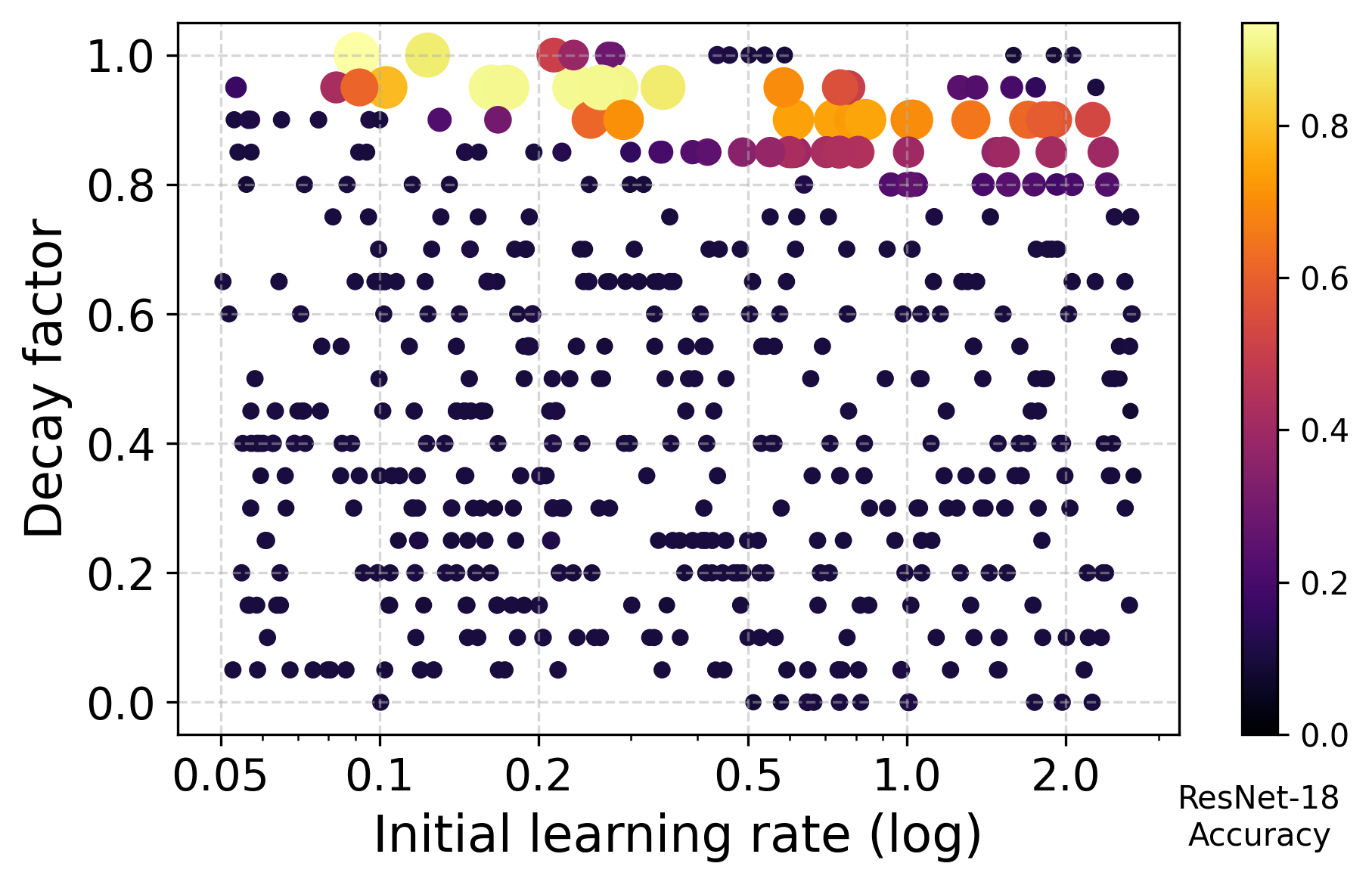}
        \label{fig:ablation_decayfactor_exp_subgradient}
    }
\vskip -0.1in
\caption{Ablation studies on hyper-parameters of signGD-based  (Def.~\ref{def:signgd_dynamics}) and subgradient-based neuronal dynamics (Def.~\ref{def:subgradient_based_neuron}). SNN inference accuracy is measured with ResNet-18 model~\cite{fang2021deepresidual} on CIFAR-10 dataset~\cite{krizhevsky2009cifar}.}
\label{fig:ablation_study}
\end{center}
\vskip -0.25in
\end{figure}
 \subsection{Ablation Study}
 
 \textbf{Effect of Neuronal Dynamics.} We first validate the acceleration effect of signGD-based dynamics in ANN-to-SNN conversion. In Figure~\ref{fig:design_choice}, we test two variables: the choice of neuronal dynamics and the normalization on ReLU layer. Results show that signGD-based neuron approximates noticeably faster than the subgradient-based neuron. Although the  normalization helps, it is auxiliary for our signGD-based neuron, different from  subgradient-based neuron. 
 The normalization is critical for the subgradient-based neuron since it approximates the clipped ReLU, not the true ReLU.

\textbf{Input Encoding Scheme.} 
Figure~\ref{fig:ablation_inputencoding} compares the effect of three input encoding schemes (See  Appendix~\ref{sec:signsgd_spike_encoding}.) on the SNN accuracy. To observe the impact of stochasticity, we also vary the parameter $s$ of the stochastic encoding. Results show that the performance gap is marginal between float and deterministic encoding, but it is noticeable with the stochastic encoding.
Marginal performance drop is desirable with deterministic encoding, since it is easier and more efficient to hardware-implement than float encoding. 

 \textbf{Case study: Inverse schedule.}
 To generalize the acceleration capability of our signGD-based dynamics, we fix the learning rate schedule and compare the converted SNN performance of signGD-based and subgradient-based neurons. As a case study, we first test an inverse LR schedule. In Figure~\ref{fig:ablation_initlr_signsgd}-\ref{fig:ablation_initlr_subgradient}, we measure the time-evolution of inference accuracy varying the initial learning rate (X-axis).  
 Figures show that compared to subgradient-based neuron, our signGD-based neuron provides (i) faster convergence, in a same initial learning rate, of converted SNN's inference accuarcy towards ANN, and (ii) wider coverage of initial learning rate that leads to high-performance SNN.

\textbf{Case study: Exponential schedule.} We also compare two neuronal dynamics in the case of an exponential schedule.
We evaluate the joint effect of two hyper-parameters, the initial learning rate and the decay factor. 
Similar to the inverse schedule case, Figures~\ref{fig:ablation_decayfactor_exp_signsgd}-\ref{fig:ablation_decayfactor_exp_subgradient} show that signGD-based neurons have a wider coverage of hyper-parameter combinations that achieve high inference accuracy. Also, the best accuracy is 
 $>4\%$ higher with signGD-based neuron;
The best accuracy in time-step $T=32$ is $91.4\%$ for subgradient-based neuron at $\eta(t) = 0.163\times(0.95)^t$ and $96.2\%$ for signGD-based neuron at $\eta(t) =0.177\times(0.95)^t$. 

\subsection{Visualization of Layer-wise Optimization Flow}
To qualitatively demonstrate the effect of our neuronal dynamics, Figure~\ref{fig:error_propagation} in the Appendix shows the time-evolution of layer-wise ANN-SNN activation error. 
The figure shows that our signGD-based neuron (i) reduces the layer-wise conversion error more rapidly and (ii) accumulates less conversion error through layers, compared to the subgradient-based neuron. In each layer, the error of signGD-based neurons converges faster than the subgradient-based neurons with the same learning rate schedule. In Figure~\ref{fig:visualize_if_rate} and \ref{fig:vis_exponential}, ANN-to-SNN conversion with subgradient-based neurons accumulates error through layers.
In contrast, the time-evolution of layer-wise error in signGD-based neurons is relatively uniform through layers. 
\section{Conclusion}
This paper constructs a novel optimization-theoretic framework on the discrete dynamical system of spiking neurons. Based on the framework, we develop new signGD-based neuronal dynamics that approximate various nonlinearities beyond ReLU, e.g., GELU, LeakyReLU, max pooling, and layer normalization. Using our neuron, we (i) successfully convert new high-performance DNN architectures for the first time, e.g., ConvNeXt, MLP-Mixer, and ResMLP,
and (ii) achieve state-of-the-art accuracy of converted ResNet34 and VGG16 in low time steps $T = 64$. We list our discussions and future works in the Appendix~\ref{sec:discussions_futureworks}.

\section*{Acknowledgements}
This work was partly supported by National Research Foundation of Korea (NRF) grant (No. 2022R1A2C3008495) and Institute of Information \& communications Technology Planning \& Evaluation (IITP) grant [NO.2021-0-01343-004, Artificial Intelligence Graduate School Program (Seoul National University)] funded by the Korean government (MSIT). We thank Minjae Kim for valuable feedback on writing the manuscript.

\section*{Impact Statement}
This paper presents a work that aims to broaden the applicability of technical advances in machine learning and biologically inspired artificial intelligence. A key advantage of SNN, which is the resource efficiency, would benefit a wide range of real-world AI applications. Resource-constrained mobile and IoT applications would especially benefit from this work by enabling high-performance intelligent inference with low-power pervasive sensing. Such an efficient AI inference can be widespread for socially underprivileged people who cannot afford large energy bills of embedded AI devices, e.g., immigrant parents getting AI assistance to nurture the linguistic skills of an under-developed child~\cite{kwon2022captivate}, or an AI chatbot to help the isolated elderly cope with their loneliness~\cite{valtolina2021charlie}. Another important societal advantage of energy-efficient AI is that it leads to environmental sustainability. The super-linear growth of AI techniques also scaled its ecological impact as an energy footprint~\cite{wu2022sustainableai}. Energy footprint directly leads to carbon emissions, i.e., carbon footprints, negatively impacting the atmosphere. SNN, with its orders-of-magnitude energy efficiency, can make DNN more environmentally sustainable while preserving the model's performance with this paper's approach.

Our theoretical findings may contribute to understanding the information dynamics of biological neuron models with refractory periods. Optimistically, this would foster neuroscientific studies in the human brain to develop diverse accessibility applications for the neurologically disabled, e.g., brain-computer interface for the motion-disabled~\cite{gao2003bci}, or a thought-controlled device to restore hand grasps of a paralyzed hand~\cite{pfurtscheller2003thought}. As a potential ethically negative impact, scientific advances in the overall neuroscience domain may enable decoding electronic signals of the human brain in the far future, thereby reading a person's mind. However, developing such a dangerous application is clearly out of the scope of this paper.
\bibliography{references}
\bibliographystyle{icml2024}

\newpage
\appendix
\onecolumn
\section{Related Works: Review on SNN Training.} 
\label{app:snn_training}
Training a SNN from scratch is an alternative of ANN-to-SNN conversion to build a high-performance SNN. Learning methodologies of SNN largely categorizes into two groups: (i) biologically plausible synaptic plasticity and (ii) backpropagation methods~\cite{li2021freelunch}. Spike timing-dependent plasticity(STDP), a representative synaptic plasticity mechanism, updates a synaptic weight based on the spike timing difference of two connected neurons~\cite{kheradpisheh2018stdp}. STDP falls behind backpropagation methods since it is a local unsupervised learning rule without globally guided error~\cite{Dong2022unsupervisedstdp}. Backpropagation methods for rate coding unfolds SNN backwards through entire time-steps~\cite{fang2021deepresidual, Zhu2022eventdriven}. To address the non-differentiablility of spiking mechanism, a common strategy is to use a surrogate gradient, which is a smooth relaxation of the spiking mechanism~\cite{Neftci2019SurrogateGL}. 
Still, they are computationally expensive and memory-intensive since they unfold a network backwards the entire time-steps~\cite{li2021freelunch}. A line of works efficiently approximate the costly backpropagation process with mixed-mode differentiations~\cite{zenke2021braininspiredlearning} or eligibility traces~\cite{bellec2020solutionLD, frenkel2022reckon}. 
Backpropagation methods for temporal coding computes backward only up to the last spike's timing, and hence it is event-driven and more resource-efficient~\cite{Zhu2022eventdriven}. 
Among these methods, surrogate gradient methods achieve state-of-the-art performance. Unfortunately, their performances still has a large margin with trained DNNs,~\cite{fang2021deepresidual,zhou2022spikformer,Yao2022AttentionSNN,Zhu2023SpikeGPT}, due to SNN-friendly architectural modification and gradient mismatch~\cite{Wang2023AdaptiveSmoothingGradient}.
\section{Discussions \& Future Works}
\label{sec:discussions_futureworks}
\textbf{Biophysical Neurons.}
Our theoretical framework explains the behavior of integrate-and-fire models. Since these models are a simplified discretization of biological neurons, we may extend our framework to continuous and biophysically-detailed models which incorporate known physiology and synaptic anatomy of a neuron cell~\cite{chizhov2021biophysicallydetailed}. 
For example, Hodgkin-huxley cell model~\cite{hodgkin1952huxleyneuron} is an experimentally discovered mathematical model of how biological neurons initiate and propagate spikes.
It is a continuous dynamical system of four ordinary differential equations over membrane potentials and conductances that cannot be solved analytically~\cite{gerstner2014neuronaldynamcis}. 
This model is a basis of biologically plausible neuron models~\cite{izhikevich2003simplemodel}.
To explain biologically realistic models, it may be helpful to extrapolate our framework to continuous dynamics. 



\textbf{Hardware Implementability.}
Our signGD-based neuron may need more complex hardware to realize than integrate-and-fire models.
Studies in efficient neuromorphic hardwares largely focus on LIF neuron~\cite{yang2020lifhw,davies2018loihi,akopyan2015truenorth}.
Computational neuroscientists invest on simulating more complex and biologically realistic models, e.g., SpiNNaker~\cite{furber2014spinnaker,mayr2019spinnakerv2}, 
and make them resource-efficient~\cite{ward2022beyondlif, rhodes2020realtimecortical}
or introduce new electrical elements for diversity~\cite{guo2021adjustablelif}.
Our neuronal dynamics is much more simple than biophysically detailed models~\cite{hodgkin1952huxleyneuron}, but slightly more complex than LIF neuron. For example, our neuron requires two or more internal variables, and GELU or layer norm approximation requires complex operators over internal variables, e.g., exponential or square function.  
Exponential operator is used in the computational neuroscience works to simulate exponential integrate-and-fire models~\cite{fourcaud2003eeif,brette2005adaptiveeif}.
 Square function is used in the izhikevich model~\cite{izhikevich2003simplemodel}. 
Thus, these operators are implemented in SpiNNaker~\cite{spinnaker}, and their efficient hardware implementations are also studied~\cite{srinivasan2022subthreshold,fang2022memristiveizhi}.
Note that operator requirements are dependent to spike mechanism; For instance, ReLU or LeakyReLU approximation with our neuron does not need such operators.

\textbf{Reset-to-Zero.} We focus on the reset-by-subtraction mechanism and does not explain the reset-to-zero mechanism. Reset-by-subtraction keeps the leftover information after thresholding operation, hence accelerating the approximation speed~\cite{rueckauer2016theory}.
In contrast, reset-to-zero discards the leftover information at reset~\cite{liu2022spikeconverter}, causing an unaccountable error for ReLU approximation~\cite{Rueckauer2017ConversionOC}. Two reset mechanisms are both valid discretizations of thresholding operation in continuous dynamics (See Appendix~\ref{sec:continuous_lif_dynamics}.), since both converge taken the continuum limit of time. Thus, it would be intriguing to interpret the reset-to-zero mechanism in an optimization-theoretic perspective, but gradient-based methods may not suit to formulate its behavior. 

\textbf{SNN Training.} We plan to investigate how our framework extends to a SNN training. Specifically, we newly interpret backpropagation~\cite{rumelhart1986backprop} and forward-forward algorithms~\cite{hinton2022forwardforward} with our optimization-theoretic perspective of neuronal dynamics. Conversely, we may analyze synaptic plasticity-based SNN learning rules~\cite{kheradpisheh2018stdp} with our theory, and discover methods to advance their training performance by applying state-of-the-art techniques of ANN training. It would also be appealing to apply short-term plasticity, e.g., habituation~\cite{glanzman2009habituationaplysia,zuo2017habituationpervoskite}, with our framework to facilitate few-shot adaptation of SNN.

\begin{figure}[!ht]
\vskip 0.0in
\begin{center}
    \subfigure[ReLU approximation.  $\eta(t) = \frac{1}{t+1}$]{
        \includegraphics[width=0.23\linewidth]{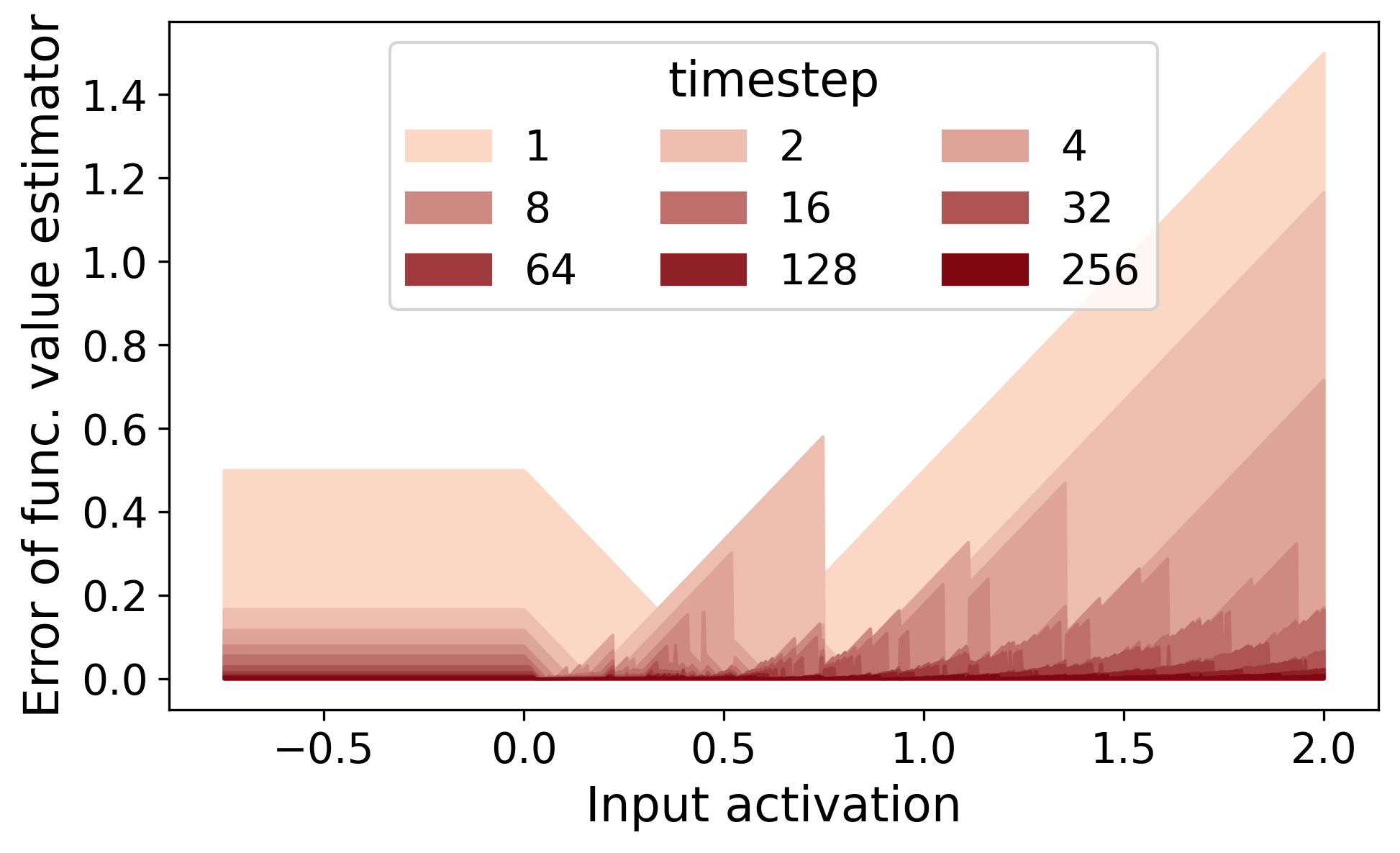}
    }
    \hfill
    \subfigure[ReLU approximation.  $\eta(t) = (0.95)^t$]{
        \includegraphics[width=0.23\linewidth]{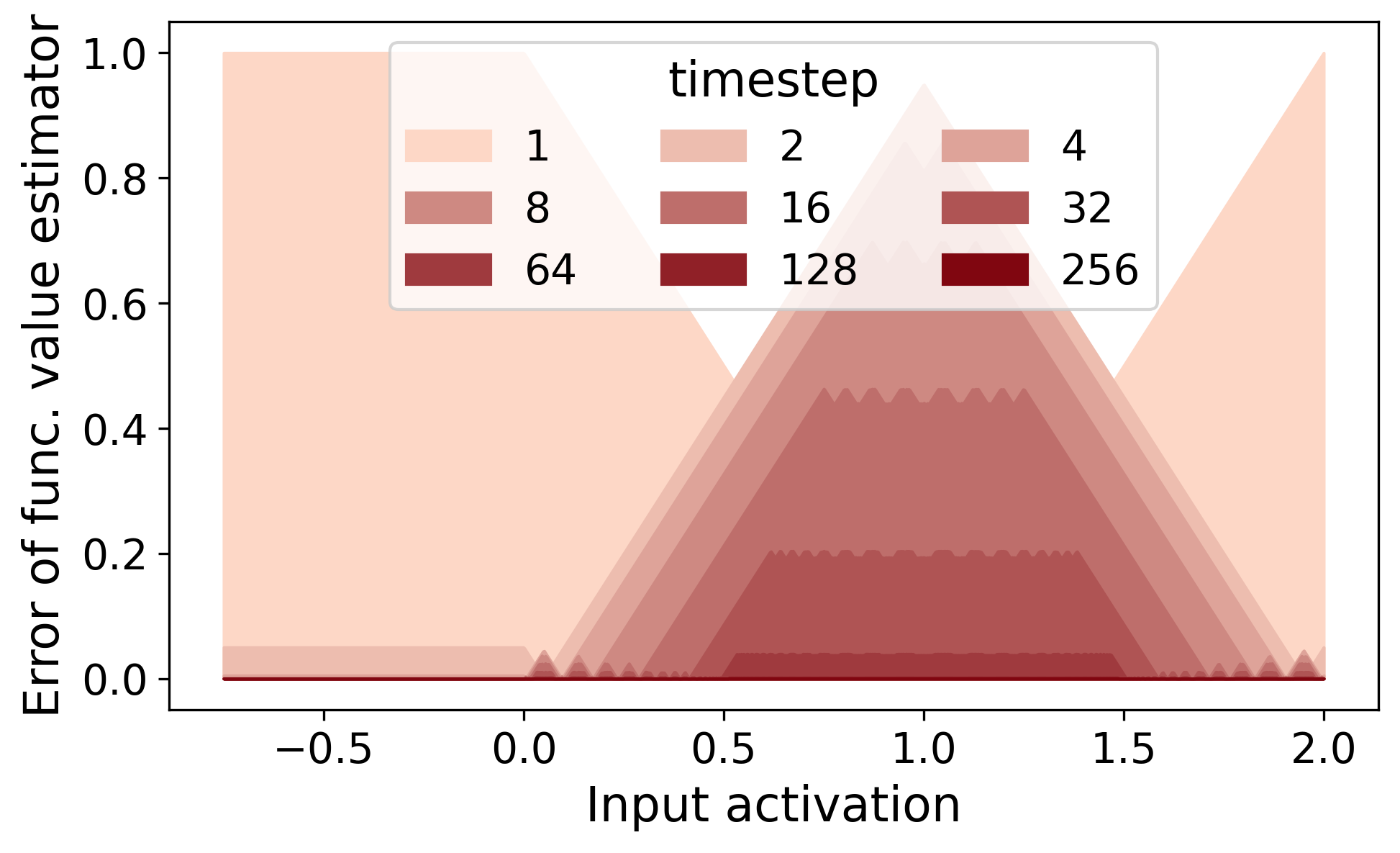}
    }\hfill
    \subfigure[LeakyReLU ($\alpha$=0.1) approx. $\eta(t) = \frac{1}{t+1}$]{
        \includegraphics[width=0.23\linewidth]{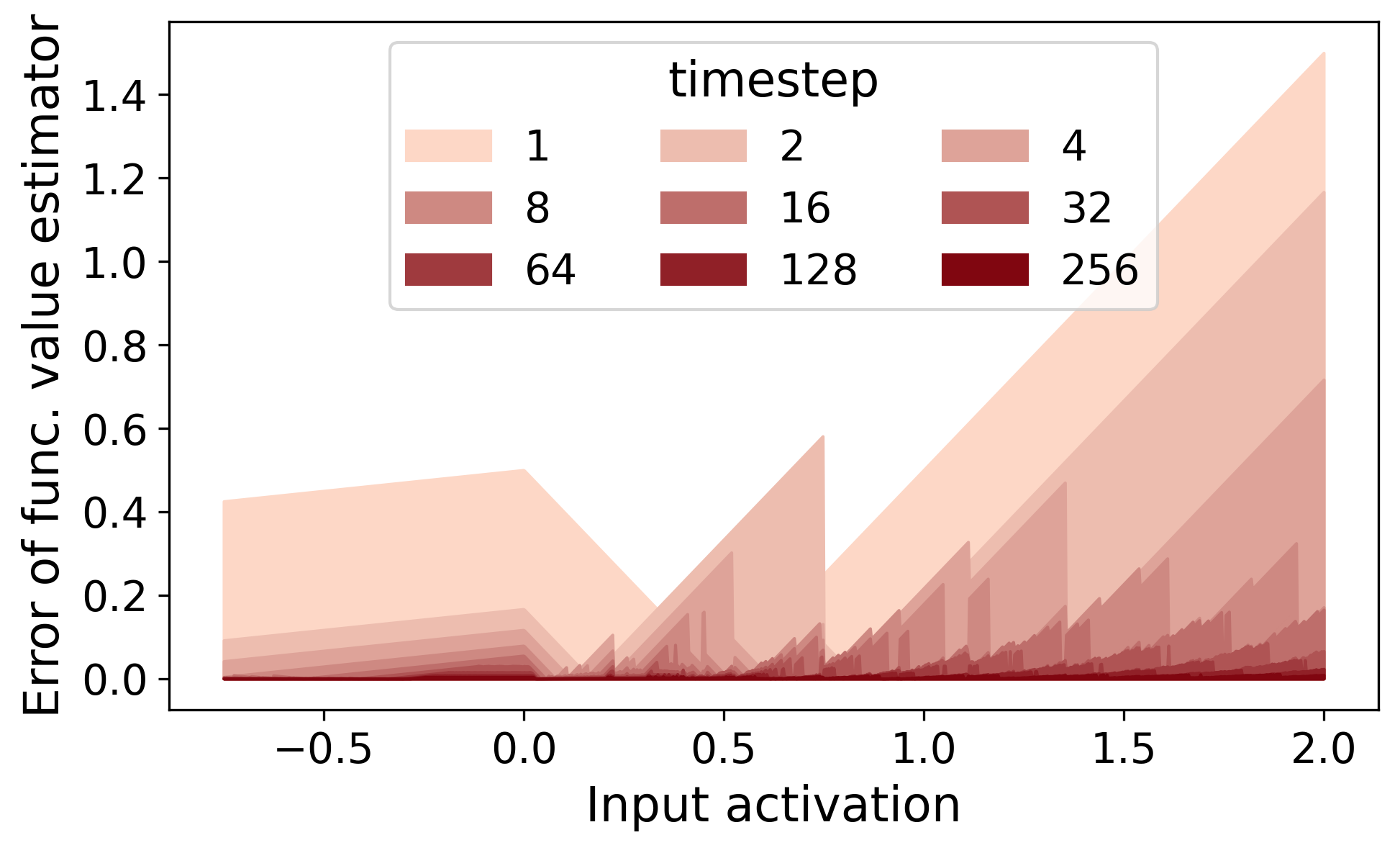}
    }
    \hfill
    \subfigure[GELU approximation.  $\eta(t) = \frac{1}{t+1}$]{
        \includegraphics[width=0.23\linewidth]{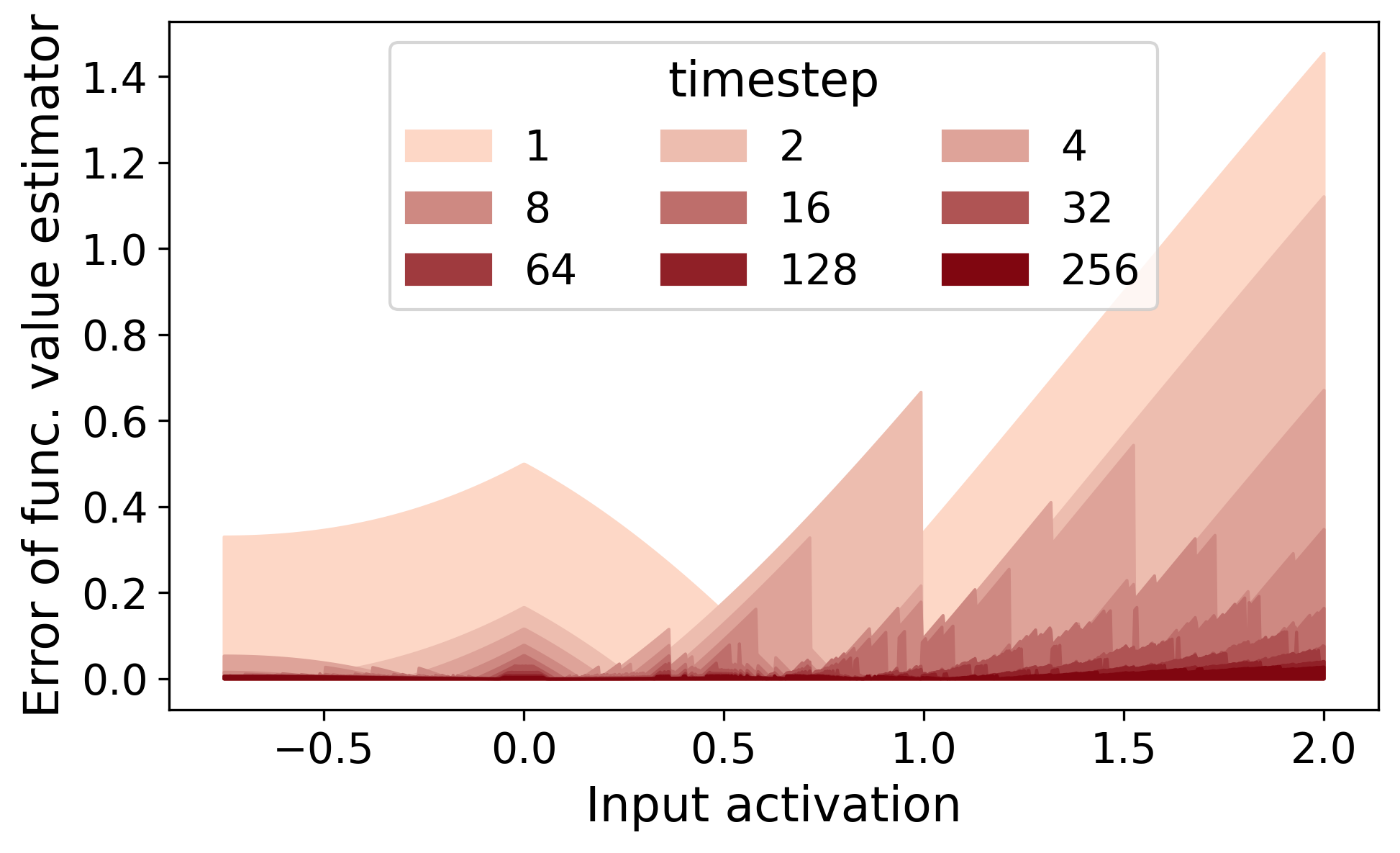}
    }
\vskip -0.1in
\caption{SignGD-based neurons approximating unary nonlinear functions, based on Corollary~\ref{cor:unary_signsgd}. We measure time-evolution of error between the spike neuron output and the nonlinear function output. Input value is deterministically spike-encoded with the Algorithm~\ref{alg:det_spike_encoding_signGD}.}
\label{fig:unary_nonlinearity_signgd}
\end{center}
\vskip -0.2in
\end{figure} 
\begin{figure}[!ht]
\vskip 0.2in
\begin{center}
    \subfigure[Max Pooling approximation with Corollary~\ref{cor:maxpool_signsgd}.  $\eta(t) = \frac{1}{t+1}$.]{
        \includegraphics[width=0.3\linewidth]{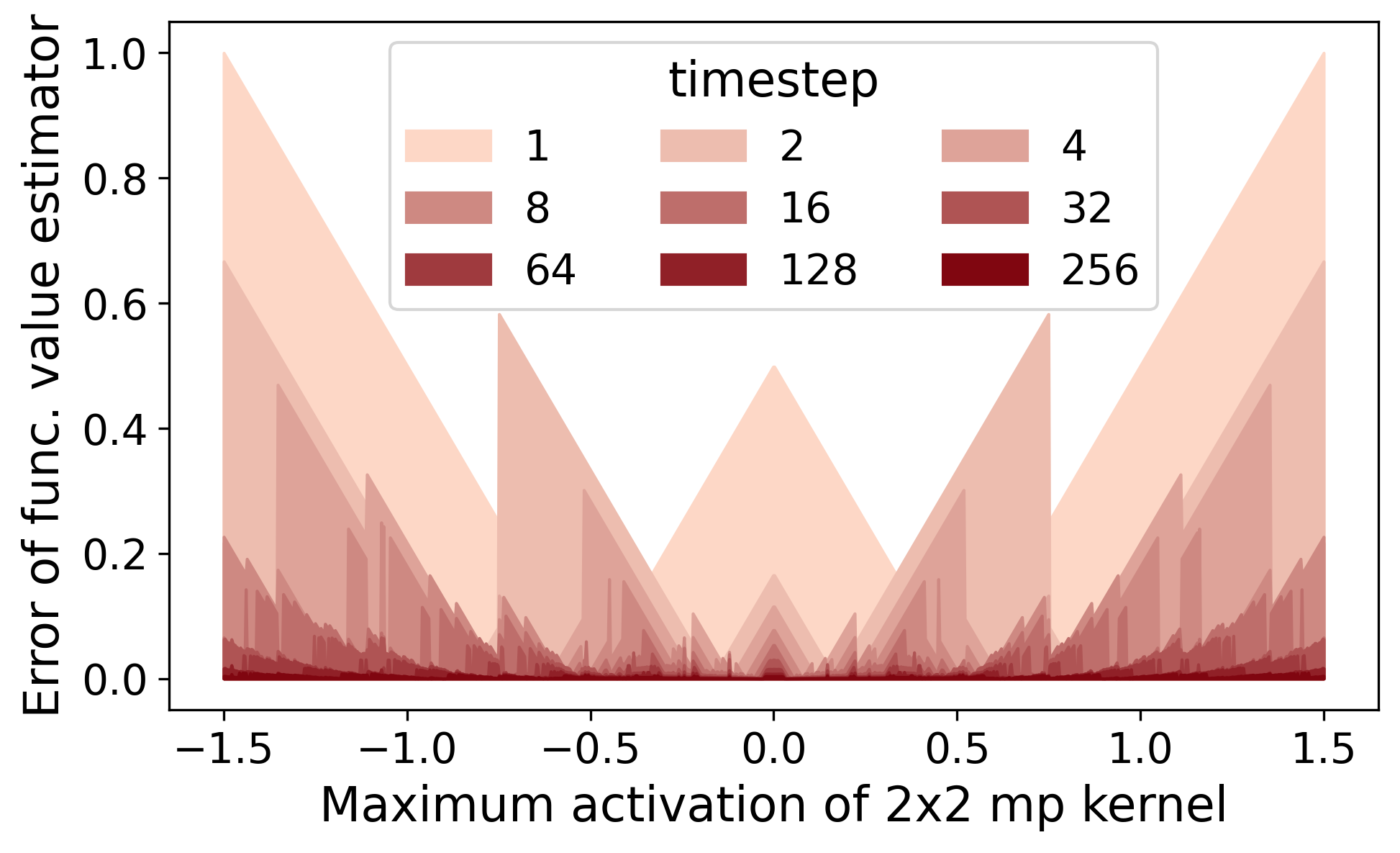}
        \label{fig:signsgd_maxpool}
    }\hfill
    \subfigure[LayerNorm approximation with Corollary~\ref{cor:square_signsgd},~\ref{cor:sqrt_inverse_mul_signsgd}.  $\eta(t) = \frac{1}{t+1}$]{
        \includegraphics[width=0.3\linewidth]{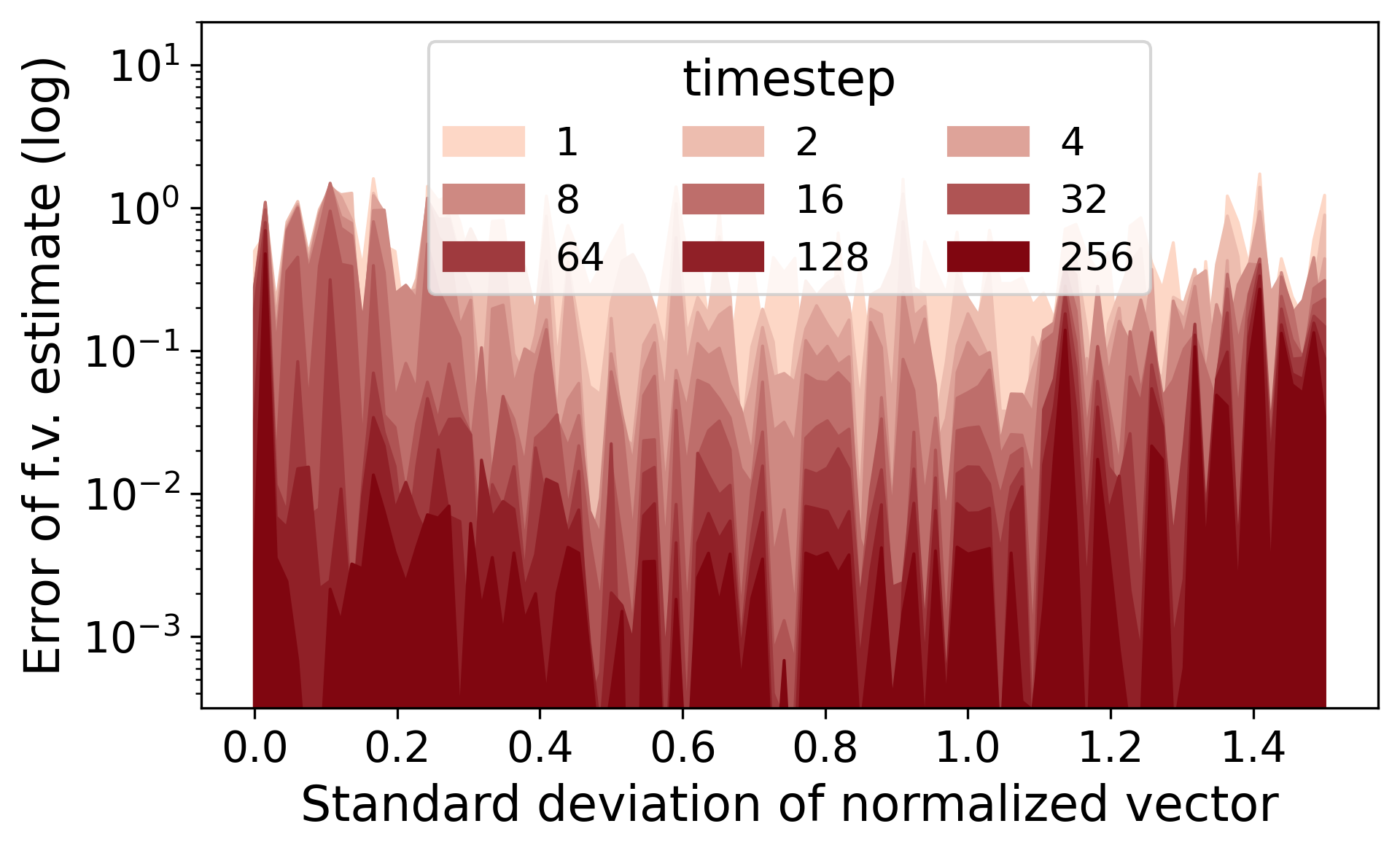}
        \label{fig:signsgd_layernorm}
    }\hfill
    \subfigure[LayerNorm approximation plot as in Fig.~\ref{fig:signsgd_layernorm}, with log-scaled X-axis $(10^{-7}, 200)$ ]{
        \includegraphics[width=0.3\linewidth]{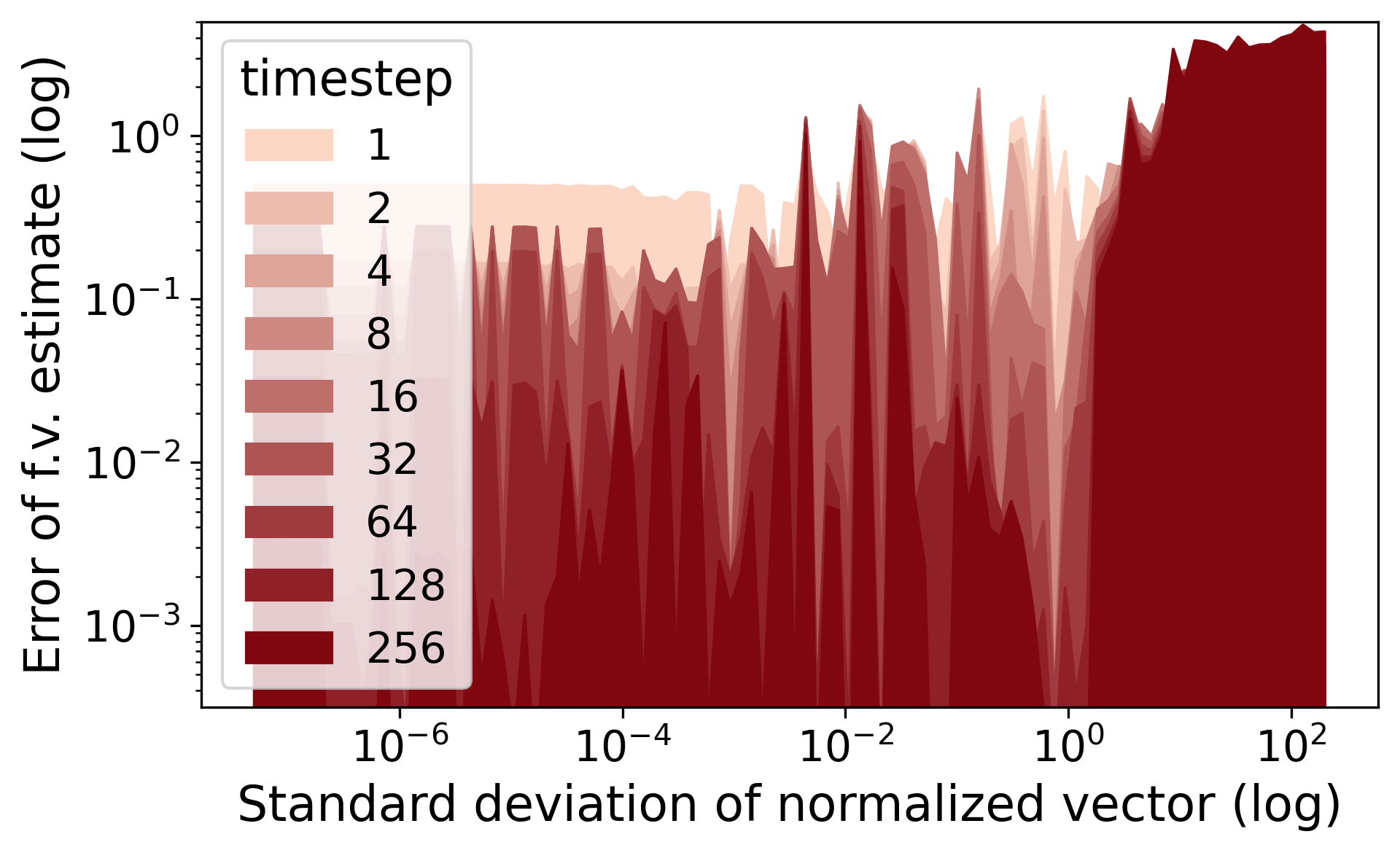}
        \label{fig:signsgd_layernorm_logspace}
    }
    \vfill
    \subfigure[Time-evolution of approximation error for distinct nonlinearities, inverse schedule.]{
        \includegraphics[width=0.3\linewidth]{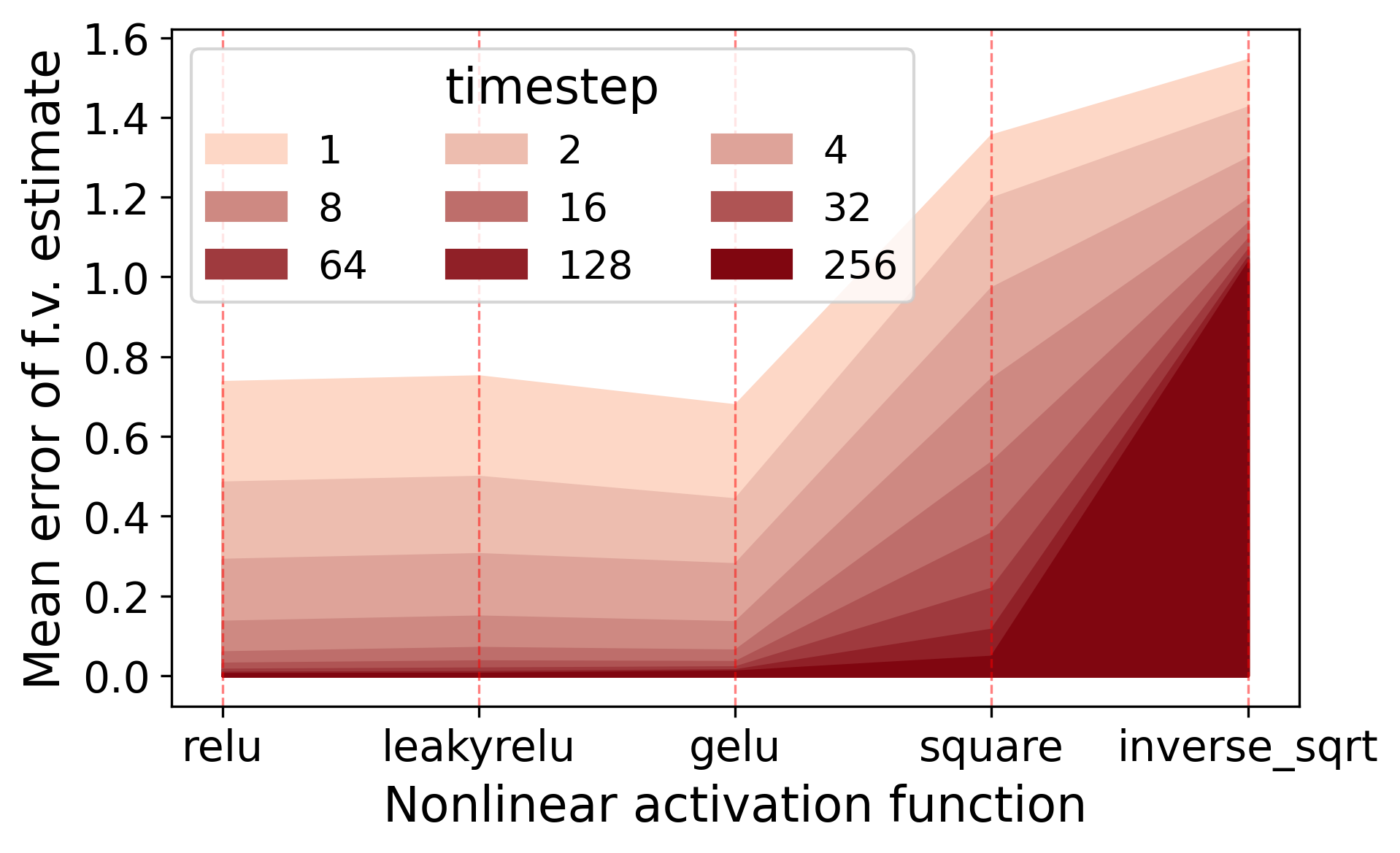}
        \label{fig:signsgd_nonlinearities}
    }\hfill
    \subfigure[Square approximation with Corollary~\ref{cor:square_signsgd}.  $\eta(t) = \frac{1}{t+1}$]{\includegraphics[width=0.3\linewidth]{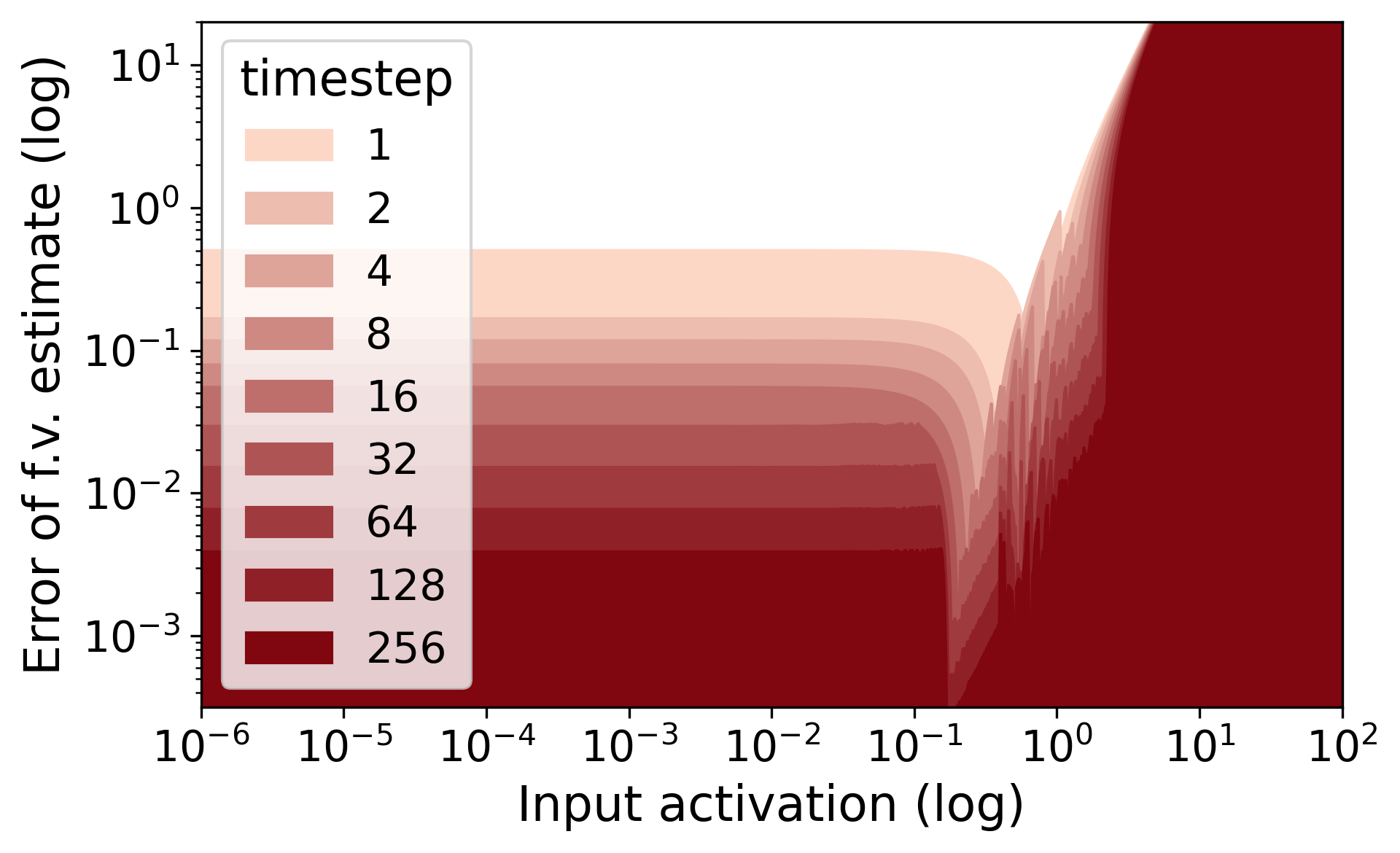}
    \label{fig:signsgd_square}
    }\hfill
    \subfigure[Inverse sqrt approximation with Corollary~\ref{cor:sqrt_inverse_mul_signsgd}.  $\eta(t) = \frac{1}{t+1}$]{
        \includegraphics[width=0.3\linewidth]{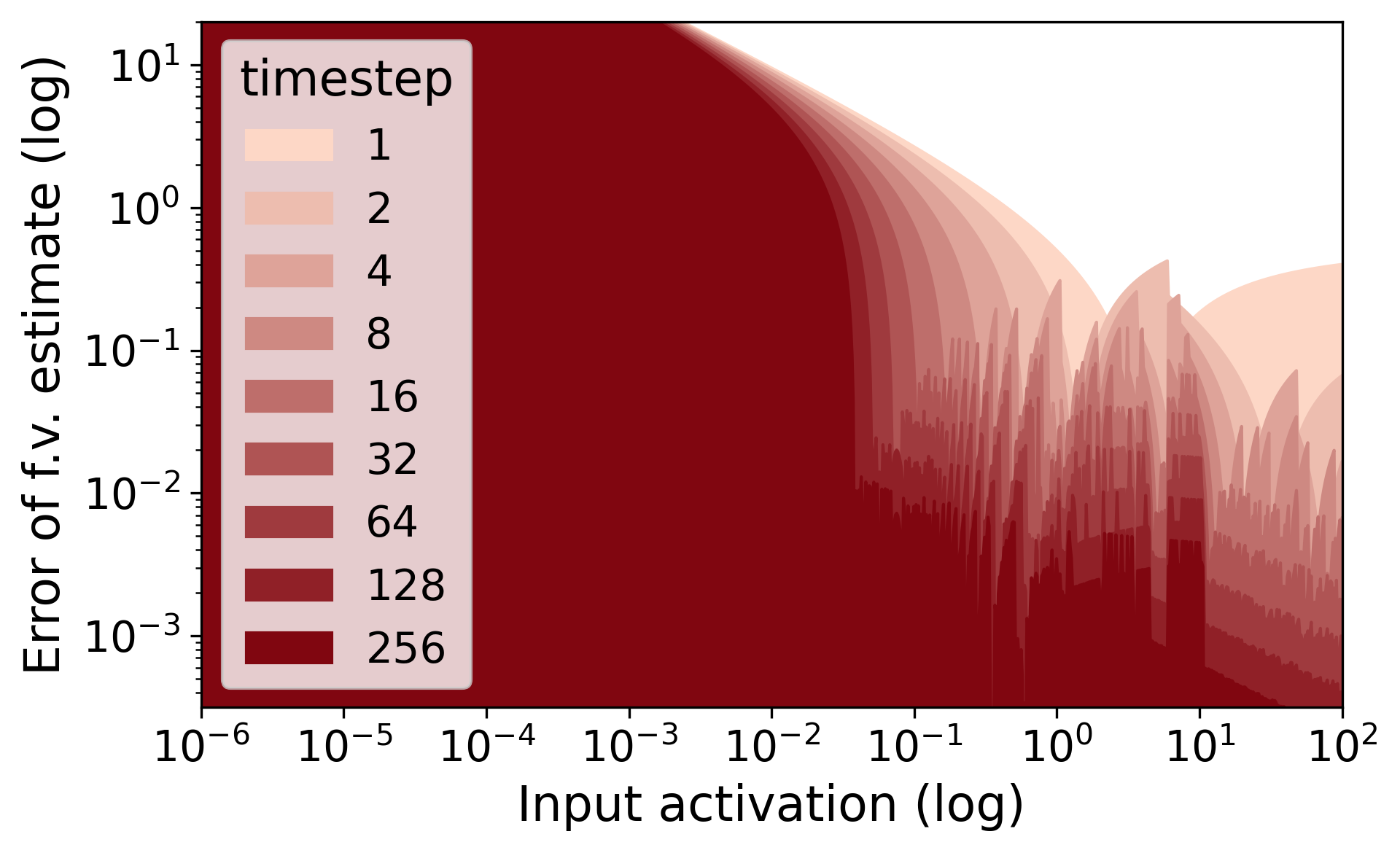}
        \label{fig:signsgd_inversesqrt}
    }
\caption{SignGD-based neurons approximating n-ary nonlinear operators. Input value is float-encoded with the Algorithm~\ref{alg:float_encoding_signGD}.}
\label{fig:binary_nonlinearity_signgd}
\end{center}
\vskip -0.2in
\end{figure}
\section{Empirical validation of multi-operand nonlinearity approximation with our signGD-based neuronal dynamics}
\label{sec:layernorm_speed}
We conduct toy experiments to validate the approximation capability of our signGD-based neuron on n-ary nonlinearities. 
Figure~\ref{fig:signsgd_maxpool} is the time-evolution of error from signGD-based neurons on the max pooling operator. X axis is the maximum $x_{max}$ of 2x2 patch. Rest of the patch elements $x$ are sampled by $x = \min(x_{max}, \epsilon-1)$ where $\epsilon \sim \mathcal{N}(0,1)$. The figure shows that our neuron approximates the max pooling operation through time-steps. Figure~\ref{fig:signsgd_layernorm} shows the time-evolution of error of our spike-based approximation of the layer normalization operator. For the experiment, we sample 10-dim vector from $\mathcal{N}(0,1)$, normalize and scale it to make its standard deviation $\sigma$. We plot the error of the first vector element. The figure shows that the approximation is noticeably slow as $\sigma \to 0$ or $\sigma$ gets larger. Figure~\ref{fig:signsgd_layernorm_logspace} further illustrates this trend with logarithmic $\sigma$ values. We analyze in-depth through Figure~\ref{fig:signsgd_nonlinearities}-~\ref{fig:signsgd_inversesqrt} and find that the slow convergence is because the step size is invariant with the input activation. The step size of signGD solely depends on the learning rate schedule. 
In Figure~\ref{fig:signsgd_square}, precise square approximation for a small input value, e.g., $0.1$, is incompatible with a square approximation of large value like $5,10$. For example, if we decrease the initial learning rate to swiftly approximate $0.01 = (0.1)^2$, then it needs more time-steps to approximate $25 = 5^2$. Also, Figure~\ref{fig:signsgd_inversesqrt} shows that inverse sqrt approximation gets slower or even fails to converge as the denominator, the input value, tends to $0$. It is difficult to design a learning rate schedule that mitigates these approximation issues. In Figure~\ref{fig:signsgd_nonlinearities}, we compare the time-evolution of mean approximation error over five unary nonlinearities, given a Gaussian-sampled input vector of length $1000$. Errors of square and inverse sqrt functions converge slower than others, explaining why it is difficult to approximate the layer normalization. Faster approximation of layer normalization is crucial to accelerate the inference of SNN-converted  MLP-Mixer, ConvNext, and Transformer, but we leave it as a future work.


\begin{algorithm}[tb]
   \caption{Float encoding for signed schedule coding (Definition~\ref{def:signed_schedule_coding})}
   \label{alg:float_encoding_signGD}
\begin{algorithmic}[1]
   \STATE {\bfseries Input:} A real-valued input activation $x$, learning rate schedule $\eta: \mathbb{N} \to \mathbb{R}$, total time-steps $T$
   \STATE {\bfseries Output:} A float train, i.e., time-series of floats, S = $[s(1),s(2), \cdots, s(T)] \in \mathbb{R}^T$
   \STATE Initialize $f \leftarrow 0$, $S \leftarrow [~]$.
   \FOR{$t=1$ {\bfseries to} $T$}
   \STATE$\text{grad} \leftarrow f - x$.
   \STATE$\text{current} \leftarrow 0.5 \cdot (1 + \text{grad})$.
   \STATE$f \leftarrow f - \eta(t)\cdot \text{grad}$.
   \STATE $S$.append(~\text{current}~)  
   \ENDFOR
\end{algorithmic}
\end{algorithm}
\begin{algorithm}[tb]
   \caption{Deterministic spike-based encoding for  signed schedule coding (Definition~\ref{def:signed_schedule_coding})}
   \label{alg:det_spike_encoding_signGD}
\begin{algorithmic}[1]
   \STATE {\bfseries Input:} A real-valued input activation $x$, learning rate schedule $\eta: \mathbb{N} \to \mathbb{R}$, total time-steps $T$
   \STATE {\bfseries Output:} A spike train, i.e., time-series of binary spikes, S = $[s(1),s(2), \cdots, s(T)] \in \{0,1\}^T$
   \STATE Initialize $f \leftarrow 0$, $S \leftarrow [~]$.
   \FOR{$t=1$ {\bfseries to} $T$}
   \STATE$\text{current} \leftarrow \mathbb{H}(f - x)$.
   \STATE$\text{grad} \leftarrow 2 \cdot \text{current} - 1$.
   \STATE$f \leftarrow f - \eta(t)\cdot \text{grad}$.
   \STATE $S$.append(~\text{current}~) 
   \ENDFOR
\end{algorithmic}
\end{algorithm}
\begin{algorithm}[tb]
   \caption{Stochastic spike-based sigmoidal encoding for signed schedule coding (Definition~\ref{def:signed_schedule_coding})}
   \label{alg:stochastic_spike_encoding_signGD}
\begin{algorithmic}[1]
   \STATE {\bfseries Input:} A real-valued input activation $x$, learning rate schedule $\eta: \mathbb{N} \to \mathbb{R}$, total time-steps $T$,  stochasticity $c \in [0,1]$
   \STATE {\bfseries Output:} A spike train, i.e., time-series of binary spikes, S = $[s(1),s(2), \cdots, s(T)] \in \{0,1\}^T$
   \STATE Initialize $f \leftarrow 0$, $S \leftarrow [~]$.
   \FOR{$t=1$ {\bfseries to} $T$}
   \STATE$p \leftarrow \text{Sigmoid}(c \cdot (f - x))$.
   \STATE$\text{current} \sim \text{Bernoulli}(p)$.
   \STATE$\text{grad} \leftarrow 2 \cdot \text{current} - 1$.
   \STATE$f \leftarrow f - \eta(t)\cdot \text{grad}$.
   \STATE $S$.append(~\text{current}~) 
   \ENDFOR
\end{algorithmic}
\end{algorithm}
\section{Neuronal codec for signGD-based neuronal dynamics}
\label{sec:signsgd_spike_encoding}
Neurons in SNN communicate with spike train. spike trains share the same neural coding scheme different from the float(or quantized integer)-based activations used in DNN. We can thus think of a \emph{neuronal codec} which encodes a real number into a spike train, and also decodes a spike train into a a real number. For example, rate coding has multiple encoding algorithms, e.g., float encoding~\cite{li2021freelunch} or poisson encoding~\cite{Sengupta2018SNNVGGResNet}. A pair of encoding and decoding scheme consists a neuronal codec. 

To leverage signGD-based neurons for practical SNN inference, we develop neuronal codecs for the signed schedule coding (Definition~\ref{def:signed_schedule_coding}). 
We develop three spike encoding algorithms for signed schedule coding: float-based (Algorithm~\ref{alg:float_encoding_signGD}), deterministic spike-based (Algorithm~\ref{alg:det_spike_encoding_signGD}), and stochastic spike-based encoding (Algorithm~\ref{alg:stochastic_spike_encoding_signGD}). In practice, We can eliminate total number of time-steps $T$ from the input arguments of algorithms; The encoding process can be lazy-evaluated to generate an infinite-length spike (or float) train.

\section {Continuous dynamics of one-dimensional integrate-and-fire models} 
\label{sec:continuous_lif_dynamics}
The continuous neuronal dynamics of general one-dimensional integrate-and-fire neuron is a linear differential equation with thresholding criterion~\cite{gerstner2014neuronaldynamcis}.
\begin{align*}
\tau_m \frac{du}{dt} &= -f(u(t)) + R I(t) & \text{(Integration)}\\
\lim_{\delta \to 0; \delta > 0} &u(t + \delta) = u_{rest} \text{\qquad if } u(t) \ge \theta_{th} & \text{(Thresholding)}
\end{align*}
with time $t \in \mathbb{R}^+$, dynamics function $f(u): \mathbb{R} \to \mathbb{R}$,  
input current $I(t)$ at time $t$, membrane potential $u$, resting potential $u_{rest}$, threshold $\theta_{th}$, membrane resistance $R$, membrane capacitance $C$ and membrane constant $\tau_m = RC$.

\section{Experimental setup \& Implementation details}
\label{app:implementation_details}
We use Top-1 classification accuracy as a measure of SNN inference performance. 
Following a standard normalization practice in conversion literatures~\cite{diehl2015thresholdbalancing,Rueckauer2017ConversionOC,Wang2022SignedNW}, we scale inputs and outputs of each ReLU layer with layer-wise maximum activation value over random $100$ batches before the conversion. Algorithm~\ref{alg:ann_to_snn_conversion} in the appendix describes the end-to-end ANN-to-SNN conversion with our neuron.
To train DNN models for CIFAR datasets, we use SGD with learning rate $0.1$, momentum $0.9$, weight decay $5\times 10^{-4}$, and cosine annealing  schedule~\cite{loshchilov2016sgdr} of $T_{max} = 300$. 
For ImageNet classification models, we use pretrained DNN weights of VGG, ResNet, ConvNext, and RegNetX from torchvision~\cite{torchvision2016},VGG and ResNet without max pooling from \cite{li2021freelunch}, and MLP-Mixer and ResMLP from timm~\cite{wightman2019timm}. We run our experiments on a machine with AMD EPYC 7313 CPU, 512GB RAM, and NVIDIA RTX A6000.

\section{Theoretical energy consumption analysis}
\label{sec:energy_consumption}

To analyze the energy consumption of our proposed neuron model, we use a theoretical energy cost estimation framework that is widely recognized in prior research~\cite{zhou2022spikformer,kundu2021hiresnn,yin2021accuratesrnn, Yao2022AttentionSNN}. We choose this approach because a direct hardware implementation and empirical evaluation of our neuron are beyond the current scope. However, we plan to explore this direction in future work. We compare the energy costs of LIF/IF neurons and our neuron as follows.

\subsection{Single synaptic operation}
We first compare the energy consumption of single spike-based accumulation (AC) operation (synaptic operation, SOP) on 45nm CMOS processors~\cite{horowitzcmos}.
SOP energy consumption of our neuron varies with the selection of dynamics coefficients. If we choose $\eta(t) = \eta(0) \gamma^t$ following the evaluation setup of Table~\ref{tab:imagenet_comparison},~\ref{tab:cifar10} and \ref{tab:cifar100}, the coefficients satisfying Theorem~\ref{thm:signsgd} and Corollary~\ref{cor:unary_signsgd} can be chosen as $\alpha_1(t) = 1/\gamma$, $\beta_1(t) = \gamma$, $\alpha_2(t) = \beta_2(t) = 1$. The difference of our neuron with LIF neuron is a new internal variable $u(t)$ and a weight term $W$ at the dynamics of $v(t)$. Since the $W$ term addition is absorbed into an existing operation, This results in an energy overhead $0.9$pJ added to the $0.9$pJ consumption of LIF neuron~\cite{zhou2022spikformer} to become $1.8$ pJ, which is twice over the LIF neuron. Table~\ref{tab:sop_energy_consumption} compares the per-operation energy consumption, which we denote as $E_{\text{SOP}}$.

\begin{table}[ht]
    \centering
    \begin{tabular}{c|c|c|c}
        \toprule
         &IF / LIF neuron (AC)&signGD-based neuron (ours)&FLOPs (MAC)\\
         \midrule
         $E_{\text{SOP}}$&0.9pJ&1.8pJ	&4.6pJ\\
         \bottomrule
    \end{tabular}
    \caption{Theoretically derived energy consumption of a single operation on 45nm CMOS processors~\cite{horowitzcmos}.}
    \label{tab:sop_energy_consumption}
    \vskip -0.1in
\end{table}

\subsection{End-to-End inference of converted DNN models}

We now compare the energy consumption of converted DNN models.
In terms of the end-to-end SNN model, our neuron differs with prior spiking neurons only in the computation of neuronal dynamics, and everything else remains the same. Table~\ref{tab:e2e_energy_consumption} lists firing rates $fr = \frac{\text{Number of spikes}}{\text{Timesteps * Neurons} }$, number of SOPs $N_{\text{SOP}} = \text{(Number of spikes)}$, and energy consumption estimated by $E = N_{\text{SOP}} \cdot E_{SOP}$.

\begin{table}
    \centering
    \begin{tabular}{c|cccc}
        \toprule
         CIFAR10, Timesteps T = 64	&Neurons&Firing Rates(\%)&	$N_{\text{SOP}}$ (M)&	Energy (uJ)\\
         \midrule
         \multirow{ 3}{*}{ResNet18~\cite{He2015ResNet}}&IF Neuron&20.97&22.431&20.188\\
         &LIF Neuron&21.17&22.651&20.386\\
         &signGD-based Neuron&16.78&17.943&32.298 ($1.59\times$)\\
         \midrule
         \multirow{ 3}{*}{VGG16~\cite{Simonyan2014VGG}}&IF Neuron&20.30&10.778&9.700 \\
         &LIF Neuron&20.71&10.993&9.893\\
         &signGD-based Neuron&16.75&8.892&16.006 ($1.65\times$)\\
         \bottomrule
    \end{tabular}
    \caption{Inference energy consumption of ReLU Networks converted with different spiking neuron models. }
    \label{tab:e2e_energy_consumption}
    \vskip -0.1in
\end{table}

Our theoretical analysis outlines that more complex internal dynamics of the signGD-based neuron lead to higher energy consumption in same time-steps compared to IF and LIF neurons. The relative increase in energy usage is manageable with an approximate factor of $1.6\times$. This increase is offset by advantages our approach provides, including broader support for nonlinear operator approximations and the reduced time steps in SNN inference. For example, in Table~\ref{tab:cifar10}, ResNet18 converted with subgradient-based neurons achieve 94\% accuracy in 64 steps, while the model converted with our neuron model achieves the same accuracy in only 22 steps. 

Besides, the energy efficiency of SNN comes not only from computation but also from memory storage. For example, while IF neurons theoretically offer a $5.11\times$ energy efficiency in SOP compared to traditional FLOPS, the VGG model implemented on a neuromorphic SNN chip achieves even higher efficiency, ranging from $35\times$ to $560\times$~\cite{bu2021optimalannsnn}, thanks to in-memory computing. To accurately assess this, a real hardware implementation is crucial, and we plan to explore this in our future work.

\section{Mathematical Proofs.}
\label{app:proofs}
\begin{theorem}
\label{thm:if_dynamics}
Dynamical system of IF neuron~(Eq. \ref{lif:2}-\ref{if:1}) with rate-coded input $\tilde{x}(t) = \frac{1}{t}\sum_{i=1}^{t} I(i)$ and output $y(t)= \frac{1}{t}\sum_{i=1}^{t} s(i)$ is equivalent to a subgradient method  over an optimization problem $\min_{y \in \mathbb{R}} \mathcal{L}(y;x)$, approximated with $x \leftarrow \tilde{x}(t+1)$ as,

\begin{align}
&\tilde{f}(t) = \tilde{f}(t-1) - \frac{1}{t+1} \cdot \tilde{g}\big(\tilde{f}(t-1); \tilde{x}(t)\big) \label{eq:if_subgradient}\\
&\mathcal{L}(y; x) = h\big( \frac{R}{\theta_{th}} x - y\big) + \frac{1}{2} y^2 \quad\text{(objective fn.)}\label{eq:if_subgradient_objfn}\\
&\tilde{f}(t) = \frac{t}{t+1} y(t) - \frac{u(0) - \theta_{th}}{\theta_{th}(t+1)}~ \text{($t$-th approx.)}\label{eq:if_rate_output_estimator}
\end{align}
where $\tilde{g}(y;x)$ is a subgradient of $\mathcal{L}(y;x)$, $h(x) = \text{ReLU}(x)$ and $u(0)$ an initial membrane potential. Solution of the problem is $\text{ReLU1}(\frac{R}{\theta_{th}} x)$.
\end{theorem}
\begin{proof}
\label{pf:if_rate}
Applying equation~\ref{lif:3} to equation~\ref{lif:1}, we rearrange the membrane equation for $u_{pre}$ as
\begin{equation}
\label{eq:u_pre}
u_{pre}(t+1) = (u_{pre}(t) - \theta_{th} s(t)) + RI(t+1)
\end{equation}
Spike $s(t)$ links the recurrence relation~\eqref{eq:u_pre} of the membrane potential $u_{pre}$ and coding scheme~\eqref{eq:rate} of output activation $y$.
\begin{equation}
\label{eq:if_recurrence_mempot_output}
    s(t) = \frac{-1}{\theta_{th}} (u_{pre}(t+1) - u_{pre}(t) - RI(t+1)) 
    =  ty(t) - (t-1)y(t-1)
\end{equation}
Solving the recurrence relation~\eqref{eq:if_recurrence_mempot_output} through time $t$ yields a linear relation of membrane potential $u_{pre}$, input $\tilde{x}$ and output $y$. Note that $u_{pre}(1) = u(0) + RI(1)$.
\begin{align*}
 ty(t) - (t-1)y(t-1)
 &= \frac{-1}{\theta_{th}} \big(u_{pre}(t+1) - u_{pre}(t) - RI(t+1)\big) \\
  (t-1)y(t-1) - (t-2)y(t-2)
 &= \frac{-1}{\theta_{th}} \big(u_{pre}(t) - u_{pre}(t-1) - RI(t)\big) \\
 &\vdots\\
  1y(1) - 0y(0)
 &= \frac{-1}{\theta_{th}} \big(u_{pre}(2) - u_{pre}(1) - RI(2)\big) \\
 +~\xdash[10em]&\xdash[18em]\\
 ty(t) &= \frac{-1}{\theta_{th}} \big(u_{pre}(t+1) - u_{pre}(1) - R \sum_{i = 2}^{t+1} I(i)\big)\\
 &= \frac{-1}{\theta_{th}} \big(u_{pre}(t+1) - u(0) - RI(1) - R \sum_{i = 2}^{t+1} I(i)\big)\\
 &= \frac{-1}{\theta_{th}} \big(u_{pre}(t+1) - u(0) - R \sum_{i = 1}^{t+1} I(i)\big)\\
 &= \frac{-1}{\theta_{th}} \big(u_{pre}(t+1) - u(0) - R\cdot(t+1)\cdot \tilde{x}(t+1)\big)
\end{align*}
Rearranging the above equation  for $u_{pre}$,
\begin{equation}
\therefore u_{pre}(t+1) = -\theta_{th} \cdot t \cdot y(t) + u(0) + R \cdot(t+1)\cdot\tilde{x}(t+1)\label{eq:if_membrane_output_linear}
\end{equation}
With equation~\eqref{eq:if_membrane_output_linear}, we remove the $u_{pre}$ term from rate coding equation~\eqref{eq:rate} to obtain a discrete dynamical system of the input $\tilde{x}$ and output $y$.
\begin{align}
 y(t) &= \frac{t-1}{t} y(t-1) + \frac{1}{t} \mathbb{H} (-\theta_{th}\cdot  (t-1) \cdot y(t-1) + u(0) + R \cdot t \cdot \tilde{x}(t) - \theta_{th}) \label{pf:if_rate_dds}
 \end{align}
By definition of $\tilde{f}(t)$ in~\eqref{eq:if_rate_output_estimator}, we substitute $y$ with $\tilde{f}$ from the dynamical system~\eqref{pf:if_rate_dds}.
 \begin{align*}
 \theta_{th}\cdot(t+1)\cdot\tilde{f}(t)  &= \theta_{th}\cdot t \cdot  y(t-1) - (u(0) - \theta_{th})   \\
 \theta_{th}\cdot t \cdot \tilde{f}(t-1)  &= \theta_{th}\cdot (t-1) \cdot y(t-1) - (u(0) - \theta_{th}) \\
 - \theta_{th} \cdot t \cdot \tilde{f}(t-1) &= -\theta_{th}\cdot (t-1) \cdot y(t-1) +  u(0) - \theta_{th} \\
  \frac{t-1}{t}y(t-1) &= \tilde{f}(t-1) + \frac{u(0) - \theta_{th}}{\theta_{th} \cdot t}  \\
  y(t) &= \frac{t+1}{t} \tilde{f}(t) + \frac{u(0) - \theta_{th}}{\theta_{th} \cdot t} &\\
  \therefore \frac{t+1}{t} \tilde{f}(t) + \frac{u(0) - \theta_{th}}{\theta_{th} \cdot t} &= \tilde{f}(t-1) + \frac{u(0) - \theta_{th}}{\theta_{th} \cdot t} 
  + \frac{1}{t}\mathbb{H}\big(-\theta_{th} \cdot t \cdot \tilde{f}(t-1) + R \cdot t \cdot \tilde{x}(t)\big)  
\end{align*}
By the scale-invariance of step function $\mathbb{H}$, we obtain
\begin{align*}
    \tilde{f}(t) &= \frac{t}{t+1} \tilde{f}(t-1) + \frac{1}{t+1}\mathbb{H}(R\cdot t \cdot \tilde{x}(t) - \theta_{th}\cdot t \cdot \tilde{f}(t-1))\\
  &=  \frac{t}{t+1} \tilde{f}(t-1) + \frac{1}{t+1}\mathbb{H}(\frac{R}{\theta_{th}}\tilde{x}(t) - \tilde{f}(t-1))\\
  &= \tilde{f}(t-1) - \frac{1}{t+1} (\tilde{f}(t-1) - \mathbb{H}(\frac{R}{\theta_{th}}\tilde{x}(t) - \tilde{f}(t-1)))\\
  &= \tilde{f}(t-1) - \frac{1}{t+1} \cdot \tilde{g}(\tilde{f}(t-1); \tilde{x}(t))
\end{align*}
where $\mathcal{L}(y; x) = f\big( \frac{R}{\theta_{th}} x - y\big) + \frac{1}{2} y^2$. It corresponds to the subgradient method over an optimization problem $\min_{y \in \mathbb{R}} \mathcal{L}(y;x)$ with a diminishing step size $\frac{1}{t+1}$, approximated with $x \leftarrow \tilde{x}(t+1)$. The objective function $\mathcal{L}(y;x)$ is a convex $l^2$-regularized negative ReLU function. 
We now show that the  $\argmin_{y} \mathcal{L}(y;x) = \text{ReLU1}(\frac{R}{\theta_{th}}x)$ through 3 different cases. Note that the sub-gradient of $\mathcal{L}(y;x)$ is
\begin{align*}
&\tilde{g}(y;x) = y - \mathbb{H}(\frac{R}{\theta_{th}}x - y)
\end{align*}

Case 1, $x < 0$. If $y \le \frac{R}{\theta_{th}}x < 0$, then $\tilde{g}(y;x) = y - 1 < 0$. If $\frac{R}{\theta_{th}}x < y < 0 $, then $\tilde{g}(y;x) = y - 0 < 0$. Finally, $y \ge 0$ then $\frac{R}{\theta_{th}}x - y < 0$ so $\tilde{g}(y;x) = y > 0$. Therefore, $\argmin_y \mathcal{L}(y;x) = 0$

Case 2, $x \ge \frac{\theta_{th}}{R}$. If $y \le 1$ then $\tilde{g}(y;x)(y;x) = y - 1 \le 0$. If $y \ge 1$ then $\tilde{g}(y;x){L}(y;x) \ge y - 1 \ge 0$. Therefore, $\argmin_y \mathcal{L}(y;x) = 1$.

Case 3, $0 < x < \frac{\theta_{th}}{R}$. If $0 < \frac{R}{\theta_{th}}x \le y$ then $\tilde{g}(y;x)(y;x) = y - 0 \ge 0$. If $y \le \frac{R}{\theta_{th}}x \le 1$ then $\tilde{g}(y;x)(y;x) \ge y - 1 \le 0$. Therefore, $\argmin_y \mathcal{L}(y;x) = \frac{R}{\theta_{th}}x$.
\end{proof}

\begin{theorem}
\label{thm:lif_dynamics}
Dynamical system of LIF neuron~(Eq. \ref{lif:2},\ref{lif:3}, \ref{lif:4}) with EMA-coded input $\tilde{x}(t)$ and output $y(t)$~(Eq. \ref{def:ema_coding}), if $\tau_m = \tau$, is equivalent to a   subgradient method over an optimization problem $\min_{y\in \mathbb{R}} \mathcal{L}(y;x)$, approximated with $x  \leftarrow \tilde{x}(t + 1)$ as,
\begin{align*}
&\tilde{f}(t+1) = \tilde{f}(t) - \frac{1}{\tau} \cdot \tilde{g}\big(\tilde{f}(t); \tilde{x}(t+1)\big) \\
&  \mathcal{L}(y; x) = \frac{1}{2} y^2 +\frac{u_{rest}}{\theta_{th}(\tau-1)} y +h(\frac{R - \theta_{th}}{\theta_{th}(\tau - 1)}x -  y)\\
& \tilde{f}(t) = y(t)  - \frac{1}{\tau \theta_{th}}(\frac{\tau-1}{\tau})^{t}u(0) - \frac{u_{rest}\sum_{i = 0}^{t}(\frac{\tau - 1}{\tau})^{t - i}}{\theta_{th} \tau(\tau-1)} 
\end{align*}
where $\tilde{g}(y;x)$ is a subgradient of $\mathcal{L}(y;x)$, $h(x) = \text{ReLU}(x)$ and $u(0)$ an initial membrane potential. If $u_{rest} = 0$, the solution to the problem is $\text{ReLU1}(\frac{R}{\theta_{th}(\tau - 1)}x - \frac{1}{\tau-1})$.
\end{theorem}
\begin{proof}
\label{pf:lif_rate}
$u_{pre}(t) = u(t-1) -\frac{u(t-1) - u_{rest}}{\tau_{m}}
    + \frac{R}{\tau_{m}} I(t)$ $s(t) = \mathbb{H}(u_{pre}(t) - \theta_{th})  u(t) = u_{pre}(t) - \theta_{th}s(t)$ We apply equation~\eqref{lif:3} on equation~\eqref{lif:4} to obtain the membrane equation of $u_{pre}$.
\begin{align}
u_{pre}(t) &= 
(u_{pre}(t-1) - \theta_{th}s(t-1)) 
-\frac{(u_{pre}(t-1) - \theta_{th}s(t-1)) - u_{rest}}{\tau_{m}}
    + \frac{R}{\tau_{m}} I(t)\nonumber\\
    &= 
\frac{\tau_{m} - 1}{\tau_{m}}u_{pre}(t-1) - \theta_{th}\frac{\tau_{m} - 1}{\tau_{m}}s(t-1) + \frac{1}{\tau_{m}}u_{rest}
    + \frac{R}{\tau_{m}} I(t)\label{pf:lif_upre_memberaneeq}
\end{align}
Rearranging the equation~\eqref{pf:lif_upre_memberaneeq} for the spike $s(t)$,
\begin{align*}
    -\frac{\theta_{th}}{\tau_m}s(t-1)&= \frac{1}{\tau_m - 1} u_{pre}(t) - \frac{1}{\tau_m}u_{pre}(t-1)
    - \frac{R}{\tau_m(\tau_m - 1)} I(t)- \frac{1}{\tau_m(\tau_m - 1)} u_{rest}\\
    -\frac{\theta_{th}}{\tau_m}s(t)&= \frac{1}{\tau_m - 1} u_{pre}(t+1) - \frac{1}{\tau_m}u_{pre}(t)
    - \frac{R}{\tau_m(\tau_m - 1)} I(t+1)- \frac{1}{\tau_m(\tau_m - 1)} u_{rest}
\end{align*}
With the definition of $y(t)$ and $\tau = \tau_m$, we yield the recurrence relation of $y$ and $u_{pre}$.
\begin{align*}
    \frac{-\theta_{th}}{\tau}s(t) &=  - \theta_{th}\cdot y(t) +\theta_{th} \cdot \frac{\tau -1}{\tau}y(t-1)\\
    - \theta_{th}\cdot y(t) +\theta_{th} \cdot \frac{\tau -1}{\tau}y(t-1) &= \frac{1}{\tau - 1} u_{pre}(t+1) - \frac{1}{\tau}u_{pre}(t)
    - \frac{R}{\tau(\tau - 1)} I(t+1)- \frac{1}{\tau(\tau - 1)} u_{rest}
\end{align*}
We solve the recurrence relation through time, deriving a linear relation of $u_{pre}$, $I$ and $y$. 
\begin{align*}
    - \theta_{th}\cdot y(t) +\theta_{th} \cdot \frac{\tau -1}{\tau}y(t-1) &= \frac{1}{\tau - 1} u_{pre}(t+1) - \frac{1}{\tau}u_{pre}(t)
    - \frac{R}{\tau(\tau - 1)} I(t+1)- \frac{1}{\tau(\tau - 1)} u_{rest}\\
    \frac{\tau -1}{\tau}\bigg(- \theta_{th}\cdot y(t-1) &+\theta_{th} \cdot \frac{\tau -1}{\tau}y(t-2)\bigg) \\&= \frac{\tau -1}{\tau}\bigg(\frac{1}{\tau - 1} u_{pre}(t) - \frac{1}{\tau}u_{pre}(t-1)
    - \frac{R}{\tau(\tau - 1)} I(t)- \frac{1}{\tau(\tau - 1)} u_{rest}\bigg)\\
 \vdots\\
    (\frac{\tau -1}{\tau})^{t-1}\bigg(- \theta_{th}\cdot y(1) &+\theta_{th} \cdot \frac{\tau -1}{\tau}y(0)\bigg) \\&= (\frac{\tau -1}{\tau})^{t-1}\bigg(\frac{1}{\tau - 1} u_{pre}(2) - \frac{1}{\tau}u_{pre}(1)
    - \frac{R}{\tau(\tau - 1)} I(2)- \frac{1}{\tau(\tau - 1)} u_{rest}\bigg)\\
 +~\xdash[10em]&\xdash[20em]\\
 - \theta_{th}\cdot y(t) + \theta_{th}(\frac{\tau -1}{\tau})^{t}y(0) &= \frac{1}{\tau - 1} u_{pre}(t+1)- \frac{1}{\tau}(\frac{\tau -1}{\tau})^{t-1}u_{pre}(1)\\
 &- \frac{R}{\tau(\tau - 1)} \sum_{i = 2}^{t+1}(\frac{\tau - 1}{\tau})^{t+1 - i} I(i)
  - \frac{u_{rest}}{\tau(\tau-1)} \sum_{i = 1}^{t}(\frac{\tau - 1}{\tau})^{t - i} 
 \end{align*}
Since $u_{pre}(1) = \frac{\tau-1}{\tau}u(0) + \frac{1}{\tau} u_{rest} +  \frac{R}{\tau(\tau - 1)} I(1)$,
\begin{align*}
- \frac{1}{\tau}(\frac{\tau -1}{\tau})^{t-1}u_{pre}(1) &= - \frac{1}{\tau}(\frac{\tau -1}{\tau})^{t-1} \bigg( \frac{\tau-1}{\tau}u(0) + \frac{1}{\tau} u_{rest} +  \frac{R}{\tau} I(1)\bigg) \\
&= - \frac{1}{\tau}(\frac{\tau -1}{\tau})^{t} u(0) - \frac{u_{rest}}{\tau(\tau - 1)}(\frac{\tau -1}{\tau})^{t} - \frac{1}{\tau}(\frac{\tau -1}{\tau})^{t-1}\cdot\frac{R}{\tau} I(1)\\
&= - \frac{1}{\tau}(\frac{\tau -1}{\tau})^{t} u(0) - \frac{u_{rest}}{\tau(\tau - 1)}(\frac{\tau -1}{\tau})^{t} - \frac{R}{\tau(\tau -1)}(\frac{\tau -1}{\tau})^{t} I(1)
\end{align*}
Removing the term $y(0) = 0$ and $u_{pre}(1)$ from  the above linear relation results in the linear relation of $u_{pre},y$ and $\tilde{x}$.
\begin{align*}
 - \theta_{th}\cdot y(t) &= \frac{1}{\tau - 1} u_{pre}(t+1)- \frac{R}{\tau(\tau - 1)} \sum_{i = 2}^{t+1}(\frac{\tau - 1}{\tau})^{t+1 - i} I(i)
  - \frac{u_{rest}}{\tau(\tau-1)} \sum_{i = 1}^{t}(\frac{\tau - 1}{\tau})^{t - i} \\
  &- \frac{1}{\tau}(\frac{\tau -1}{\tau})^{t-1}u_{pre}(1)\\
  &= \frac{1}{\tau - 1} u_{pre}(t+1)- \frac{R}{\tau(\tau - 1)} \sum_{i = 2}^{t+1}(\frac{\tau - 1}{\tau})^{t+1 - i} I(i)
  - \frac{u_{rest}}{\tau(\tau-1)} \sum_{i = 1}^{t}(\frac{\tau - 1}{\tau})^{t - i}\\
  &- \frac{1}{\tau}(\frac{\tau -1}{\tau})^{t} u(0) - \frac{u_{rest}}{\tau(\tau - 1)}(\frac{\tau -1}{\tau})^{t} - \frac{R}{\tau(\tau -1)}(\frac{\tau -1}{\tau})^{t} I(1)\\
  &= \frac{1}{\tau - 1} u_{pre}(t+1)- \frac{1}{\tau}(\frac{\tau -1}{\tau})^{t} u(0) - \frac{R}{\tau(\tau - 1)} \sum_{i = 1}^{t+1}(\frac{\tau - 1}{\tau})^{t+1 - i} I(i)
  - \frac{u_{rest}}{\tau(\tau-1)} \sum_{i = 0}^{t}(\frac{\tau - 1}{\tau})^{t - i} \\
  &= \frac{1}{\tau - 1} u_{pre}(t+1)- \frac{1}{\tau}(\frac{\tau -1}{\tau})^{t} u(0) - \frac{R}{\tau - 1} \tilde{x}(t+1) - \frac{u_{rest}}{\tau(\tau-1)} \sum_{i = 0}^{t}(\frac{\tau - 1}{\tau})^{t - i}
\end{align*}
By definition $\tilde{f}(t) = y(t)  - \frac{1}{\tau \theta_{th}}(\frac{\tau-1}{\tau})^{t}u(0) - \frac{u_{rest}}{\theta_{th} \tau(\tau-1)} \sum_{i = 0}^{t}(\frac{\tau - 1}{\tau})^{t - i}$, thus 
 \begin{align*}
 -\theta_{th} \tilde{f}(t) = \frac{1}{\tau - 1} u_{pre}(t+1) - \frac{R}{\tau - 1} \tilde{x}(t+1)\\
 \therefore u_{pre}(t+1) = R \cdot \tilde{x}(t+1) - \theta_{th}\cdot  (\tau - 1)\cdot \tilde{f}(t)
\end{align*}
We now reformulate the EMA coding equation~\eqref{def:ema_coding} with $\tilde{f}$ and $\tilde{x}$.
\begin{align*}
y(t) &- \frac{\tau-1}{\tau} y(t-1) = \tilde{f}(t) +  \frac{1}{\tau \theta_{th}}(\frac{\tau-1}{\tau})^{t}u(0) +  \frac{u_{rest}}{\theta_{th} \cdot \tau(\tau-1)} \sum_{i = 0}^{t}(\frac{\tau - 1}{\tau})^{t - i} \\
&- \frac{\tau-1}{\tau}\bigg(\tilde{f}(t-1) +  \frac{1}{\tau \theta_{th}}(\frac{\tau-1}{\tau})^{t-1}u(0) + \frac{u_{rest}}{\theta_{th} \cdot \tau(\tau-1)} \sum_{i = 0}^{t-1}(\frac{\tau - 1}{\tau})^{t-1 - i}\bigg)\\
&= \tilde{f}(t) - \frac{\tau-1}{\tau} \tilde{f}(t-1) + \frac{u_{rest}}{\theta_{th} \cdot \tau(\tau-1)} = \frac{1}{\tau}s(t)\\
  \frac{1}{\tau}s(t)=& \tilde{f}(t) - \frac{\tau-1}{\tau} \tilde{f}(t-1) + \frac{u_{rest}}{\theta_{th} \cdot \tau(\tau-1)}\\
 =& \frac{1}{\tau}\mathbb{H}(u_{pre}(t) - \theta_{th}) \\
 =& \frac{1}{\tau}\mathbb{H}(R \cdot \tilde{x}(t) - \theta_{th}\cdot  (\tau - 1)\cdot \tilde{f}(t-1) - \theta_{th})\\
 \tilde{f}(t) &= \frac{\tau-1}{\tau} \tilde{f}(t-1) - \frac{u_{rest}}{\theta_{th} \cdot \tau(\tau-1)} + \frac{1}{\tau}\mathbb{H}(R \cdot \tilde{x}(t) - \theta_{th}\cdot  (\tau - 1)\cdot \tilde{f}(t-1) - \theta_{th}) \\
 \therefore \tilde{f}(t) =& \tilde{f}(t-1) - \frac{1}{\tau} \bigg(\tilde{f}(t-1) +  \frac{u_{rest}}{\theta_{th}(\tau-1)} 
 - \mathbb{H}(R \cdot \tilde{x}(t) - \theta_{th}\cdot  (\tau - 1)\cdot \tilde{f}(t-1) - \theta_{th})\bigg)\\
 =& \tilde{f}(t-1) - \frac{1}{\tau} \bigg(\tilde{f}(t-1) +  \frac{u_{rest}}{\theta_{th}(\tau-1)} 
 - \mathbb{H}(\frac{R}{\theta_{th}(\tau - 1)}  \tilde{x}(t) - \tilde{f}(t-1) - \frac{1}{\tau - 1})\bigg)
\end{align*}
The subgradient $\tilde{g}(y;x)$ of the objective function $\mathcal{L}(y;x) = \frac{1}{2} y^2 +\frac{u_{rest}}{\theta_{th}(\tau-1)} y +\text{ReLU}(\frac{R}{\theta_{th}(\tau - 1)}x - y - \frac{1}{\tau-1}) $ is 
\begin{equation}
    \tilde{g}(y;x) = y + \frac{u_{rest}}{\theta_{th}(\tau-1)} - \text{ReLU}(\frac{R}{\theta_{th}(\tau - 1)}x - y - \frac{1}{\tau-1})
\end{equation}
Therefore, we finally derive the approximation $ \tilde{x}(t) \to x $ of the subgradient method of constant step-size $\frac{1}{\tau}$ over the optimization problem $\min_{y\in \mathbb{R}} \mathcal{L}(y;x)$. If $u_{rest} = 0$, the solution to the problem is $\text{ReLU1}(\frac{R}{\theta_{th}(\tau - 1)}x - \frac{1}{\tau-1})$.
\begin{equation}
\tilde{f}(t) = \tilde{f}(t-1) - \frac{1}{\tau} \cdot \tilde{g}\big(\tilde{f}(t-1) ; \tilde{x}(t)\big)
\end{equation}
\end{proof}

\begin{theorem}
\label{thm:if_neuron_convergence_analysis}
Let $f^*(x) = \argmin_{y \in \mathbb{R}}\mathcal{L}(y;x)$ be the minimizer of $ \mathcal{L}(y;x) = \text{ReLU}(x - y) + \frac{1}{2} y^2$ over a true input $x \in \mathbb{R}$, and its $t$-th approximation $\tilde{f}(t)$ be  defined as  equation~\ref{eq:main_if_subgradient}  and rate-coded input $\tilde{x}(t) = \frac{1}{t}\sum_{i=1}^{t} I(i)$. Denote $h(i) = \big(1 - \sqrt{\frac{i-1}{i+1}}\big)$. If $\Vert \tilde{f}(t) \Vert < M \in  \mathbb{R}^+$, then the approximation error $\Vert \tilde{f}(t) - f^*(x) \Vert$ is upper bounded as
\begin{align}
&\Vert \tilde{f}(t) - f^*(x)\Vert^2 
\le \frac{\Vert \tilde{f}(0) - f^*(x)\Vert^2}{t+1}
  \\
&+\underbrace{\frac{M + 1}{t+1}\sum_{i = 1}^{t} h(i)}_{\text{Nondifferentiability Error}}
+\underbrace{\frac{4}{t+1}\sum_{i = 1}^{t}  \min(\Vert x - \tilde{x}(i)\Vert, 1)}_{\text{Input Error}}\nonumber
\end{align}
\end{theorem}
\begin{proof}
For simplicity, we denote true objective function $\mathcal{L}(y) = \mathcal{L}(y;x) = \text{ReLU}(x - y) + \frac{1}{2} y^2$ and the estimated objective function $\mathbf{L}(y) = \mathcal{L}(y;\tilde{x}(t+1)) = f(\tilde{x}(t+1) - y) + \frac{1}{2} y^2$. Respectively, we denote their minimizers $f^* = f^*(x) = \argmin_{y \in \mathbb{R}}\mathcal{L}(y)$ and $\mathbf{f}^* =  \argmin_{y \in \mathbb{R}}\mathbf{L}(y)$. We also use the fact that $f^* = \text{ReLU1}(x)$, $\mathbf{f}^* = \text{ReLU1}(\tilde{x}(t+1))$. 

From equation~\ref{eq:if_subgradient} , $\tilde{f}(t) 
= \tilde{f}(t-1) - \frac{1}{t+1} \nabla_y \mathbf{L}(\tilde{f}(t-1))$. Hence
\begin{align}
\Vert \tilde{f}(t+1) - f^*\Vert^2
&= \Vert \tilde{f}(t) 
- \frac{1}{t+2} \nabla_y \mathbf{L}(\tilde{f}(t)) 
- f^*\Vert^2\\
&= \Vert \tilde{f}(t)- f^*\Vert^2
- \frac{2}{t+2} \nabla_y \mathbf{L}(\tilde{f}(t))
\big(\tilde{f}(t) - f^*\big)+ \frac{1}{(t+2)^2} \Vert \nabla_y \mathbf{L}(\tilde{f}(t)) \Vert^2 \label{pf:error_recurrence}
\end{align}
We first use the fact that $\mathbf{L}(y)$ is 1-strongly convex over y, since $\mathbf{L}(y) - \frac{1}{2} y^2$ is convex. First-order convexity over $\mathbf{L}$ show
\begin{equation}
\label{pf:if_convergence_firstorder}
    \mathbf{L}(f^*) 
    \ge \mathbf{L}(\tilde{f}(t)) 
    - \nabla_y \mathbf{L}(\tilde{f}(t))(\tilde{f}(t) - f^*) + \frac{1}{2} \Vert \tilde{f}(t) - f^* \Vert^2
\end{equation}
The inequality~\ref{pf:if_convergence_firstorder} applies to the right hand-side of the equation~\ref{pf:error_recurrence} as follows.
\begin{align}
    &  \frac{1}{(t+2)^2} \Vert \nabla_y \mathbf{L}(\tilde{f}(t))\Vert^2 - \frac{2}{t+2}(\mathbf{L}(\tilde{f}(t))  - \mathbf{L}(f^*))\nonumber\\
    &\ge  \frac{1}{(t+2)^2} \Vert \nabla_y \mathbf{L}(\tilde{f}(t))\Vert^2 - \frac{2}{t+2}\nabla_y \mathbf{L}(\tilde{f}(t))(\tilde{f}(t) - f^*) + \frac{1}{t+2} \Vert \tilde{f}(t) - f^* \Vert^2 \label{pf:ifrate_ineq_errorterms}
\end{align}
Also, $\mathbf{L}(y)$ is piece-wise smooth since it is a maximum of two smooth functions. Define each piece function as  $\mathbf{L}_i(y)$ for $i = 1,2$. In other words,
\begin{equation}
\mathbf{L}(y) = \max(\frac{1}{2}y^2 - y + \tilde{x}(t+1), \frac{1}{2}y^2) = \max(\mathbf{L}_1(y), \mathbf{L}_2(y))
\end{equation}
The difference of two piece functions is a line. 
\begin{align}
\Vert \mathbf{L}_1(y) - \mathbf{L}_2(y)\Vert = \Vert y - \tilde{x}(t+1)\Vert\\
\therefore\mathbf{L}_i(y) \ge \mathbf{L}(y) - \Vert y - \tilde{x}(t+1) \Vert \label{pf:piecewise_smooth_pieces}
\end{align}
Suppose  $\mathbf{L}(y) = \mathbf{L}_i(y)$ at a point $y \in \mathbb{R}$.  For any $\delta \in \mathbb{R}$, 
\begin{align}
    \mathbf{L}_i(y + \delta) &\ge \mathbf{L}(y + \delta) - \Vert y + \delta - \tilde{x}(t+1) \Vert & (\because \text{Inequality }\eqref{pf:piecewise_smooth_pieces})\\
    &\ge \mathbf{L}(y + \delta) - (\Vert y - \tilde{x}(t+1) \Vert + \Vert  \delta  \Vert ) & (\because \text{Triangle Inequality of Norm})\\
    &\ge  \mathbf{L}(y + \delta) - \Vert \delta \Vert & (\because \text{Positiveness of Norm})\label{pf:piecewise_smooth_ineq}
\end{align} 
Note that the error term $\Vert y + \delta - \tilde{x}(t+1) \Vert$ appears only when $\tilde{x}(t+1) \in [y, y + \delta]$, i.e., when the segment $[y,y+\delta]$ crosses the singularity $\tilde{x}(t+1)$. The error maximizes only if $y = \tilde{x}(t+1)$.\\
The smoothness constant is $2$ for both $\mathbf{L}_1$ and $\mathbf{L}_2$, since
\begin{align*}
\Vert \nabla_y \mathbf{L}_1(x_1) - \nabla_y \mathbf{L}_1(x_2) \Vert &= \Vert 2x_1 - 1 - 2x_2 + 1 \Vert = 2 \Vert x_1 - x_2 \Vert\\
\Vert \nabla_y \mathbf{L}_2(x_1) - \nabla_y \mathbf{L}_2(x_2) \Vert&= \Vert 2x_1 - 2x_2 \Vert = 2 \Vert x_1 - x_2 \Vert
\end{align*}
We now use the smoothness of $\mathbf{L}_i$ and the inequality~\ref{pf:piecewise_smooth_pieces}, \ref{pf:piecewise_smooth_ineq} to show the main result. Without loss of generality, let $\mathbf{L}(\tilde{f}(t)) = \mathbf{L}_i(\tilde{f}(t))$. For any $\alpha \in \mathbb{R}$, by the property of 2-smooth function,
\begin{align}
\mathbf{L}_i(\tilde{f}(t) - \alpha \nabla_y \mathbf{L}_i(\tilde{f}(t))) 
&\le \mathbf{L}_i(\tilde{f}(t)) + \nabla_y \mathbf{L}_i(\tilde{f}(t))) (-\alpha \nabla_y \mathbf{L}_i(\tilde{f}(t)))) 
+ \frac{2}{2} \alpha^2 \Vert \nabla_y \mathbf{L}_i(\tilde{f}(t)))\Vert^2&(\because \text{2-smooth})\nonumber\\
&= \mathbf{L}_i(\tilde{f}(t)) + (\alpha^2 - \alpha) \Vert \nabla_y \mathbf{L}_i(\tilde{f}(t)))\Vert^2\nonumber\\
\mathbf{L}_i(\tilde{f}(t) - \alpha \nabla_y \mathbf{L}_i(\tilde{f}(t))) &\ge 
\mathbf{L}(\tilde{f}(t) - \alpha \nabla_y \mathbf{L}_i(\tilde{f}(t))) - \Vert \tilde{f}(t) - \alpha \nabla_y \mathbf{L}_i(\tilde{f}(t)) - \tilde{x}(t+1) \Vert &(\because \eqref{pf:piecewise_smooth_pieces}~)\nonumber\\
&\ge \mathbf{L}(\tilde{f}(t) - \alpha \nabla_y \mathbf{L}_i(\tilde{f}(t))) 
- \alpha \Vert \nabla_y \mathbf{L}_i(\tilde{f}(t)) \Vert &(\because \eqref{pf:piecewise_smooth_ineq}~)\nonumber\\
\therefore \mathbf{L}(\tilde{f}(t) - \alpha \nabla_y \mathbf{L}_i(\tilde{f}(t))) &\le \mathbf{L}(\tilde{f}(t)) + (\alpha^2 - \alpha) \Vert \nabla_y \mathbf{L}_i(\tilde{f}(t)))\Vert^2 + \alpha \Vert \nabla_y \mathbf{L}_i(\tilde{f}(t)) \Vert \label{pf:piecewise_left_ineq}
\end{align}
Since $f^* = \text{ReLU1}(x)$, $\mathbf{f}^* = \text{ReLU1}(\tilde{x}(t+1))$ and  $x^2 - y^2 \ge -\Vert x+y\Vert\Vert x-y\Vert$,
\begin{align}
\mathbf{L}(\tilde{f}(t)& - \alpha \nabla_y \mathbf{L}_i(\tilde{f}(t)))
\ge \mathbf{L}(\mathbf{f}^*) \qquad(\because \mathbf{f}^* =  \argmin_{y \in \mathbb{R}}\mathbf{L}(y))\\ 
&= \mathbf{L}(f^*) + \big(\mathbf{L}(\mathbf{f}^*) - \mathbf{L}(f^*)\big)\nonumber\\
&= \mathbf{L}(f^*) + \bigg(\text{ReLU}(\tilde{x}(t+1) - \mathbf{f}^*) + \frac{1}{2}(\mathbf{f}^*)^2- \text{ReLU}(\tilde{x}(t+1) - f^*) - \frac{1}{2}(f^*)^2\bigg)\nonumber\\
&= \mathbf{L}(f^*) + \bigg(\frac{1}{2}(\mathbf{f}^*)^2 - \frac{1}{2}(f^*)^2 + \text{ReLU}(\tilde{x}(t+1) - \mathbf{f}^*) - \text{ReLU}(\tilde{x}(t+1) - f^*)\bigg) \nonumber\\
&\ge \mathbf{L}(f^*) 
- \frac{1}{2} \Vert \mathbf{f}^* + f^* \Vert \Vert \mathbf{f}^* - f^* \Vert
- \Vert \mathbf{f}^* - f^* \Vert \qquad(\because \text{Definition of ReLU})\nonumber\\
&= \mathbf{L}(f^*) 
- \frac{1}{2} \Vert \text{ReLU1}(\tilde{x}(t+1)) + \text{ReLU1}(x) + 2 \Vert \Vert \text{ReLU1}(\tilde{x}(t+1)) - \text{ReLU1}(x) \Vert \nonumber\\
&\ge \mathbf{L}(f^*) - \frac{1}{2} \cdot 4 \cdot \Vert \text{ReLU1}(\tilde{x}(t+1)) - \text{ReLU1}(x)\Vert \nonumber\\
&\ge \mathbf{L}(f^*) - 2 \Vert \min(x - \tilde{x}(t+1), 1) \Vert \label{pf:piecewise_right_ineq}
\end{align}
Merging two inequalities~\eqref{pf:piecewise_left_ineq} and~\eqref{pf:piecewise_right_ineq}, we yield the following inequality.
\begin{align*}
 \mathbf{L}(\tilde{f}(t)) + (\alpha^2 - \alpha) \Vert \nabla_y \mathbf{L}_i(\tilde{f}(t)))\Vert^2 + \Vert \nabla_y \mathbf{L}_i(\tilde{f}(t)) \Vert
&\ge 
 \mathbf{L}(f^*) - 2 \Vert \min(x - \tilde{x}(t+1), 1)  \Vert \\
 2 \Vert\min(x - \tilde{x}(t+1), 1)  \Vert + \alpha \Vert \nabla_y \mathbf{L}_i(\tilde{f}(t)) \Vert 
 &\ge 
 (\alpha - \alpha^2) \Vert \nabla_y \mathbf{L}_i(\tilde{f}(t)))\Vert^2 
 - (\mathbf{L}(\tilde{f}(t)) - \mathbf{L}(f^*))\\
 \frac{4}{t+2} \Vert \min(x - \tilde{x}(t+1), 1)  \Vert +  \frac{2\alpha}{t+2}\Vert \nabla_y \mathbf{L}_i(\tilde{f}(t)) \Vert 
 &\ge 
  \frac{2(\alpha - \alpha^2)}{t+2}\Vert \nabla_y \mathbf{L}_i(\tilde{f}(t)))\Vert^2 
 - \frac{2}{t+2} (\mathbf{L}(\tilde{f}(t)) - \mathbf{L}(f^*))
\end{align*}
We let $\alpha = \frac{1}{2} - \frac{1}{2}\sqrt{\frac{t}{t+2}}$, since $\alpha = \frac{1}{2} \pm \frac{1}{2}\sqrt{1 - \frac{2}{t+2}} = \frac{1}{2} \pm \frac{1}{2}\sqrt{\frac{t}{t+2}}$ satisfies $2(\alpha - \alpha^2) = \frac{1}{t+2}$. 
\begin{align*}
&\frac{4}{t+2} \Vert \min(x - \tilde{x}(t+1), 1) \Vert +  \frac{1}{t+2}\bigg(1 - \sqrt{\frac{t}{t+2}}\bigg)\Vert \nabla_y \mathbf{L}_i(\tilde{f}(t)) \Vert \\
 &\ge 
  \frac{1}{(t+2)^2}\Vert \nabla_y \mathbf{L}_i(\tilde{f}(t)))\Vert^2 
 - \frac{2}{t+2} (\mathbf{L}(\tilde{f}(t)) - \mathbf{L}(f^*))\\
 &\ge \frac{1}{(t+2)^2} \Vert \nabla_y \mathbf{L}(\tilde{f}(t))\Vert^2 - \frac{2}{t+2}\nabla_y \mathbf{L}(\tilde{f}(t))(\tilde{f}(t) - f^*) +\frac{1}{t+2} \Vert \tilde{f}(t) - f^* \Vert^2\qquad(\because \text{Inequality}~\eqref{pf:ifrate_ineq_errorterms}) \\
&= \Vert \tilde{f}(t+1) - f^*\Vert^2 - \frac{t+1}{t+2}\Vert \tilde{f}(t) - f^*\Vert^2 \qquad (\because \text{Equation }\eqref{pf:error_recurrence})\\
&\therefore \Vert \tilde{f}(t+1) - f^*\Vert^2
\le \frac{t+1}{t+2}\Vert \tilde{f}(t) - f^*\Vert^2+ \frac{4}{t+2}  \Vert\min(\tilde{x}(t+1) - x, 1)   \Vert +  \frac{1}{t+2}\bigg(1 - \sqrt{\frac{t}{t+2}}\bigg)\Vert \nabla_y \mathbf{L}(\tilde{f}(t)) \Vert
\end{align*}
We now analyze the sub-gradient magnitude term $\Vert \nabla_y \mathbf{L}(\tilde{f}(t)) \Vert$. By definition of $\mathbf{L}$ and triangle inequality,
\begin{align*}
\Vert \nabla_y \mathbf{L}(\tilde{f}(t)) \Vert
= \Vert \tilde{f}(t) - \mathbb{H}(\tilde{x}(t+1) - 
  \tilde{f}(t) )\Vert
\le \Vert \tilde{f}(t) \Vert + 1
\end{align*}
We assumed that $\Vert \tilde{f}(t) \Vert \le M \in \mathbb{R}^+$ for any $t \in \mathbb{N}$. Thus,
\begin{align}
&\Vert \tilde{f}(t+1) - f^*\Vert^2
\le \frac{t+1}{t+2}\Vert \tilde{f}(t) - f^*\Vert^2+ \frac{4}{t+2}  \Vert \tilde{x}(t+1) - x  \Vert 
+  \frac{1}{t+2}\bigg(1 - \sqrt{\frac{t}{t+2}}\bigg)(M + 1) \label{pf:ifrate_error_recurrence_relation}
\end{align}
Solving the recurrence relation~\eqref{pf:ifrate_error_recurrence_relation} through time $t$ leads to the upper bound of the  approximation error.
\begin{equation}
\label{pf:ifrate_error_inequality_final}
\Vert \tilde{f}(t+1) - f^*\Vert^2 \le   \frac{1}{t+2}\Vert \tilde{f}(0) - f^*\Vert^2 + \frac{4}{t+2}\sum_{i = 1}^{t+1}  \Vert \min(x - \tilde{x}(i), 1) \Vert +  \frac{M + 1}{t+2}\sum_{i = 0}^{t}\bigg(1 - \sqrt{\frac{i}{i+2}}\bigg)
\end{equation}
Recovering the notations $f^* \to f^*(x)$ and  $t+1 \to t$, the above inequality becomes the desired inequality.
\begin{equation}
\label{pf:ifrate_error_inequality_final}
\Vert \tilde{f}(t) - f^*(x)\Vert^2 \le   \frac{1}{t+1}\Vert \tilde{f}(0) - f^*(x)\Vert^2 + \frac{4}{t+1}\sum_{i = 1}^{t}  \Vert \min(x - \tilde{x}(i), 1) \Vert +  \frac{M + 1}{t+1}\sum_{i = 1}^{t}\bigg(1 - \sqrt{\frac{i-1}{i+1}}\bigg)
\end{equation}
\end{proof}
\begin{lemma}
\label{lem:harmonic_convergence}
The harmonic number $H_n = \sum_{i=1}^{n} \frac{1}{i}$ satisfies $\lim_{n\to\infty} \frac{1}{n}H_n = 0$.
\end{lemma}
\begin{proof}
Let $1 \le k \le n$. For every $i \le k$, $\frac{1}{i} \le 1$. For every $k < i \le n$, $\frac{1}{i} \le \frac{1}{k}$. Thus
\begin{equation*}
H_n = \sum_{i=1}^{n} \frac{1}{i} \le 1 \cdot k + \frac{1}{k} \cdot (n-k) = k + \frac{1}{k} - 1
\end{equation*}
If we choose $\sqrt{n} \le k \le \sqrt{n}+1$, then $k-1 \le \sqrt{n}$ and $\frac{1}{k} \le \sqrt{n}$. Thus, $H_n \le 2 \sqrt{n}$, and 
\begin{equation*}
    \frac{1}{n}H_n \le \frac{2}{\sqrt{n}}
\end{equation*}
Therefore, $\lim_{n\to\infty} \frac{1}{n}H_n = 0$.
\end{proof}
\begin{lemma}
\label{lem:h_i}
Let $t \in \mathbb{N}$. 
$\lim_{t \to \infty}\frac{1}{t+1}\sum_{i = 1}^{t}\big(1 - \sqrt{\frac{i-1}{i+1}}\big) = 0$
\end{lemma}
\begin{proof}
\begin{align}
0 &\le 1 - \sqrt{\frac{i-1}{i+1}} = 1 - \frac{\sqrt{(i-1)(i+1)}}{i+1}
= \frac{i + 1 - \sqrt{(i-1)(i+1)}}{i+1} \\
&\le \frac{i + 1 - \sqrt{(i-1)(i-1)}}{i+1}
= \frac{i + 1 - (i-1)}{i+1}= \frac{2}{i+1}\\
\therefore 0 &\le \frac{1}{t+1}\sum_{i = 1}^{t}\big(1 - \sqrt{\frac{i-1}{i+1}}\big) \le \frac{1}{t+1}\sum_{i = 1}^{t}\frac{2}{i+1} 
\end{align}
By Lemma~\ref{lem:harmonic_convergence},
\begin{align*}
    \lim_{t \to \infty} \frac{1}{t+1}\sum_{i = 1}^{t}\frac{2}{i+1} = \lim_{t \to \infty} \frac{2}{t+1} (H_{t} - 1) = 0 - 0 = 0
\end{align*}
Therefore, $0 \le \lim_{t \to \infty}\frac{1}{t+1}\sum_{i = 1}^{t}\big(1 - \sqrt{\frac{i-1}{i+1}}\big) \le 0$, and thus $\lim_{t \to \infty}\frac{1}{t+1}\sum_{i = 1}^{t}\big(1 - \sqrt{\frac{i-1}{i+1}}\big) = 0$.
\end{proof}
\begin{corollary}
\label{cor:conv_guarantee_input}
        Let $\Vert \tilde{f}(t) \Vert < M \in \mathbb{R}^+$, then
    \begin{itemize}
    \item (Exact Input) If $\tilde{x}(t) = x$, then $\Vert \tilde{f}(t) - f^*\Vert \to 0$ as $t \to \infty$.
    \item (Deterministic Input) If $\Vert \tilde{x}(t) - x \Vert = \mathcal{O}(\frac{1}{t})$, then $\Vert \tilde{f}(t) - f^*\Vert \to 0$ as $t \to \infty$.
    \item (Stochastic Input) If $\mathbb{E}[\Vert \tilde{x}(t) - x \Vert] = \mathcal{O}(\frac{1}{t})$, then $\mathbb{E}[\Vert \tilde{f}(t) - f^*\Vert] \to 0$ as $t \to \infty$. 
    \end{itemize}
\end{corollary}
\begin{proof}
\textbf{(Exact Input)} We first prove it for the case of exact input, i.e., if $\tilde{x}(t) = x$, then $\Vert \tilde{f}(t) - f^*\Vert \to 0$ as $t \to \infty$.
\begin{align*}
    \Vert \tilde{f}(t) - f^*(x)\Vert^2 
&\le \frac{\Vert \tilde{f}(0) - f^*(x)\Vert^2}{t+1}
+\frac{M + 1}{t+1}\sum_{i = 1}^{t} \left((1 - \sqrt{\frac{i-1}{i+1}}\right)
+\frac{4}{t+1}\sum_{i = 1}^{t}  \min(\Vert x - \tilde{x}(i)\Vert, 1) \quad (\text{Theorem~\ref{thm:convergence_if}})\\
&= \frac{\Vert \tilde{f}(0) - f^*(x)\Vert^2}{t+1}
+\frac{M + 1}{t+1}\sum_{i = 1}^{t} \left(1 - \sqrt{\frac{i-1}{i+1}}\right)\qquad  (\text{Assumption on }\tilde{x}(t))
\end{align*}
By Lemma~\ref{lem:harmonic_convergence} and~\ref{lem:h_i}, 
\begin{align*} \lim_{t \to \infty} \Vert \tilde{f}(t) - f^*(x)\Vert^2\le \lim_{t \to \infty}\frac{\Vert \tilde{f}(0) - f^*(x)\Vert^2}{t+1}
+\lim_{t \to \infty}\frac{M + 1}{t+1}\sum_{i = 1}^{t} \left( 1 - \sqrt{\frac{i-1}{i+1}}\right)
= 0 + (M+1) \cdot 1 \cdot 0 = 0
    \end{align*}
\textbf{(Deterministic Input)} We show the deterministic input case that if $\Vert \tilde{x}(t) - x \Vert = \mathcal{O}(\frac{1}{t})$, then  $\Vert \tilde{f}(t) - f^*\Vert \to 0$ as $t \to \infty$.
\begin{align*}
    \Vert \tilde{f}(t) - f^*(x)\Vert^2 
&\le \frac{\Vert \tilde{f}(0) - f^*(x)\Vert^2}{t+1}
+\frac{M + 1}{t+1}\sum_{i = 1}^{t} \left(1 - \sqrt{\frac{i-1}{i+1}}\right)
+\frac{4}{t+1}\sum_{i = 1}^{t}  \min(\Vert x - \tilde{x}(i)\Vert, 1) \quad (\text{Theorem~\ref{thm:convergence_if}})\\
&\le \frac{\Vert \tilde{f}(0) - f^*(x)\Vert^2}{t+1}
+\frac{M + 1}{t+1}\sum_{i = 1}^{t} \left(1 - \sqrt{\frac{i-1}{i+1}}\right)
+\frac{4}{t+1}\sum_{i = 1}^{t}  \min\left(\frac{C}{t}, 1\right)\quad  (\text{Assumption on }\tilde{x}(t))\\
&\le \frac{\Vert \tilde{f}(0) - f^*(x)\Vert^2}{t+1}
+\frac{M + 1}{t+1}\sum_{i = 1}^{t} \left( 1 - \sqrt{\frac{i-1}{i+1}}\right)
+\frac{4C}{t+1}\sum_{i = 1}^{t}  \frac{1}{t}
\end{align*}
By Lemma~\ref{lem:harmonic_convergence} and~\ref{lem:h_i}, 
\begin{align*} \lim_{t \to \infty} \Vert \tilde{f}(t) - f^*(x)\Vert^2&\le \lim_{t \to \infty}\frac{\Vert \tilde{f}(0) - f^*(x)\Vert^2}{t+1}
+\lim_{t \to \infty}\frac{M + 1}{t+1}\sum_{i = 1}^{t} \left( 1 - \sqrt{\frac{i-1}{i+1}}\right)
+\lim_{t \to \infty}\frac{4C}{t+1}\sum_{i = 1}^{t}  \frac{1}{t}\\
&= 0 + (M+1) \cdot 1 \cdot 0 +  4C \cdot 1 \cdot 0 = 0
    \end{align*}
\textbf{(Stochastic Input)} We now demonstrate the stochastic input case that if $\mathbb{E}[\Vert \tilde{x}(t) - x \Vert] = \mathcal{O}(\frac{1}{t})$, then $\mathbb{E}[\Vert \tilde{f}(t) - f^*\Vert] \to 0$ as $t \to \infty$.
      \begin{align*}
    \mathbb{E}\left[\Vert \tilde{f}(t) - f^*(x)\Vert^2\right] 
&\le \mathbb{E}\left[\frac{\Vert \tilde{f}(0) - f^*(x)\Vert^2}{t+1}
+\frac{M + 1}{t+1}\sum_{i = 1}^{t} (1 - \sqrt{\frac{i-1}{i+1}})
+\frac{4}{t+1}\sum_{i = 1}^{t}  \min(\Vert x - \tilde{x}(i)\Vert, 1)\right]\\
& = \frac{\Vert \tilde{f}(0) - f^*(x)\Vert^2}{t+1}
+\frac{M + 1}{t+1}\sum_{i = 1}^{t} (1 - \sqrt{\frac{i-1}{i+1}})
+\mathbb{E}\left[\frac{4}{t+1}\sum_{i = 1}^{t}  \min(\Vert x - \tilde{x}(i)\Vert, 1)\right]\\
&\le \frac{\Vert \tilde{f}(0) - f^*(x)\Vert^2}{t+1}
+\frac{M + 1}{t+1}\sum_{i = 1}^{t} (1 - \sqrt{\frac{i-1}{i+1}})
+\frac{4}{t+1}\sum_{i = 1}^{t}  \min(\frac{C}{t}, 1)\quad  (\text{Assumption on }\tilde{x}(t))\\
&\le \frac{\Vert \tilde{f}(0) - f^*(x)\Vert^2}{t+1}
+\frac{M + 1}{t+1}\sum_{i = 1}^{t} ( 1 - \sqrt{\frac{i-1}{i+1}})
+\frac{4C}{t+1}\sum_{i = 1}^{t}  \frac{1}{t}
\end{align*}
By Lemma~\ref{lem:harmonic_convergence} and~\ref{lem:h_i}, 
\begin{align*} \lim_{t \to \infty} \mathbb{E}\left[\Vert \tilde{f}(t) - f^*(x)\Vert^2\right] &\le \lim_{t \to \infty}\frac{\Vert \tilde{f}(0) - f^*(x)\Vert^2}{t+1}
+\lim_{t \to \infty}\frac{M + 1}{t+1}\sum_{i = 1}^{t} \left( 1 - \sqrt{\frac{i-1}{i+1}}\right)
+\lim_{t \to \infty}\frac{4C}{t+1}\sum_{i = 1}^{t}  \frac{1}{t}\\
&= 0 + (M+1) \cdot 1 \cdot 0 +  4C \cdot 1 \cdot 0 = 0
    \end{align*}
\end{proof}

\begin{theorem}
\label{thm:signgd_dynamics}
Let an input activation $\tilde{x}(t)$ be signed schedule coded over an $N$ weighted input spike trains, i.e., 
\begin{align}
    &\tilde{x}^{(k)}(t) = \tilde{x}^{(k)}(t-1 ) - \sum_{i=1}^N W_i\bigg(  \eta(t)(2I_i^{(k)}(t) -1)\bigg)\label{eq:signsgd_activation}
\end{align}
$\tilde{f}(t)  = \tilde{f}(t-1) - \eta(t) \big(2\cdot s(t)- 1\big)$ is signed schedule coded with the spike output $s(t)$ of signGD-based neuronal dynamics.\\
If $\alpha_1$, $\alpha_2$ $\beta_1$, $\beta_2$ and $\eta$ satisfies $\eta(1) = \alpha_2(1) = \beta_2(1)$ and
\begin{align}
\frac{\eta(t)}{\eta(t-1)} = \beta_1(t)\frac{\beta_2(t)}{\beta_2(t-1)} = \frac{1}{\alpha_1(t-1)}\frac{\alpha_2(t)}{\alpha_2(t-1)} \label{eq:signsgd_coeff_condition}
\end{align}
Then the dynamical system of $\tilde{f}(t)$ 
is equivalent to a sign gradient descent method
\begin{align}
\tilde{f}&(t) = \tilde{f}(t-1) - \eta(t) \text{sgn} (\nabla_y\mathcal{L}(\tilde{f}(t-1); \tilde{x}(t))) \label{eq:signsgd_subg}
\end{align}
\end{theorem}
\begin{proof}
We start from the equation~(\ref{eq:signsgd_subg}) to derive the signGD-based neuron in Definition~\ref{def:signgd_dynamics}. We define $u(t), v(t)$ to be
\begin{align}
u(t) = \frac{\beta_2(t-1)}{\eta(t-1)}\tilde{f}(t-1)
\qquad~v(t) = \frac{\alpha_2(t)}{\eta(t)}\tilde{x}(t) \qquad~s(t) = \mathbb{H}\big(\nabla_y\mathcal{L}(\frac{\eta(t-1)}{\beta_2(t-1)}u(t); \frac{\eta(t)}{\alpha_2(t)}v(t))\big)\label{pf:def_u_v_signgd}
\end{align}
From equation~\ref{eq:signsgd_subg}, by substituting $\tilde{f}(t)$ with $u(t)$ and $\tilde{x}(t)$ with $v(t)$,
\begin{align}
\frac{\eta(t)}{\beta_2(t)}u(t+1) &= \frac{\eta(t-1)}{\beta_2(t-1)}u(t) - \eta(t) \text{sgn} (\nabla_y\mathcal{L}(\frac{\eta(t-1)}{\beta_2(t-1)}u(t); \frac{\eta(t)}{\alpha_2(t)}v(t))) \nonumber\\
u(t+1) &= \frac{\beta_2(t)}{\eta(t)}\frac{\eta(t-1)}{\beta_2(t-1)}u(t) - \frac{\beta_2(t)}{\eta(t)}\eta(t) \text{sgn} (\nabla_y\mathcal{L}(\frac{\eta(t-1)}{\beta_2(t-1)}u(t); \frac{\eta(t)}{\alpha_2(t)}v(t)))\nonumber\\
&= \frac{\beta_2(t)}{\eta(t)}\frac{\eta(t-1)}{\beta_2(t-1)}u(t) - \beta_2(t) \text{sgn} (\nabla_y\mathcal{L}(\frac{\eta(t-1)}{\beta_2(t-1)}u(t); \frac{\eta(t)}{\alpha_2(t)}v(t))) \label{pf:signgd_u}
\end{align}
By equation~\eqref{eq:signsgd_coeff_condition}, $\beta_1(t) = \frac{\eta(t)}{\beta_2(t)}\frac{\beta_2(t-1)}{\eta(t-1)}$. The equation~\eqref{pf:signgd_u} becomes identical to the equation~\eqref{eq:signsgd_reset} in the definition of signGD-based neuron.
\begin{equation}
u(t+1) = \beta_1(t)u(t) - \beta_2(t) (2\cdot s(t) - 1)
\end{equation}
We now derive the membrane equation for $v^{(k)}(t)$.
\begin{align}
v^{(k)}(t+1) - \alpha_1(t)v^{(k)}(t) &= \frac{\alpha_2(t+1)}{\eta(t+1)}\tilde{x}^{(k)}(t+1) - \alpha_1(t)\frac{\alpha_2(t)}{\eta(t)}\tilde{x}^{(k)}(t)\nonumber\\
&= \frac{\alpha_2(t+1)}{\eta(t+1)} \bigg(\tilde{x}^{(k)}(t ) - \sum_{i=1}^N W_i\big(  \eta(t+1)(2I_i^{(k)}(t+1) -1)\big)\bigg)- \alpha_1(t)\frac{\alpha_2(t)}{\eta(t)}\tilde{x}^{(k)}(t)
\end{align}
By equation~\eqref{eq:signsgd_coeff_condition}, $\alpha_1(t)\frac{\alpha_2(t)}{\eta(t)} = \frac{\alpha_2(t+1)}{\eta(t+1)}$. Since $I^{(k)}(t+1) = \sum_{i=1}^N W_i(I_i^{(k)})(t+1)$ and $W = \sum_{i=1}^N W_i$,
\begin{align}
v^{(k)}(t+1) - \alpha_1(t)v^{(k)}(t) 
&= \frac{\alpha_2(t+1)}{\eta(t+1)} \bigg(\tilde{x}^{(k)}(t ) - \sum_{i=1}^N W_i\big(  \eta(t+1)(2I_i^{(k)}(t+1) -1)\big)\bigg)- \alpha_1(t)\frac{\alpha_2(t)}{\eta(t)}\tilde{x}^{(k)}(t)\nonumber\\
&= - \frac{\alpha_2(t+1)}{\eta(t+1)}  \sum_{i=1}^N W_i\big(  \eta(t+1)(2I_i^{(k)}(t+1) -1)\big)\nonumber\\
&= - \alpha_2(t+1) \sum_{i=1}^N W_i(2I_i^{(k)}(t+1) -1) = - \alpha_2(t+1) \big(2(\sum_{i=1}^N W_iI_i^{(k)}(t+1)) -W\big)\nonumber\\
&=  - \alpha_2(t+1) (2 \cdot I^{(k)} - W)\nonumber\\
\therefore v^{(k)}(t+1) &= \alpha_1(t)v^{(k)}(t) - \alpha_2(t+1) (2 \cdot I^{(k)} - W)
\end{align}
We thus derive the equation for $v^{(k)}(t)$. The proof direction from signGD-based neuronal dynamics to optimization algorithm can be similarly derived by defining $\tilde{f}(t)$ and $\tilde{x}(t)$ as~\eqref{pf:def_u_v_signgd}.
\end{proof}
\begin{corollary}
\label{thm:signgd_unary}
SignGD-based neuronal dynamics~(Def.~\ref{def:signgd_dynamics}) satisfying Eq.~\ref{eq:main_signsgd_coeff_condition} and $\eta(t) = \alpha_2(t) = \beta_2(t)$ is equivalent to signGD (Eq.~\ref{eq:main_signsgd_subg}) if for $s(t)$ (Eq.~\ref{eq:signsgd_fire}),
\begin{itemize}
\item \textbf{(ReLU)} $\mathcal{L}(y;x) = \frac{1}{2}\Vert y - \text{ReLU(x)} \Vert^2$ and $s(t) = \mathbb{H}\big(v(t)\big)\mathbb{H}\big(u(t) - \beta_1(t)v(t)\big) + \mathbb{H}\big(-v(t)\big)\mathbb{H}\big(u(t)\big)$
\item \textbf{(Sigmoid approximation of GELU)}~\cite{hendrycks2016gelu}) $\mathcal{L}(y;x) = \frac{1}{2}\Vert y - \frac{x}{1+e^{-1.702x}} \Vert^2$ and $s(t) = \mathbb{H}\big((1 + (e^{-1.702})^{v(t)}) u(t) - \beta_1(t) v(t) \big)$
\item \textbf{(LeakyReLU)} $\mathcal{L}(y;x) = \frac{1}{2}\Vert y - \text{LeakyReLU(x, $\delta$)} \Vert^2$, where $\delta$ is the negative slope, and $s(t) = \mathbb{H}(v(t))\mathbb{H}\big(u(t) -\beta_1(t)v(t)\big) + \mathbb{H}(-v(t))\mathbb{H}\big(u(t) - \delta \beta_1(t)v(t)\big) $
\end{itemize}
\end{corollary}
\begin{proof}
\textbf{(ReLU)} We first show that signGD-based neuronal dynamics with a spiking mechanism $s(t) = \mathbb{H}\big(v(t)\big)\mathbb{H}\big(u(t) - \beta_1(t)v(t)\big) + \mathbb{H}\big(-v(t)\big)\mathbb{H}\big(u(t)\big)$ is equivalent to the signGD algorithm (Eq.~\ref{eq:main_signsgd_subg}) with the objective function $\mathcal{L}(y;x) = \frac{1}{2}\Vert y - \text{ReLU}(x)\Vert^2$, which has $\text{ReLU}(x)$ as its minimizer.

Since $\mathcal{L}(y;x) = \frac{1}{2}\Vert y - \text{ReLU}(x)\Vert^2$, $\nabla_y\mathcal{L}(y;x) = (y - \text{ReLU}(x)) = \mathbb{H}(x)(y - x) + \mathbb{H}(-x)y$
\begin{align*}
    s(t) &= \mathbb{H}\bigg(\nabla_y\mathcal{L}\big(\frac{\eta(t-1)}{\beta_2(t-1)}u(t);\frac{\eta(t)}{\alpha_2(t)}v(t)\big)\bigg)\\
    &= \mathbb{H}\bigg(\mathbb{H}(\frac{\eta(t)}{\alpha_2(t)}v(t))\big(\frac{\eta(t-1)}{\beta_2(t-1)}u(t) - \frac{\eta(t)}{\alpha_2(t)}v(t)\big) + \mathbb{H}(-\frac{\eta(t)}{\alpha_2(t)}v(t))\frac{\eta(t-1)}{\beta_2(t-1)}u(t)\bigg)\\
    &= \mathbb{H}\bigg(\frac{\eta(t)}{\alpha_2(t)}v(t)\bigg)\mathbb{H}\bigg(\frac{\eta(t-1)}{\beta_2(t-1)}u(t) - \frac{\eta(t)}{\alpha_2(t)}v(t)\bigg) + \mathbb{H}\bigg(-\frac{\eta(t)}{\alpha_2(t)}v(t)\bigg)\mathbb{H}\bigg(\frac{\eta(t-1)}{\beta_2(t-1)}u(t)\bigg)\\
    &= \mathbb{H}(v(t))\mathbb{H}\bigg(u(t) - \frac{\eta(t)\beta_2(t-1)}{\eta(t-1)\alpha_2(t)}v(t)\bigg) + \mathbb{H}(-v(t))\mathbb{H}(u(t))\\
    &= \mathbb{H}(v(t))\mathbb{H}\bigg(u(t) - \beta_1(t)\frac{\beta_2(t)}{\alpha_2(t)}v(t)\bigg) + \mathbb{H}(-v(t))\mathbb{H}(u(t))\\
    &= \mathbb{H}(v(t))\mathbb{H}\bigg(u(t) - \beta_1(t)v(t)\bigg) + \mathbb{H}(-v(t))\mathbb{H}(u(t))
\end{align*}
\textbf{(GELU)} We now show that signGD-based neuronal dynamics with a spiking mechanism $s(t) = \mathbb{H}\big((1 + (e^{-1.702})^{v(t)}) u(t) - \beta_1(t) v(t) \big)$ is equivalent to the signGD algorithm (Eq.~\ref{eq:main_signsgd_subg}) with the objective function $\mathcal{L}(y;x) = \frac{1}{2}\Vert y - \frac{x}{1+e^{-1.702x}} \Vert^2$, which has the sigmoid approximation of GELU function value $\frac{x}{1+e^{-1.702x}}$ as its minimizer.

Since $\mathcal{L}(y;x) = \frac{1}{2}\Vert y - g(x) \Vert^2$, $\nabla_y\mathcal{L}(y;x) = (y - g(x)) = y - \frac{x}{1+e^{-1.702x}}$
\begin{align*}
    s(t) &= \mathbb{H}\bigg(\nabla_y\mathcal{L}\big(\frac{\eta(t-1)}{\beta_2(t-1)}u(t);\frac{\eta(t)}{\alpha_2(t)}v(t)\big)\bigg)\\
    &= \mathbb{H}\bigg(\frac{\eta(t-1)}{\beta_2(t-1)}u(t) - \frac{\frac{\eta(t)}{\alpha_2(t)}v(t)}{1+e^{-1.702\frac{\eta(t)}{\alpha_2(t)}v(t)}}\bigg)\\
    &= \mathbb{H}\bigg((1+e^{-1.702\frac{\eta(t)}{\alpha_2(t)}v(t)})\frac{\eta(t-1)}{\beta_2(t-1)}u(t) - \frac{\eta(t)}{\alpha_2(t)}v(t)\bigg)\\
    &= \mathbb{H}\bigg((1+e^{-1.702\frac{\eta(t)}{\alpha_2(t)}v(t)})u(t) - \frac{\eta(t)\beta_2(t-1)}{\eta(t-1)\alpha_2(t)}v(t)\bigg)\\
    &= \mathbb{H}\bigg((1+e^{-1.702\frac{\eta(t)}{\alpha_2(t)}v(t)})u(t) - \beta_1(t)\frac{\beta_2(t)}{\alpha_2(t)}v(t)\bigg)\\
    &= \mathbb{H}\bigg((1+e^{-1.702v(t)})u(t) - \beta_1(t)v(t)\bigg)
\end{align*}

\textbf{(LeakyReLU)} Finally, we show that signGD-based neuronal dynamics with a spiking mechanism $s(t) = \mathbb{H}(v(t))\mathbb{H}\big(u(t) -\beta_1(t)v(t)\big) + \mathbb{H}(-v(t))\mathbb{H}\big(u(t) - \delta \beta_1(t)v(t)\big) $ is equivalent to the signGD algorithm (Eq.~\ref{eq:main_signsgd_subg}) with the objective function $\mathcal{L}(y;x) = \frac{1}{2}\Vert y - \text{LeakyReLU(x, $\delta$)} \Vert^2$, which has $\text{LeakyReLU(x, $\delta$)}$ as its minimizer.

Since $\mathcal{L}(y;x) = \frac{1}{2}\Vert y - \text{LeakyReLU}(x, \delta) \Vert^2$, $\nabla_y\mathcal{L}(y;x) = y - \text{LeakyReLU}(x, \delta) = \mathbb{H}(x)(y -x) + \mathbb{H}(-x)(y - \delta x)$
\begin{align*}
    s(t) &= \mathbb{H}\bigg(\nabla_y\mathcal{L}\big(\frac{\eta(t-1)}{\beta_2(t-1)}u(t);\frac{\eta(t)}{\alpha_2(t)}v(t)\big)\bigg)\\
    &=\mathbb{H}(\frac{\eta(t)}{\alpha_2(t)}v(t))\mathbb{H}\big(\frac{\eta(t-1)}{\beta_2(t-1)}u(t) -\frac{\eta(t)}{\alpha_2(t)}v(t)\big) + \mathbb{H}(-\frac{\eta(t)}{\alpha_2(t)}v(t))\mathbb{H}\big(\frac{\eta(t-1)}{\beta_2(t-1)}u(t) - \delta \frac{\eta(t)}{\alpha_2(t)}v(t)\big)\\
    &=\mathbb{H}(v(t))\mathbb{H}\big(u(t) -\frac{\eta(t)\beta_2(t-1)}{\eta(t-1)\alpha_2(t)}v(t)\big) + \mathbb{H}(-v(t))\mathbb{H}\big(u(t) - \delta \frac{\eta(t)\beta_2(t-1)}{\eta(t-1)\alpha_2(t)}v(t)\big)\\
    &=\mathbb{H}(v(t))\mathbb{H}\big(u(t) -\beta_1(t)\frac{\beta_2(t)}{\alpha_2(t)}v(t)\big) + \mathbb{H}(-v(t))\mathbb{H}\big(u(t) - \delta \beta_1(t)\frac{\beta_2(t)}{\alpha_2(t)}v(t)\big)\\
    &=\mathbb{H}(v(t))\mathbb{H}\big(u(t) -\beta_1(t)v(t)\big) + \mathbb{H}(-v(t))\mathbb{H}\big(u(t) - \delta \beta_1(t)v(t)\big)
\end{align*}
\end{proof}

\begin{corollary}
\label{thm:signgd_max}
SignGD (Eq.~\ref{eq:main_signsgd_subg}) with $\mathcal{L}\big(y;(x_1,x_2)\big) = \frac{1}{2}\Vert y - \max(x_1, x_2) \Vert^2$ is equivalent to the signGD-based neuronal dynamics~(Def.~\ref{def:signgd_dynamics})
satisfying equation~\ref{eq:main_signsgd_coeff_condition}, $\alpha_2(t) = \beta_2(t)$ and 
\begin{equation*}
\begin{split}
s(t) =\mathbb{H}(v^{(1)}(t) - v^{(2)}(t))(u(t) -\beta_1(t) v^{(1)}(t))+ \mathbb{H}(v^{(2)}(t) - v^{(1)}(t))(u(t) - \beta_1(t)v^{(2)}(t))
\end{split}
\end{equation*}    
\end{corollary}
\begin{proof}
Since $\mathcal{L}\big(y;(x_1,x_2)\big) = \frac{1}{2}\Vert y - \max(x_1, x_2) \Vert^2$, $\nabla_y\mathcal{L}(y;x_1, x_2) = y - \max(x_1, x_2) = \mathbb{H}(x_1 - x_2)(y -x_1) + \mathbb{H}(x_2 - x_1)(y - x_2)$.
\begin{align*}
    s(t) &= \mathbb{H}\bigg(\nabla_y\mathcal{L}\big(\frac{\eta(t-1)}{\beta_2(t-1)}u(t);\frac{\eta(t)}{\alpha_2(t)}v(t)\big)\bigg)\\
    &=  \mathbb{H}(\frac{\eta(t)}{\alpha_2(t)}v^{(1)}(t) - \frac{\eta(t)}{\alpha_2(t)}v^{(2)}(t))\mathbb{H}(\frac{\eta(t-1)}{\beta_2(t-1)}u(t) -\frac{\eta(t)}{\alpha_2(t)}v^{(1)}(t)) \\
    &+ \mathbb{H}(\frac{\eta(t)}{\alpha_2(t)}v^{(2)}(t) - \frac{\eta(t)}{\alpha_2(t)}v^{(1)}(t))\mathbb{H}(\frac{\eta(t-1)}{\beta_2(t-1)}u(t) - \frac{\eta(t)}{\alpha_2(t)}v^{(2)}(t))\\
    &= \mathbb{H}(v^{(1)}(t) - v^{(2)}(t))\mathbb{H}(u(t) -\frac{\eta(t)\beta_2(t-1)}{\eta(t-1)\alpha_2(t)}v^{(1)}(t))+ \mathbb{H}(v^{(2)}(t) - v^{(1)}(t))\mathbb{H}(u(t) - \frac{\eta(t)\beta_2(t-1)}{\eta(t-1)\alpha_2(t)}v^{(2)}(t))\\
    &= \mathbb{H}(v^{(1)}(t) - v^{(2)}(t))\mathbb{H}(u(t) -\frac{\eta(t)\beta_2(t-1)}{\eta(t-1)\alpha_2(t)}v^{(1)}(t))+ \mathbb{H}(v^{(2)}(t) - v^{(1)}(t))\mathbb{H}(u(t) - \frac{\eta(t)\beta_2(t-1)}{\eta(t-1)\alpha_2(t)}v^{(2)}(t))\\
    &= \mathbb{H}(v^{(1)}(t) - v^{(2)}(t))\mathbb{H}(u(t) -\beta_1(t)\frac{\beta_2(t)}{\alpha_2(t)} v^{(1)}(t))+ \mathbb{H}(v^{(2)}(t) - v^{(1)}(t))\mathbb{H}(u(t) - \beta_1(t)\frac{\beta_2(t)}{\alpha_2(t)}v^{(2)}(t))\\
    &= \mathbb{H}(v^{(1)}(t) - v^{(2)}(t))\mathbb{H}(u(t) -\beta_1(t) v^{(1)}(t))+ \mathbb{H}(v^{(2)}(t) - v^{(1)}(t))\mathbb{H}(u(t) - \beta_1(t)v^{(2)}(t))
\end{align*}
\end{proof}
\begin{corollary}
\label{thm:signgd_square}
SignGD (Eq.~\ref{eq:main_signsgd_subg}) with $\mathcal{L}(y;x) =  \frac{1}{2}\Vert y - x^2 \Vert^2$ is equivalent to the signGD-based neuronal dynamics~(Def.~\ref{def:signgd_dynamics})
satisfying equation~\ref{eq:main_signsgd_coeff_condition}, $\eta(t) =\beta_2(t) = \alpha_2(t)$ and
\begin{equation*}
\begin{split}
s(t) =& \mathbb{H}\bigg(u(t) - \beta_1(t)v(t)^2\bigg)
\end{split}
\end{equation*}    
\end{corollary}
\begin{proof}
Since $\mathcal{L}(y;x) =  \frac{1}{2}\Vert y - x^2 \Vert^2$, $\nabla_y\mathcal{L}(y;x) = y - x^2$.
\begin{align*}
    s(t) &= \mathbb{H}\bigg(\nabla_y\mathcal{L}\big(\frac{\eta(t-1)}{\beta_2(t-1)}u(t);\frac{\eta(t)}{\alpha_2(t)}v(t)\big)\bigg)\\
    &=\mathbb{H}\bigg(\frac{\eta(t-1)}{\beta_2(t-1)}u(t) - (\frac{\eta(t)}{\alpha_2(t)}v(t))^2\bigg)\\
    &=\mathbb{H}\bigg(u(t) - \frac{\eta(t)^2\beta_2(t-1)}{\eta(t-1)\alpha_2(t)^2}v(t)^2\bigg)\\
    &=\mathbb{H}\bigg(u(t) - \beta_1(t)\frac{\eta(t)\beta_2(t)}{\alpha_2(t)^2}v(t)^2\bigg)\\
    &=\mathbb{H}\bigg(u(t) - \beta_1(t)v(t)^2\bigg)
\end{align*}
\end{proof}
\begin{corollary}
\label{thm:signgd_misr}
SignGD (Eq.~\ref{eq:main_signsgd_subg}) with $ \mathcal{L}(y;x_1,x_2) = \Vert y - \frac{x_1}{\sqrt{x_2}} \Vert^2$ is equivalent to the signGD-based neuronal dynamics~(Def.~\ref{def:signgd_dynamics})
satisfying equation~\ref{eq:main_signsgd_coeff_condition}, and $\eta(t) = \alpha_2(t) = \beta_2(t)$ and 
\begin{equation*}
\begin{split}
s(t) = &~\mathbb{H}\big(u(t)\big)\mathbb{H}\big(v_1(t)\big)\mathbb{H}\big(v_2(t)u(t)^2 - v_1(t)^2\big) 
 \\
&+ \mathbb{H}\big(-u(t)\big)\mathbb{H}\big(-v_1(t)\big)\mathbb{H}\big(v_1(t)^2 - v_2(t)u(t)^2\big)  
\\
&+ \mathbb{H}(u(t))\mathbb{H}(-v_1(t))
\end{split}
\end{equation*}    
\end{corollary}
\begin{proof}
Since $\mathcal{L}(y;x_1,x_2) = \frac{1}{2}\Vert y - \frac{x_1}{\sqrt{x_2}} \Vert^2$, $\nabla_y\mathcal{L}(y;x) = y - \frac{x_1}{\sqrt{x_2}}$.
\begin{align*}
\mathbb{H}(y - \frac{x_1}{\sqrt{x_2}}) 
 = \mathbb{H}\big(y\big)\mathbb{H}\big(x_1\big)\mathbb{H}\big(x_2y^2 - x_1^2\big) + \mathbb{H}\big(-y\big)\mathbb{H}\big(-x_1\big)\mathbb{H}\big(x_1^2 - x_2y^2\big) + \mathbb{H}(y)\mathbb{H}(-x_1)
\end{align*}
Applying the equation to the spike mechanism $s(t)$,
\begin{align*}
    s(t) =& \mathbb{H}\bigg(\nabla_y\mathcal{L}\big(\frac{\eta(t-1)}{\beta_2(t-1)}u(t);\frac{\eta(t)}{\alpha_2(t)}v(t)\big)\bigg)\\
    =&\mathbb{H}\big(\frac{\eta(t-1)}{\beta_2(t-1)}u(t)\big)\mathbb{H}\big(\frac{\eta(t)}{\alpha_2(t)}v^{(1)}(t)\big)\mathbb{H}\big(\frac{\eta(t)}{\alpha_2(t)}v^{(2)}(t)(\frac{\eta(t-1)}{\beta_2(t-1)}u(t))^2 - (\frac{\eta(t)}{\alpha_2(t)}v^{(1)}(t))^2\big) \\
    &+ \mathbb{H}\big(-\frac{\eta(t-1)}{\beta_2(t-1)}u(t)\big)\mathbb{H}\big(-\frac{\eta(t)}{\alpha_2(t)}v^{(1)}(t)\big)\mathbb{H}\big((\frac{\eta(t)}{\alpha_2(t)}v^{(1)}(t))^2 - \frac{\eta(t)}{\alpha_2(t)}v^{(2)}(t)(\frac{\eta(t-1)}{\beta_2(t-1)}u(t))^2\big) \\
    &+ \mathbb{H}(\frac{\eta(t-1)}{\beta_2(t-1)}u(t))\mathbb{H}(-\frac{\eta(t)}{\alpha_2(t)}v^{(1)}(t))\\    =&\mathbb{H}\big(u(t)\big)\mathbb{H}\big(v^{(1)}(t)\big)\mathbb{H}\big(\frac{\alpha_2(t)}{\eta(t)}\frac{\eta(t-1)^2}{\beta_2(t-1)^2}v^{(2)}(t)u(t)^2 - v^{(1)}(t)^2\big) \\
    &+ \mathbb{H}\big(-u(t)\big)\mathbb{H}\big(-v^{(1)}(t)\big)\mathbb{H}\big(v^{(1)}(t)^2 - \frac{\alpha_2(t)}{\eta(t)}\frac{\eta(t-1)^2}{\beta_2(t-1)^2}v^{(2)}(t)u(t)^2\big) \\
    &+ \mathbb{H}(u(t))\mathbb{H}(-v^{(1)}(t))\\    =&\mathbb{H}\big(u(t)\big)\mathbb{H}\big(v^{(1)}(t)\big)\mathbb{H}\big(v^{(2)}(t)u(t)^2 - v^{(1)}(t)^2\big) \\
    &+ \mathbb{H}\big(-u(t)\big)\mathbb{H}\big(-v^{(1)}(t)\big)\mathbb{H}\big(v^{(1)}(t)^2 - v^{(2)}(t)u(t)^2\big) \\
    &+ \mathbb{H}(u(t))\mathbb{H}(-v^{(1)}(t))
\end{align*}
\end{proof}

\begin{figure}[!ht]
\vskip 0.2in
\begin{center}
    \subfigure[Inverse LR, $\eta(t) = \frac{1}{t+1}$ $\approx$ IF + Rate of $\theta_{th} = R = 1$]{
        \includegraphics[width=0.31\linewidth]{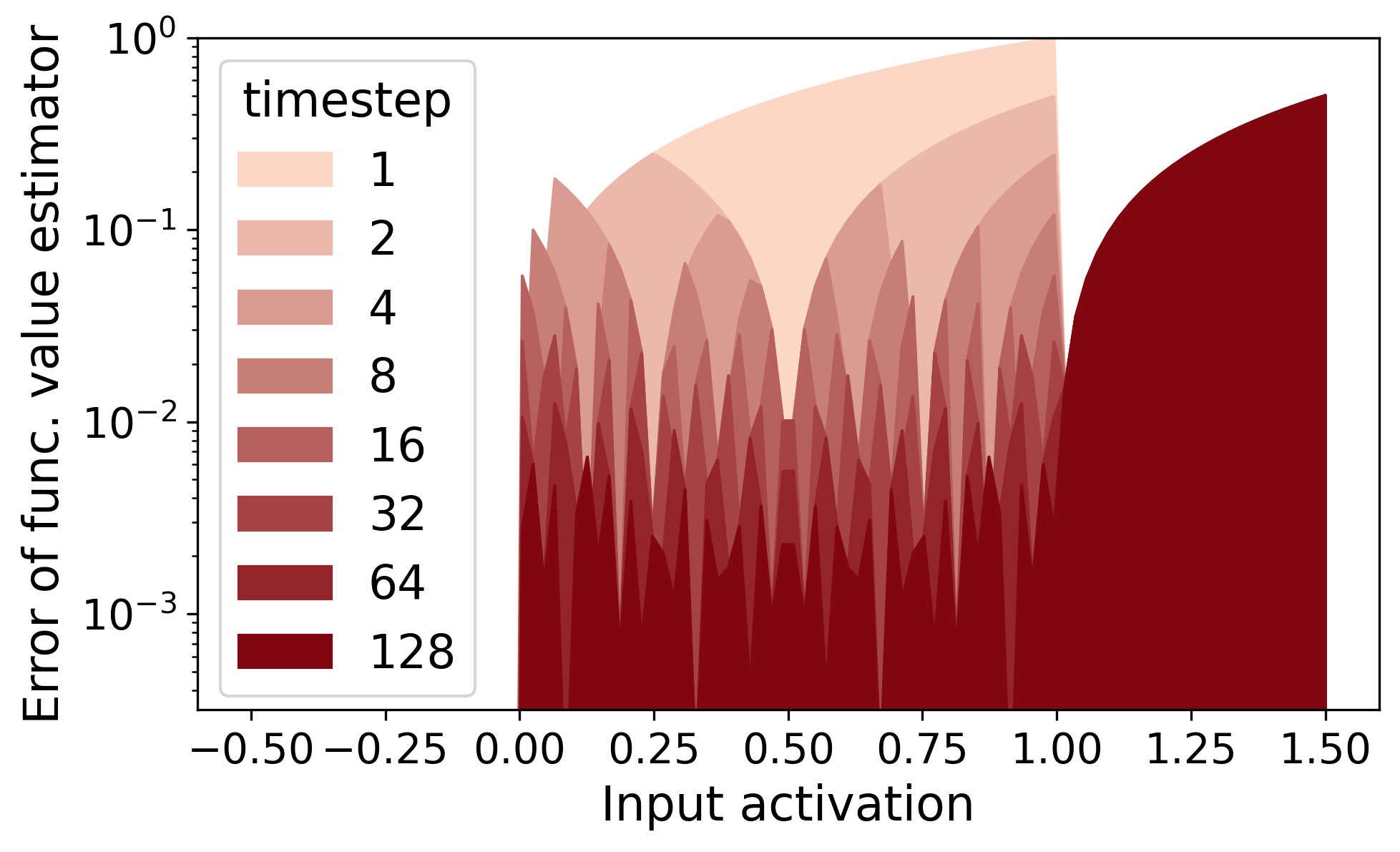}
        \label{fig:toy_subgradient_inverse_1.0}
    }
    \hfill
    \subfigure[Constant LR, $\eta(t) = 0.05$ $\approx$ LIF + EMA of $\theta_{th}= 1, \tau = 20,u_{rest} = 0$]{
        \includegraphics[width=0.31\linewidth]{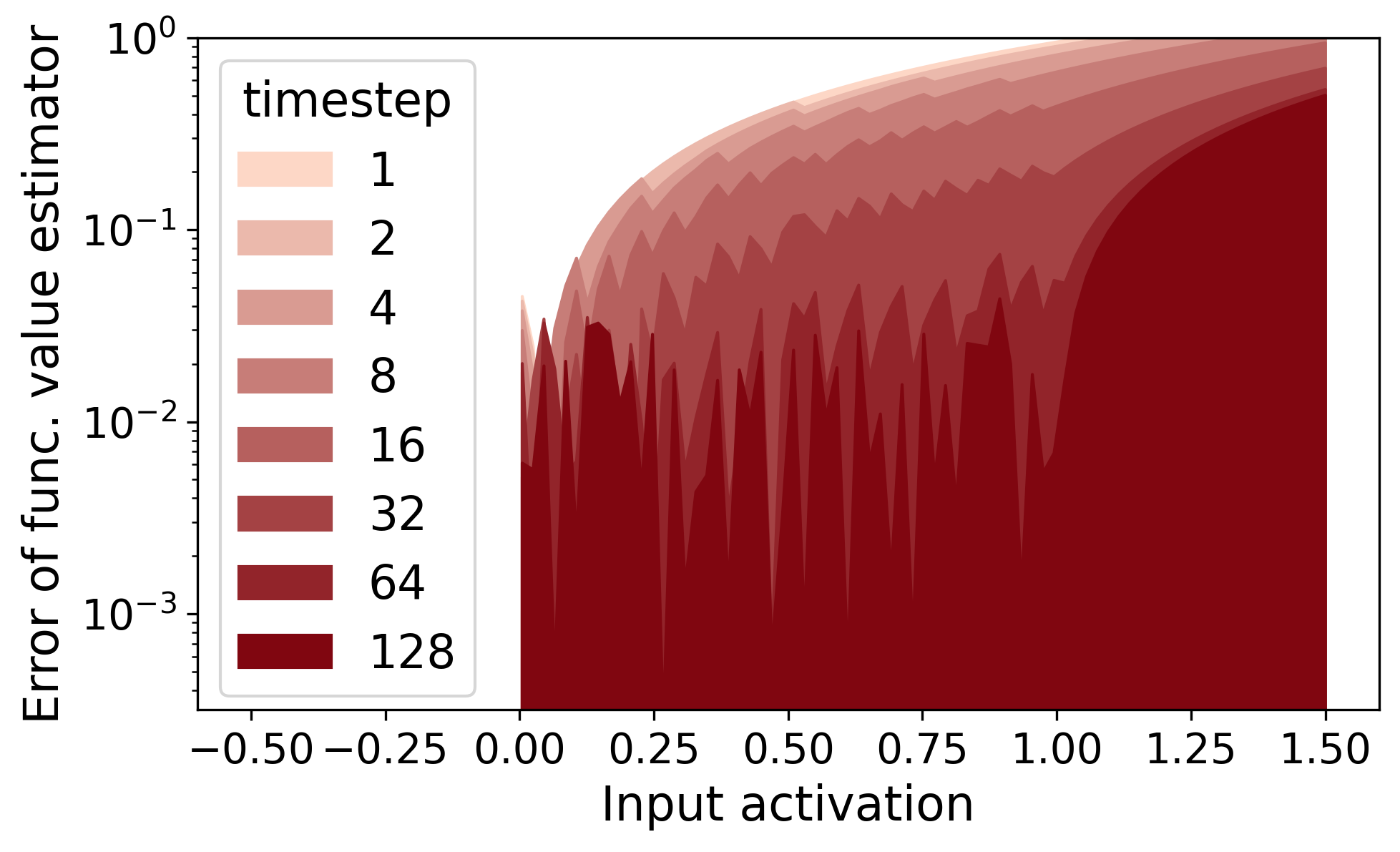}
        \label{fig:toy_subgradient_constant_0.05}
    }
    \hfill
    \subfigure[Exp. LR, $\eta(t) = (0.9)^t$.]{
        \includegraphics[width=0.31\linewidth]{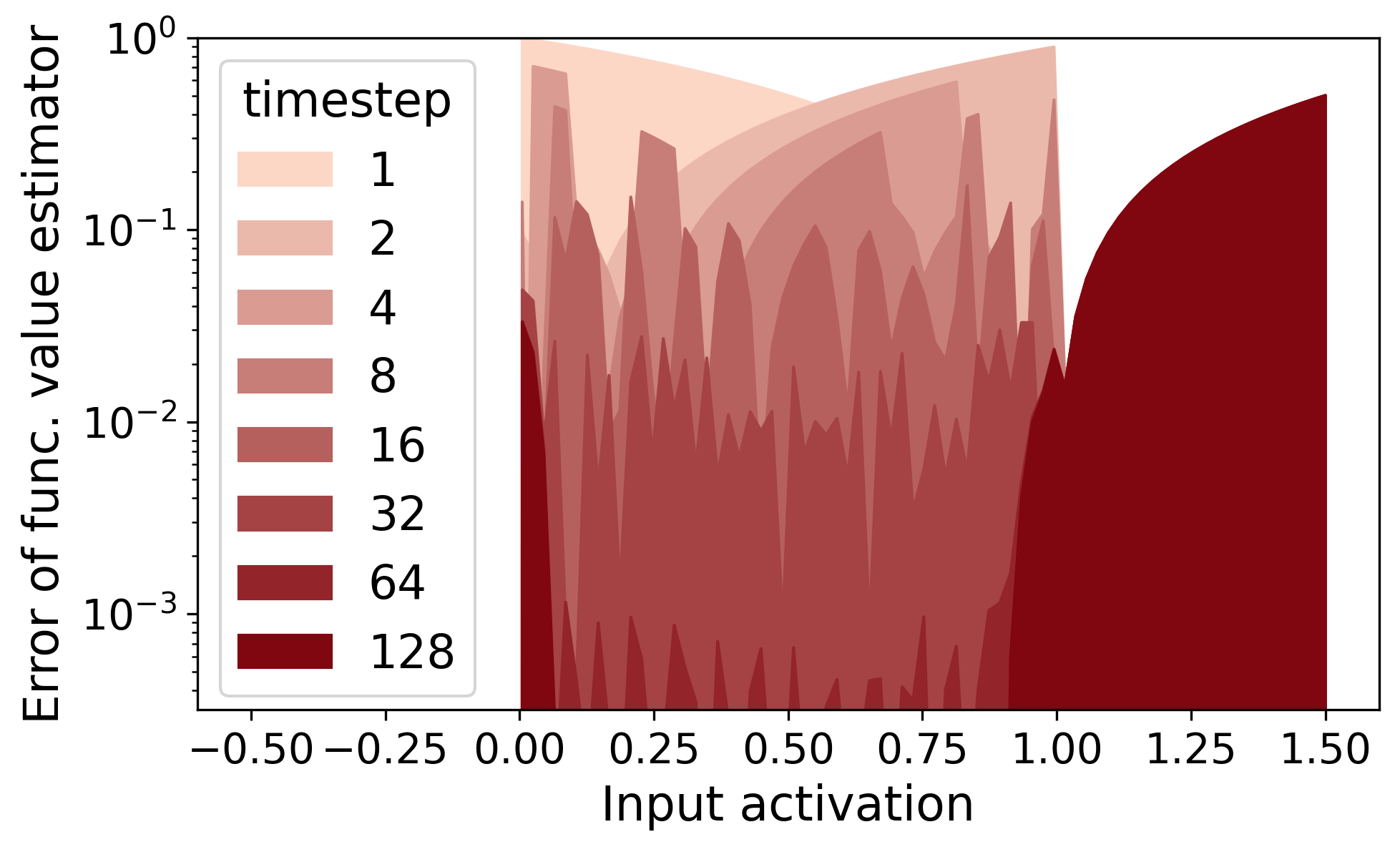}
        \label{fig:toy_subgradient_exponential}
    }
\caption{Toy experiment on subgradient-based neuronal dynamics with generalized learning rate schedule (Definition~\ref{thm:lr_schedule}). Time-evolution of error (Y-axis, log-scale) between ReLU function and spike-decoded neuron output, over a float input activation (X-axis) }
\label{fig:subgradient_lr_schedule_toy_experiment}
\end{center}
\vskip -0.2in
\end{figure}
\section{Generalizing learning rate schedule over subgradient-based neuronal dynamics}
To verify the effect of signGD dynamics on acceleration, we generalize the learning rate schedule of subgradient method-based neuronal dynamics, which underlies the simple integrate-and-fire models. In Figure~\ref{fig:subgradient_lr_schedule_toy_experiment}, we also empirically verify that the subgradient-based neuronal dynamics are well-defined by approximating the clipped ReLU function with the subgradient-based neuron. We design neuronal dynamics with three different learning rate schedules: inverse (Figure~\ref{fig:toy_subgradient_inverse_1.0}), exponential (Figure~\ref{fig:toy_subgradient_exponential}), and constant (Figure~\ref{fig:toy_subgradient_constant_0.05}) schedule. We formulate them as follows. 
\label{sec:lr_schedule_subgradient}
\begin{definition}
\textbf{(Schedule coding)} Let a spike train $s(t) \in \{0,1\}$, a step size schedule $\eta(t) \in \mathbb{R}^{+}$ for $t \in \mathbb{N}$ and $y(0) = 0$. Schedule coding with $\eta(t)$ decodes an activation $y(t)$ from $s(t)$ as
\begin{equation}
y(t) = (1-\eta(t)) y(t-1) + \eta(t) s(t)\label{eq:lr_schedule_coding}
\end{equation}
\end{definition}
\begin{definition}
\label{def:subgradient_based_neuron}
\textbf{(Subgradient-based neuronal dynamics)} Positive coefficients $\alpha_{i}(t) \in \mathbb{R}^{+}$, $\beta_{i}(t) \in \mathbb{R}^{+}$, and step size schedule $\eta(t) \in \mathbb{R}^{+}$ for $i=1,2$ and $t \in \mathbb{N}$. The subgradient-based neuronal dynamics over membrane potential $u(t) \in \mathbb{R}$ is 
\begin{align}
u_{pre}(t+1) &= \alpha(t)u(t) +  \gamma(t+1)I(t+1) \label{eq:lr_schedule_charge}\\
s(t+1) &= \mathbb{H}\big(u_{pre}(t+1) - u_{pre}(0)\prod^{t}_{j=0}\alpha(j) \big)\label{eq:lr_schedule_fire}\\
u(t+1) &= u_{pre}(t+1) -\beta(t+1) s(t+1)\label{eq:lr_schedule_reset}
\end{align}
\end{definition}
\begin{theorem}
\label{thm:lr_schedule}
Let $\mathcal{L}(y;x)= \text{ReLU}\big( x - y\big) + \frac{1}{2} y^2$ and its sub-gradient $\tilde{g}(y;x)$ over $y$. If input $\tilde{x}(t) = (1-\eta(t)) x(t-1) + \eta(t) I(t)$ and output activation $\tilde{f}(t) = (1-\eta(t)) \tilde{f}(t-1) + \eta(t) s(t)$ are schedule coded with $\eta(t)$ and $\eta$, $\alpha$, $\beta$, $\gamma$ satisfies for $i,t \in \mathbb{N}$
\begin{align}
\frac{\beta(t)}{\eta(t)}(1-\eta(t)) =  \frac{\beta(t-1)}{\eta(t-1)}\alpha(t-1) \label{eq:coeff_condition}\\
\frac{\eta(i)}{\eta(t)} \prod_{j=i+1}^{t}(1-\eta(j))
= \frac{\gamma(i)}{\beta(t)}\prod_{j=i}^{t-1} \alpha(j)\label{eq:coeff_prod_condition}
\end{align}
Then the dynamical system of $\tilde{f}(t)$ and $\tilde{x}(t)$ 
is equivalent to a subgradient method
\begin{align}
\tilde{f}(t&) = \tilde{f}(t-1) - \eta(t) \cdot \tilde{g}\big(\tilde{f}(t-1); \tilde{x}(t)\big) \label{eq:lr_schedule_subg}
\end{align}
\end{theorem}

\begin{proof}
Subgradient $\tilde{g}(t) = \tilde{f}(t-1)  - \mathbb{H}\big( \tilde{x}(t+1) - \tilde{f}(t) \big)$. We start from the neuronal dynamics (Definition~\ref{def:subgradient_based_neuron}) to derive the subgradient method. By definition of $\tilde{f}(t)$,
\begin{align*}
s(t)= \frac{1}{\eta(t)} \tilde{f}(t) - \bigg( \frac{1}{\eta(t)} - 1\bigg)\tilde{f}(t-1) 
\end{align*}
The equation~\eqref{eq:lr_schedule_reset} applies to equation~\eqref{eq:lr_schedule_charge} to formulate the recurrence relation of $\tilde{f}(t)$ and $u(t)$.
\begin{align}
u_{pre}(t+1) &= \alpha(t) u_{pre}(t) -\alpha(t)\beta(t)s(t) + \gamma(t+1) I(t+1)  \\
  \frac{\alpha(t)\beta(t)}{\eta(t)}\bigg(\tilde{f}(t) -  \big(1 - \eta(t)\big)\tilde{f}(t-1) \bigg)&= \alpha(t) u_{pre}(t) - u_{pre}(t+1) + \gamma(t+1) I(t+1) \label{eq:learning_rate_schedule_membrane_potential}
\end{align}
By equation~\eqref{eq:coeff_prod_condition}, $\frac{\alpha(t)\beta(t)}{\eta(t)}(1-\eta(t)) = \alpha(t) \frac{\alpha(t-1)\beta(t-1)}{\eta(t-1)}$
. With $\tilde{f}(0) = 0$, solving the recurrence relation yields
\begin{align*}
 \frac{\alpha(t)\beta(t)}{\eta(t)}  \tilde{f}(t)&= - u_{pre}(t+1) + u_{pre}(0)\prod^{t}_{j=0}\alpha(j)   + \sum_{i=1}^{t+1}\gamma(i) I(i) \prod_{j=i}^{t}\alpha(j)\\
    \tilde{f}(t)&= -\frac{\eta(t)}{\alpha(t)\beta(t)} u_{pre}(t+1) + \frac{\eta(t)}{\alpha(t)\beta(t)}u_{pre}(0)\prod^{t}_{j=0}\alpha(j)   + \frac{\eta(t)}{\alpha(t)\beta(t)}\sum_{i=1}^{t+1}\gamma(i) I(i) \prod_{j=i}^{t}\alpha(j)
\end{align*}
By equation~\eqref{eq:coeff_condition}, $\eta(i) \prod_{j=i+1}^{t}(1-\eta(j))
= \frac{\eta(t)}{\alpha(t)\beta(t)}\gamma(i)\prod_{j=i}^{t} \alpha(j)$. Hence,
\begin{align*}
    \tilde{f}(t)&= -\frac{\eta(t)}{\alpha(t)\beta(t)} u_{pre}(t+1) + \frac{\eta(t)}{\alpha(t)\beta(t)}u_{pre}(0)\prod^{t}_{j=0}\alpha(j)   + \sum_{i=1}^{t+1} \eta(i) \prod_{j=i+1}^{t}(1-\eta(j)) I(i)\\
    &= -\frac{\eta(t)}{\alpha(t)\beta(t)} u_{pre}(t+1) + \frac{\eta(t)}{\alpha(t)\beta(t)}u_{pre}(0)\prod^{t}_{j=0}\alpha(j)   + \tilde{x}(t+1)\\
    u_{pre}(t+1) &- u_{pre}(0)\prod^{t}_{j=0}\alpha(j) = \frac{\alpha(t)\beta(t)}{\eta(t)}\big( \tilde{x}(t+1) - \tilde{f}(t) \big)
\end{align*}
Substituting $u_{pre}$ with $\tilde{f}$ and $\tilde{x}$ in the definition of $s$,
\begin{align*}
    s(t+1) &= \mathbb{H}\big(u_{pre}(t+1) - u_{pre}(0)\prod^{t}_{j=0}\alpha(j) \big) =  \mathbb{H}\bigg( \frac{\alpha(t)\beta(t)}{\eta(t)} \big(\tilde{x}(t+1) - \tilde{f}(t)\big) \bigg) = \mathbb{H}\big( \tilde{x}(t+1) - \tilde{f}(t) \big)\\
    \therefore \tilde{f}(t) &= (1-\eta(t)) \tilde{f}(t-1) + \eta(t) \mathbb{H}\big( \tilde{x}(t) - \tilde{f}(t-1) \big)\\
    &= \tilde{f}(t-1) - \eta(t)\cdot  \tilde{g}(\tilde{f}(t-1); \tilde{x}(t))
\end{align*} 
\end{proof}

\begin{figure}[!ht]
\vskip -0.0in
\begin{center}
    \subfigure[IF neuron with $\theta_{th} = 1, R = 1$, float encoding.]{
        \includegraphics[width=0.31\linewidth]{        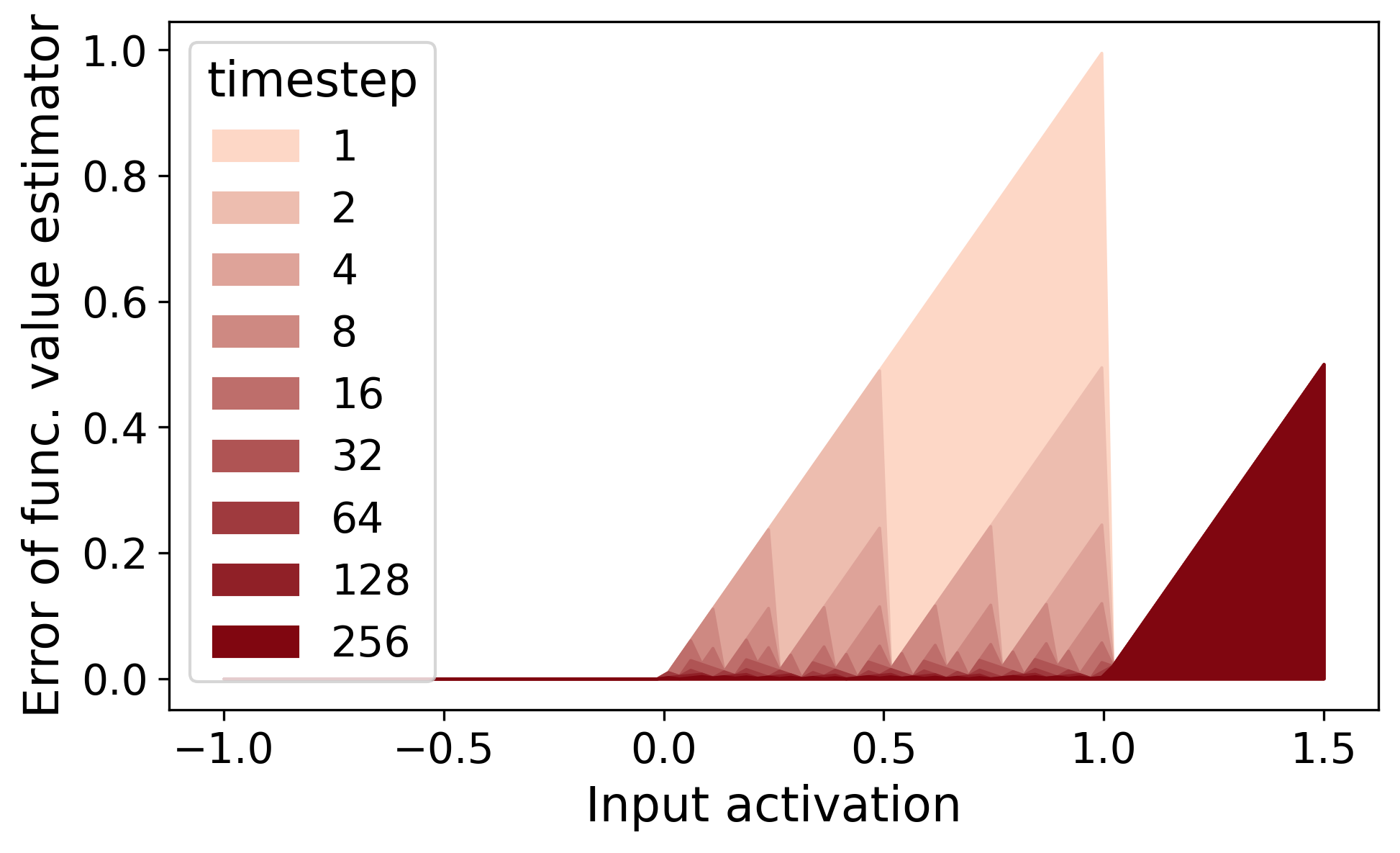}\label{fig:if_rate_error}
    }
    \hfill
    \subfigure[Fig.~\ref{fig:if_rate_error} transformed to $\tilde{f}(t)$]{
        \includegraphics[width=0.31\linewidth]{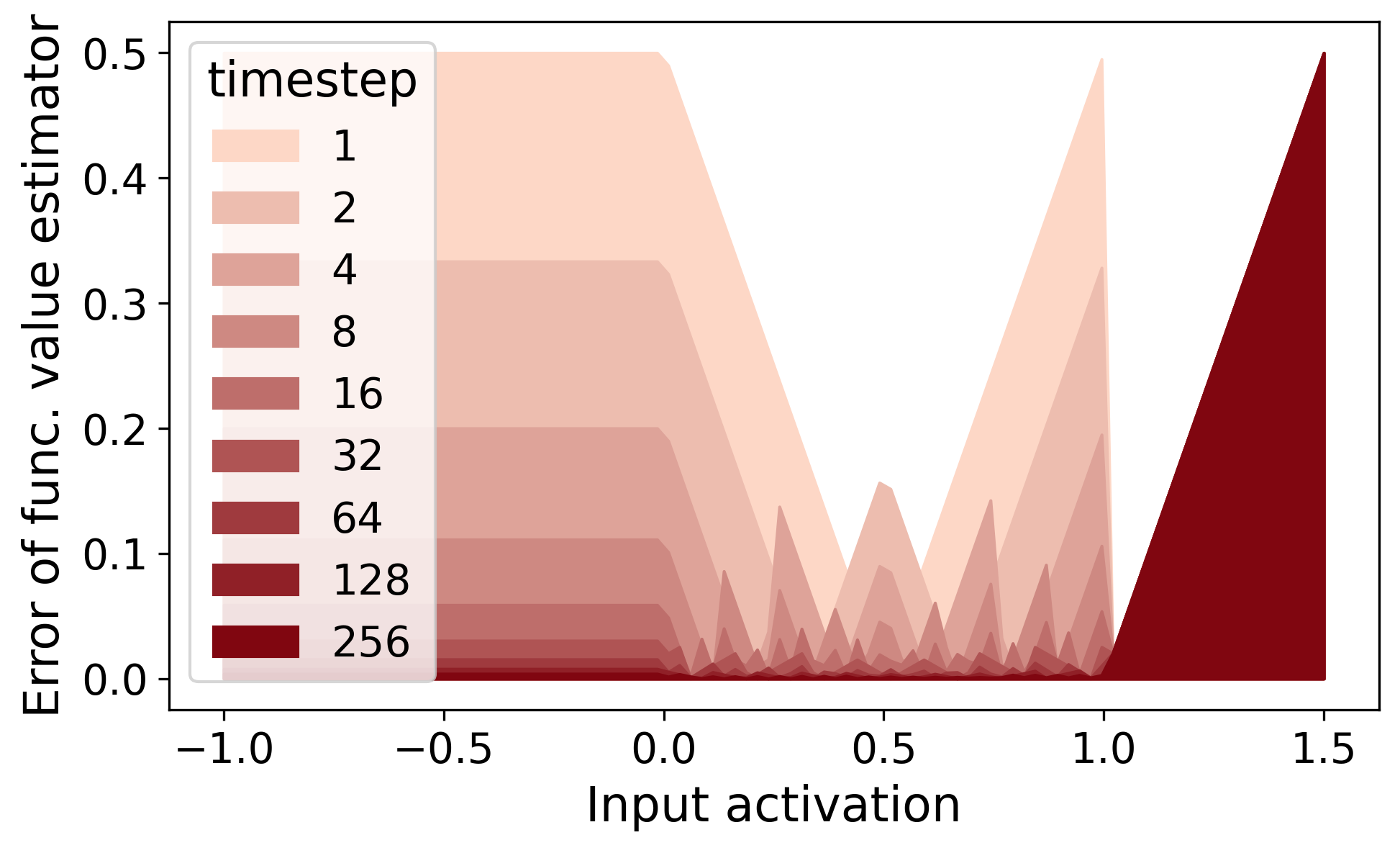}
    }
    \hfill
    \subfigure[Subgradient method on $\tilde{f}(t)$. Float encoding.]{
        \includegraphics[width=0.31\linewidth]{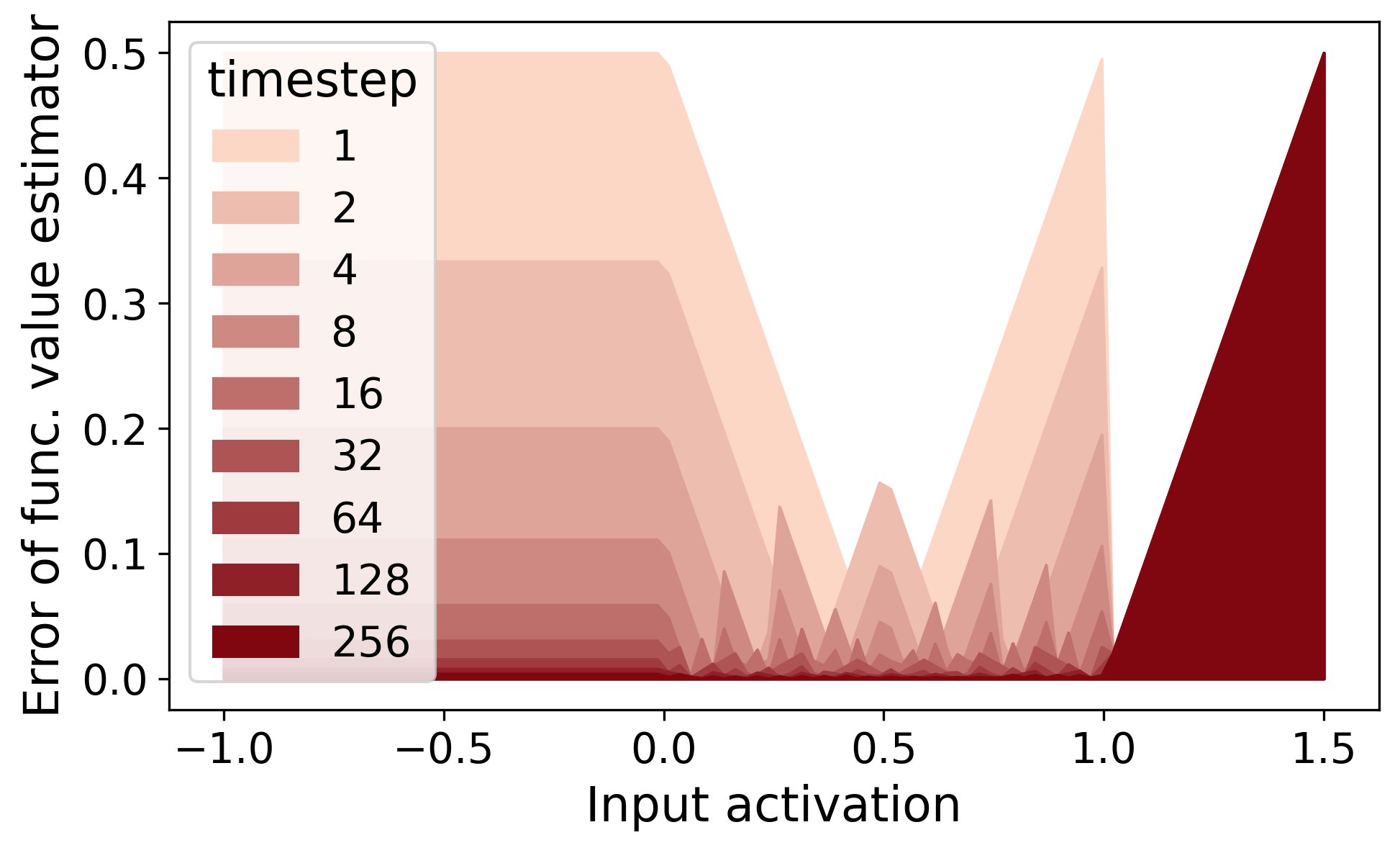}
        \label{fig:subgradient_error_float}
    }  
    \vfill
    \subfigure[IF neuron with $\theta_{th} = 1, R = 1$, poisson encoding.]{
        \includegraphics[width=0.31\linewidth]{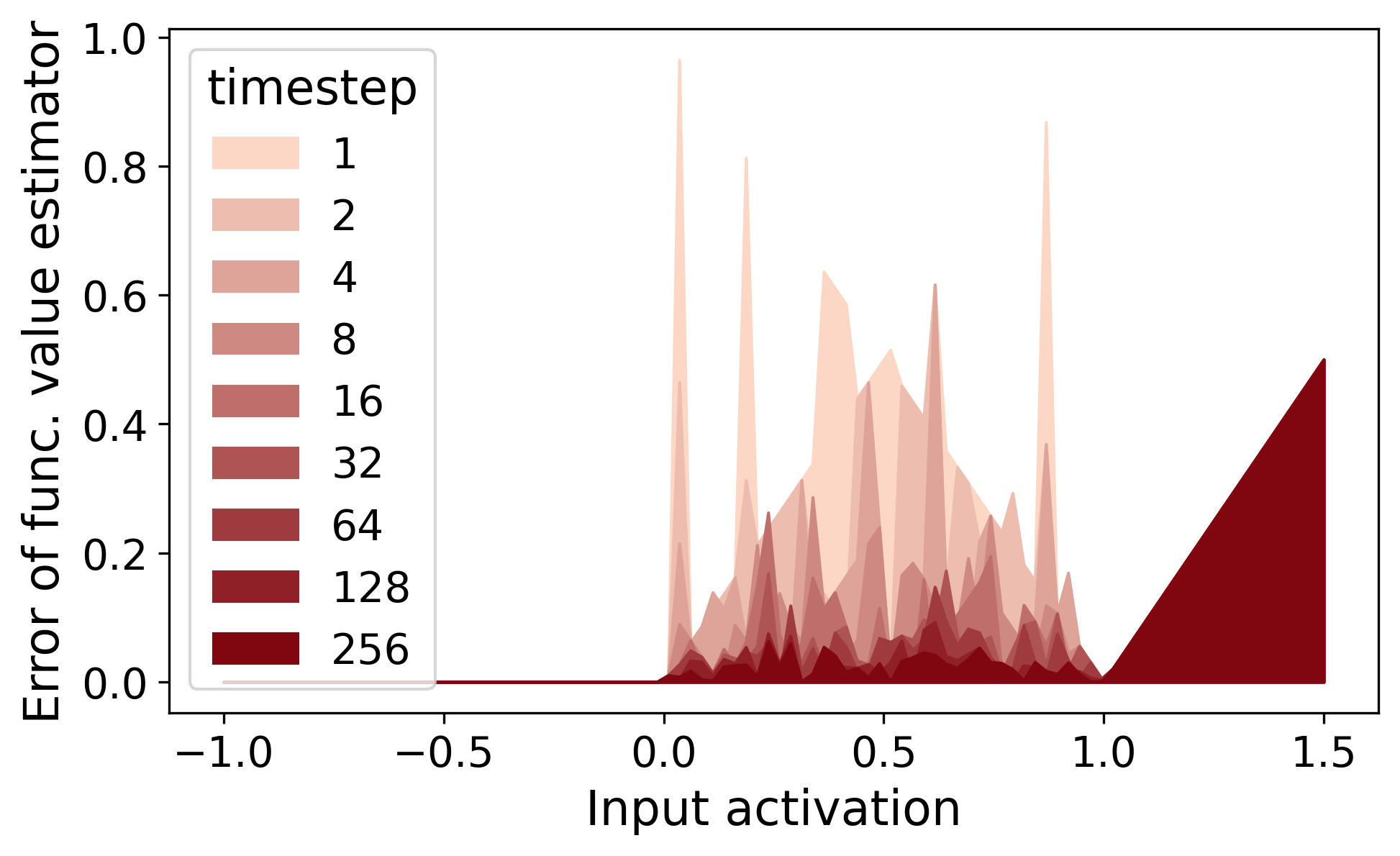}\label{fig:if_rate_poisson_error}
    }
    \hfill
    \subfigure[Fig.~\ref{fig:if_rate_poisson_error} transformed to $\tilde{f}(t)$]{
        \includegraphics[width=0.31\linewidth]{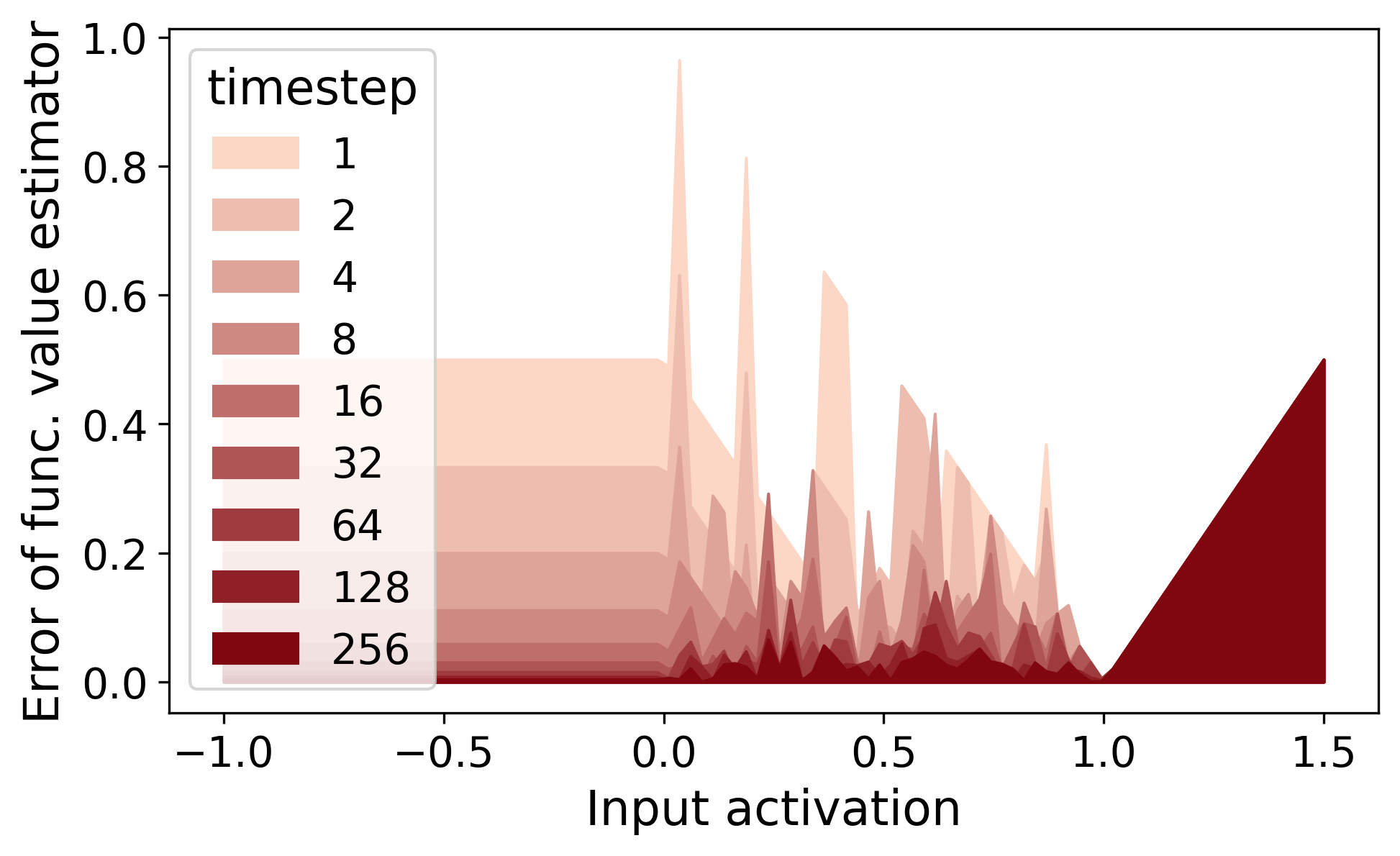}
    }
    \hfill
    \subfigure[Subgradient method on $\tilde{f}(t)$. Poisson encoding.]{\includegraphics[width=0.31\linewidth]{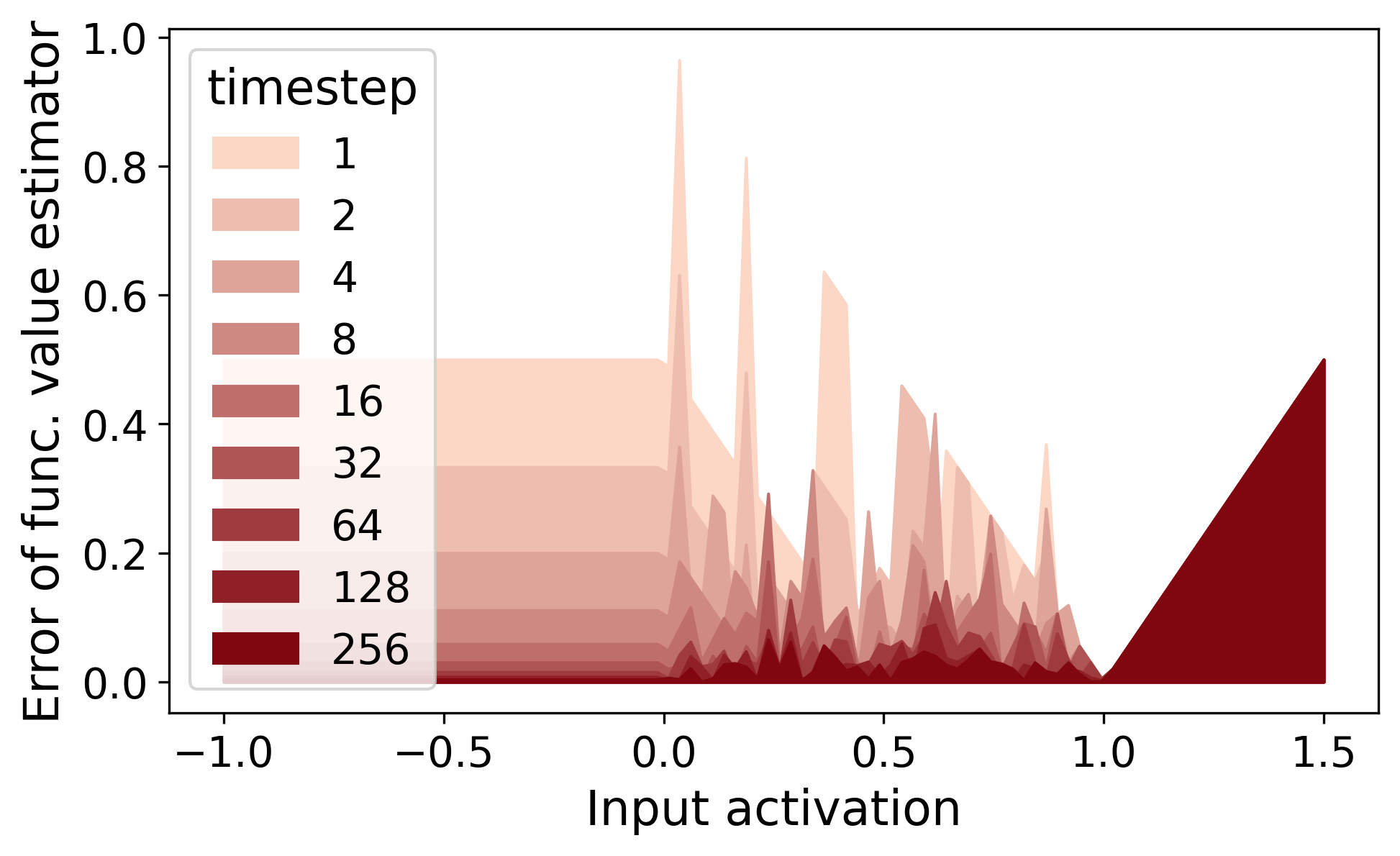}
        \label{fig:subgradient_error_poisson}
    }
\vskip -0.1in
\caption{Equivalence of IF neuron (Eq.~\ref{lif:2}-\ref{if:1}) + rate-coded input (Eq.~\ref{eq:rate}) with the subgradient method defined as Theorem~\ref{thm:if_rate}. We plot time-evolution of error between ReLU function and the rate-decoded output of IF neuron (Fig.~\ref{fig:if_rate_error},~\ref{fig:if_rate_poisson_error}) or the approximation of subgradient method (Fig.~\ref{fig:subgradient_error_float},~\ref{fig:subgradient_error_poisson}). Given input $x \in \mathbb{R}$, Float encoding is $I(t) = x~\forall t$~\cite{li2021freelunch}, and Poisson encoding is a stochastic spike train representation of $x$ that samples $I(t)\sim B(1,\text{ReLU1}(x)), \forall t$~\cite{Sengupta2018SNNVGGResNet}  }
\label{fig:toy_if_rate}
\end{center}
\vskip -0.2in
\end{figure}
\begin{figure}[!ht]
\vskip 0.0in
\begin{center}
    \subfigure[LIF neuron, $\tau = 10, R = 1, \theta_{th} = 1$, Float encoding]{
        \includegraphics[width=0.40\linewidth]{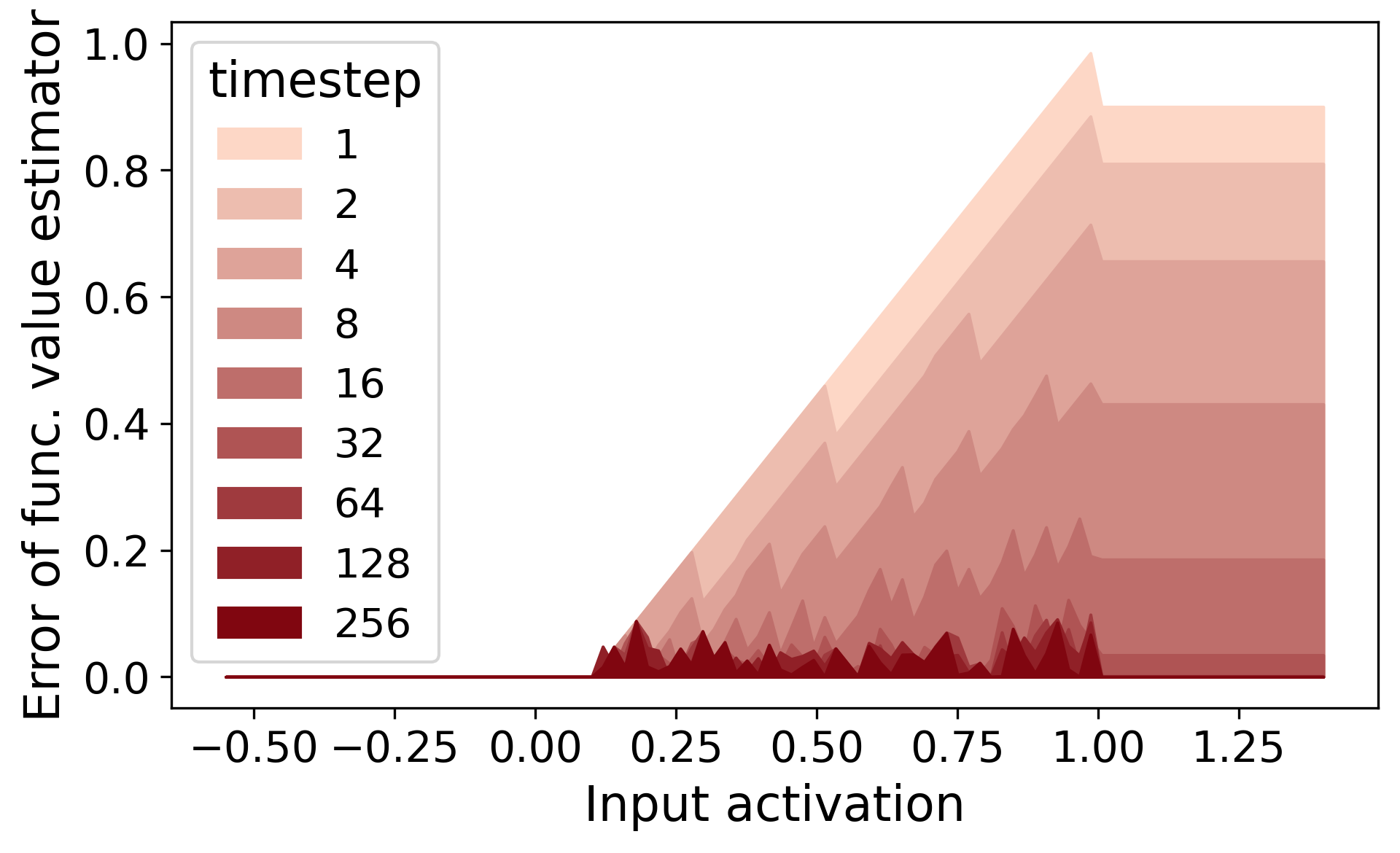}\label{fig:lif_error_float}
    }
    \hfill
    \subfigure[LIF neuron,  $\tau = 10, R = 1, \theta_{th} = 1$, spike encoding.]{
        \includegraphics[width=0.40\linewidth]{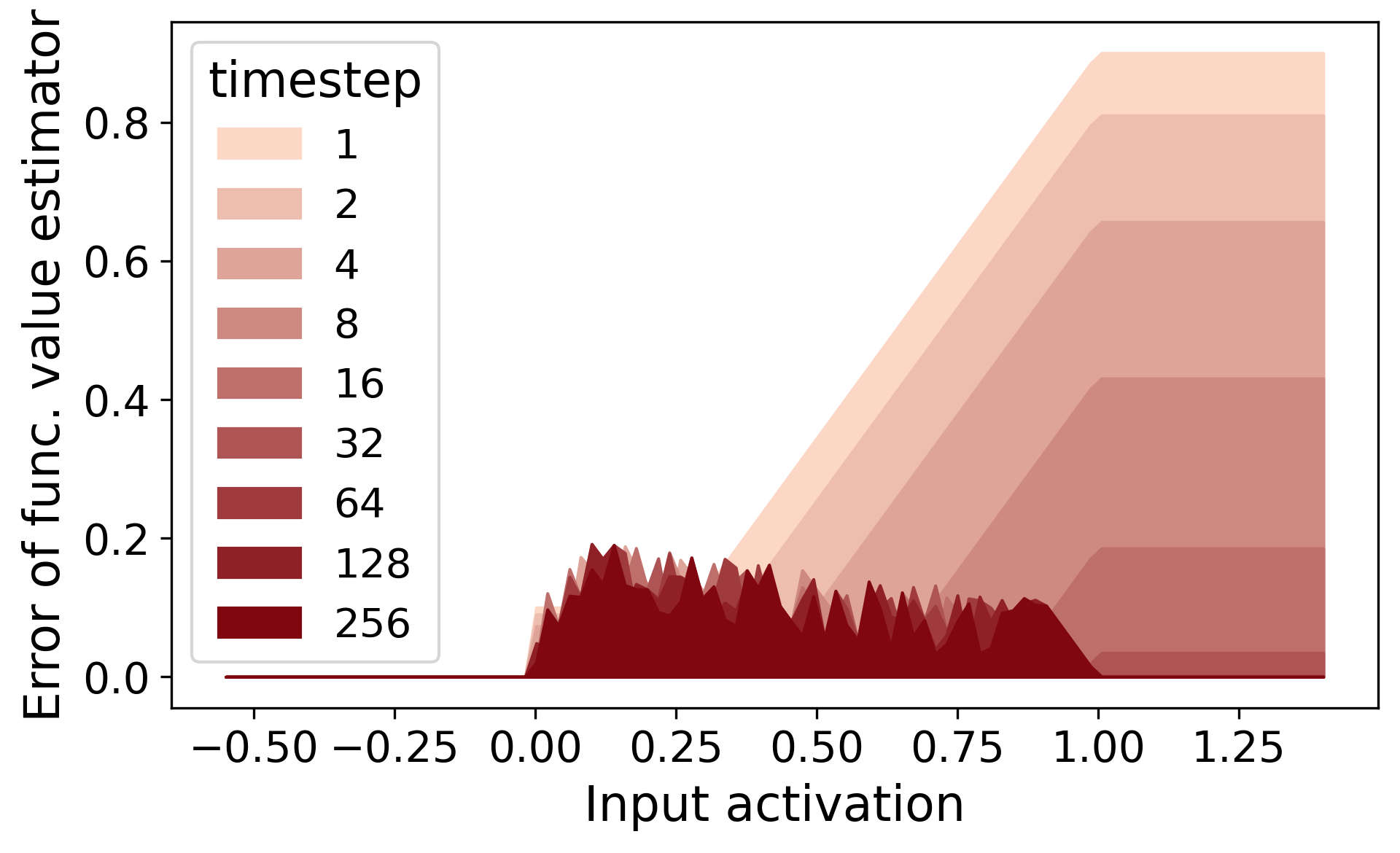}\label{fig:lif_error_spike}
    }\vfill
    \subfigure[Subgradient method, $\tau = 10, R = 1$, Float encoding.]{
        \includegraphics[width=0.40\linewidth]{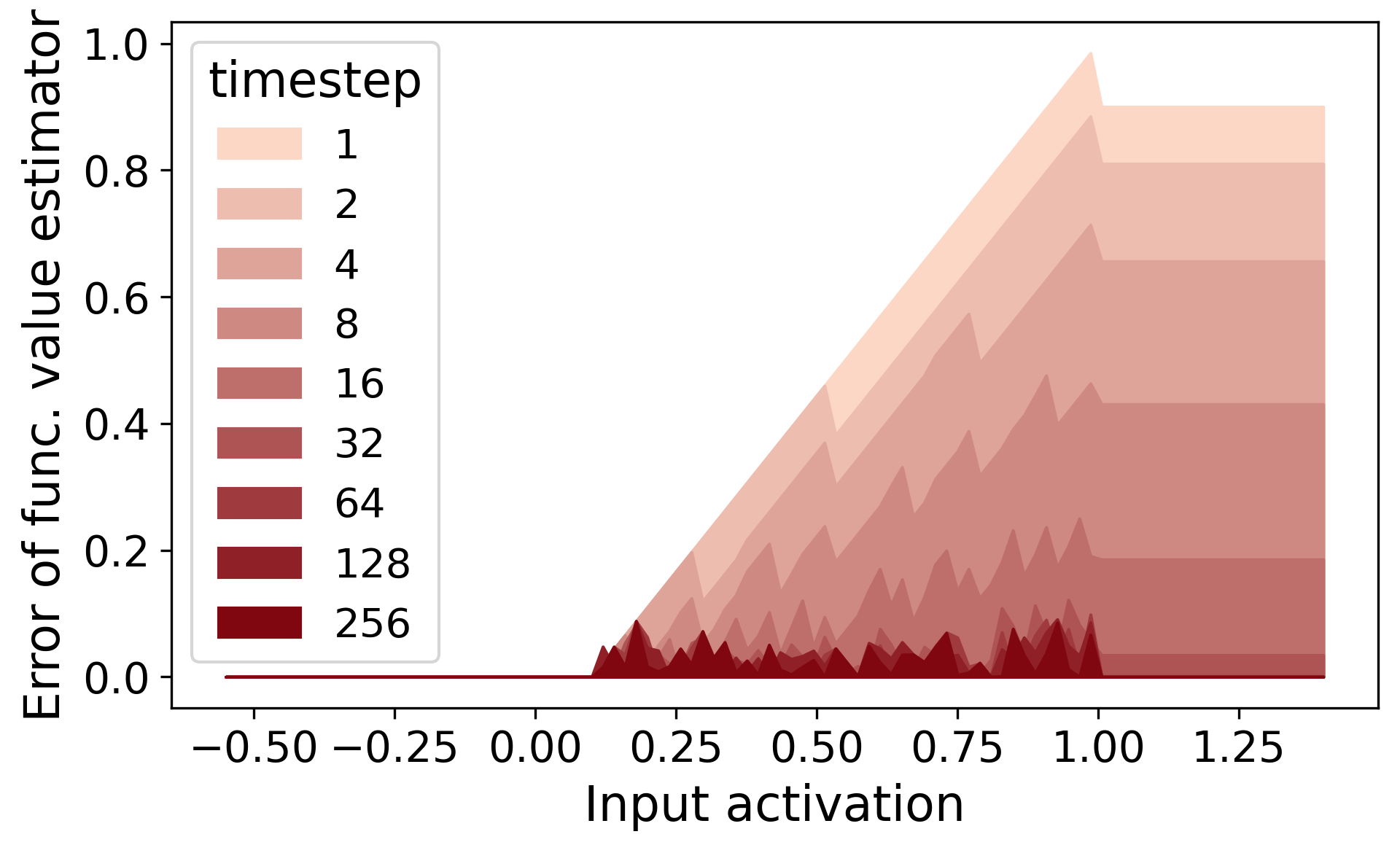}\label{fig:lif_subgradient_error_float}
    }
    \hfill
    \subfigure[Subgradient method $\tau = 10, R = 1$, spike encoding.]{
        \includegraphics[width=0.40\linewidth]{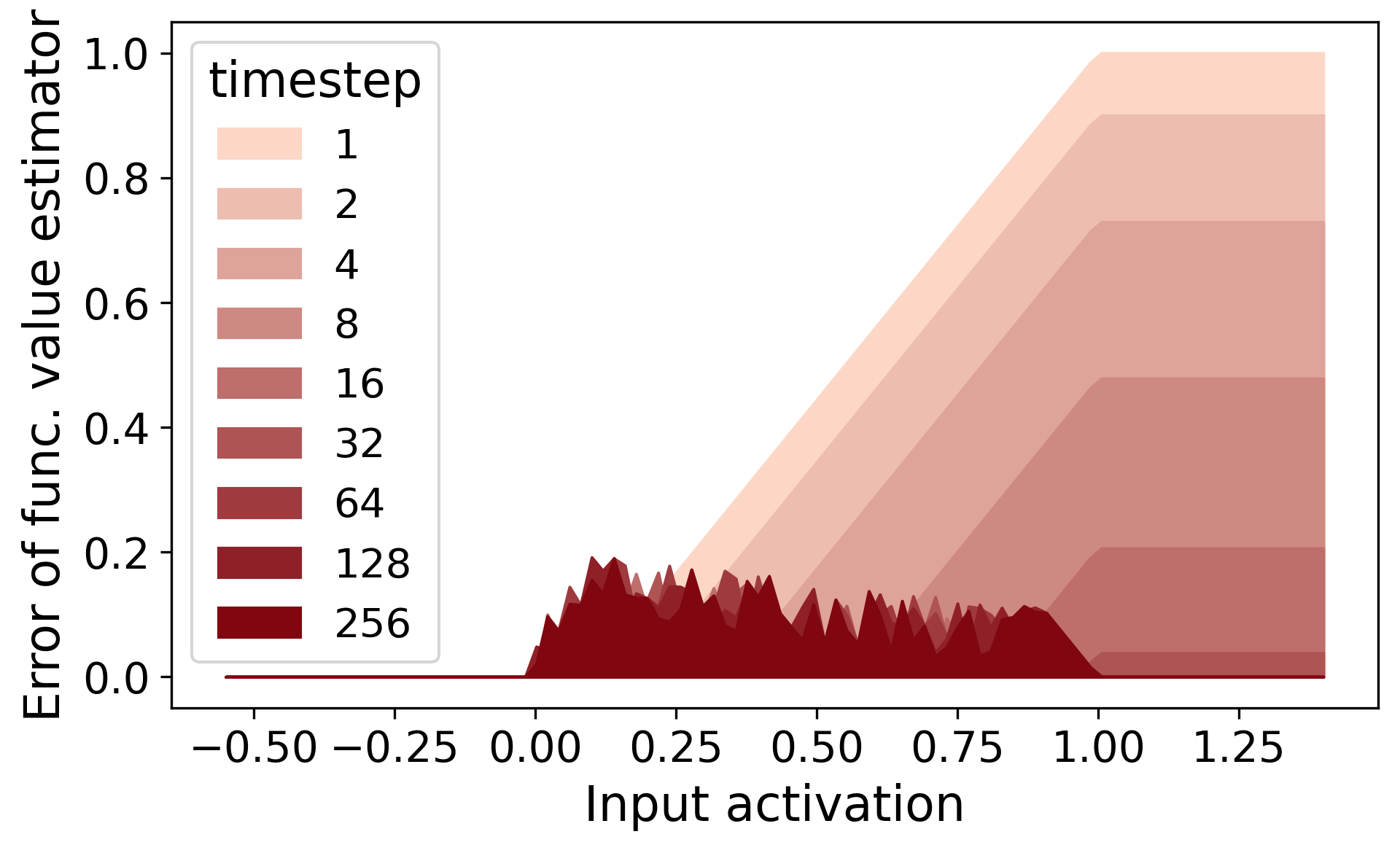}
     \label{fig:lif_subgradient_error_spike}
    }
\vskip -0.1in
\caption{Equivalence of LIF neuron (Eq.~\ref{lif:1}-\ref{lif:3}) + EMA-coded input (Eq.~\ref{def:ema_coding}) with the subgradient method defined as Theorem~\ref{thm:lif_rate}. We plot time-evolution of error between $\text{ReLU1}(\frac{Rx - \theta_{th}}{\tau - 1})$ and the EMA-decoded output of neuron (Fig.~\ref{fig:lif_error_float},~\ref{fig:lif_error_spike}) or the approximation of subgradient method (Fig.~\ref{fig:lif_subgradient_error_float},~\ref{fig:lif_subgradient_error_spike}). Given input $x \in \mathbb{R}$, Float encoding is $I(t) = x~\forall t$, and spike encoding is a deterministic spike train representation $I(t)$ of $x$ defined as $I(t)= \mathbb{H}(x - 
\frac{\tau-1}{\tau} \tilde{x}(t-1)),~\tilde{x}(t) = \frac{\tau-1}{\tau} \tilde{x}(t-1) + \frac{1}{\tau}I(t)$. }
\label{fig:toy_lif}
\end{center}
\vskip -0.3in
\end{figure}
\begin{figure}[!ht]
\vskip 0.2in
\begin{center}    
    \subfigure[SignGD-based neuron, $\eta(t) = \frac{1}{t+1}$]{
        \includegraphics[width=0.45\linewidth]{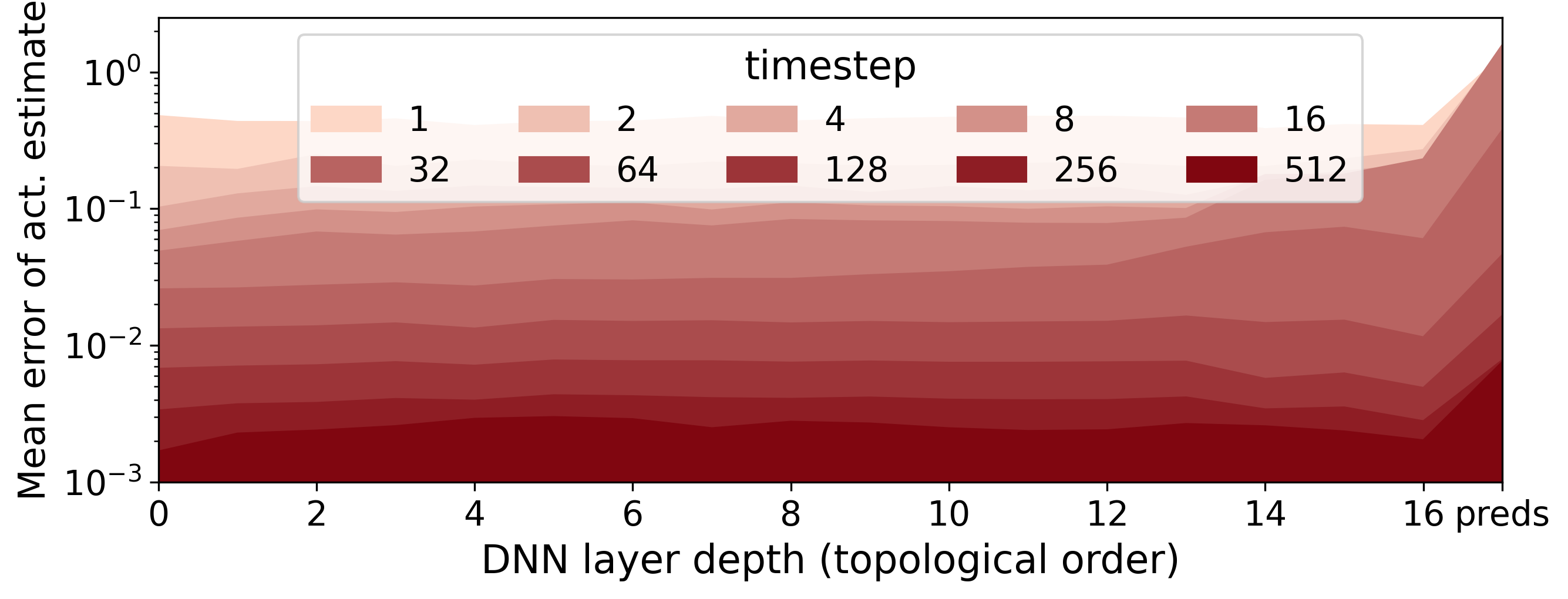}
        \label{fig:visualize_signgd_inv}
    }
    \hfill
    \subfigure[Subgradient-based neuron, $\eta(t) = \frac{1}{t+1}$ $\approx$ IF + Rate]{
        \includegraphics[width=0.45\linewidth]{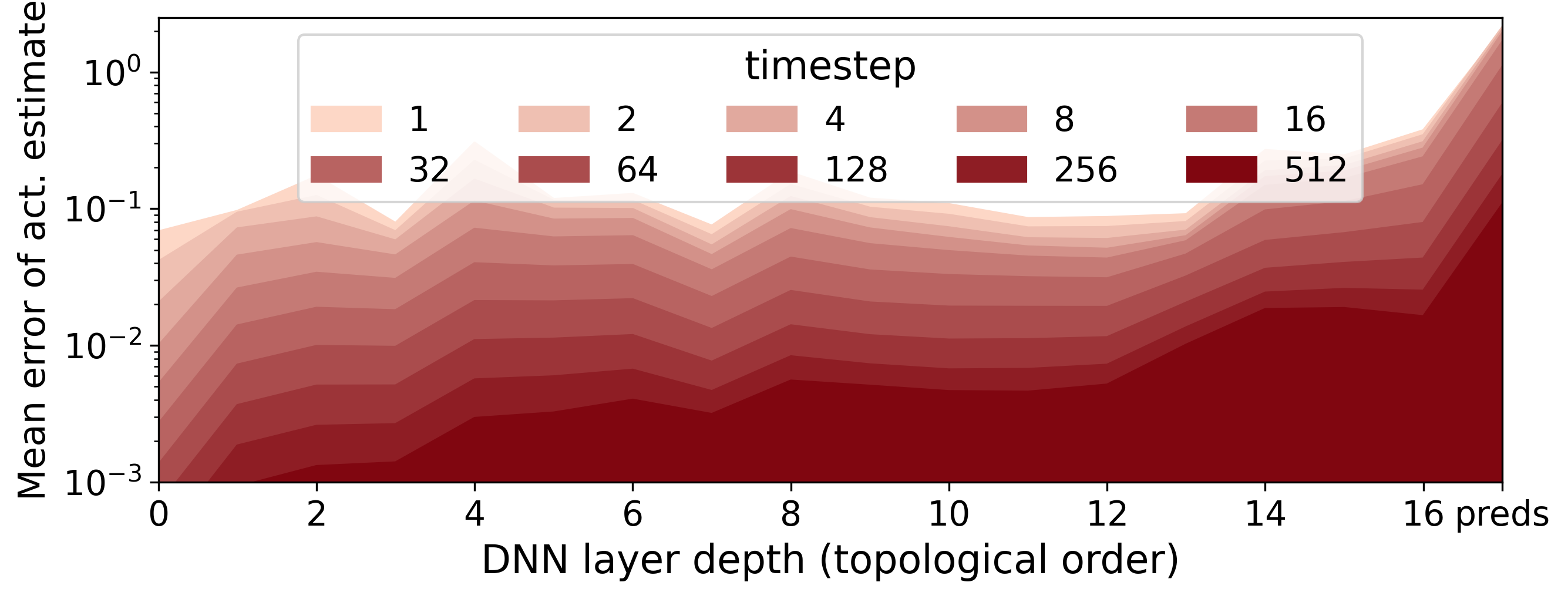}
        \label{fig:visualize_if_rate}
    }
    \vfill
    \subfigure[SignGD-based neuron, $\eta(t) = 0.15 \cdot (0.95)^t$.]{
        \includegraphics[width=0.45\linewidth]{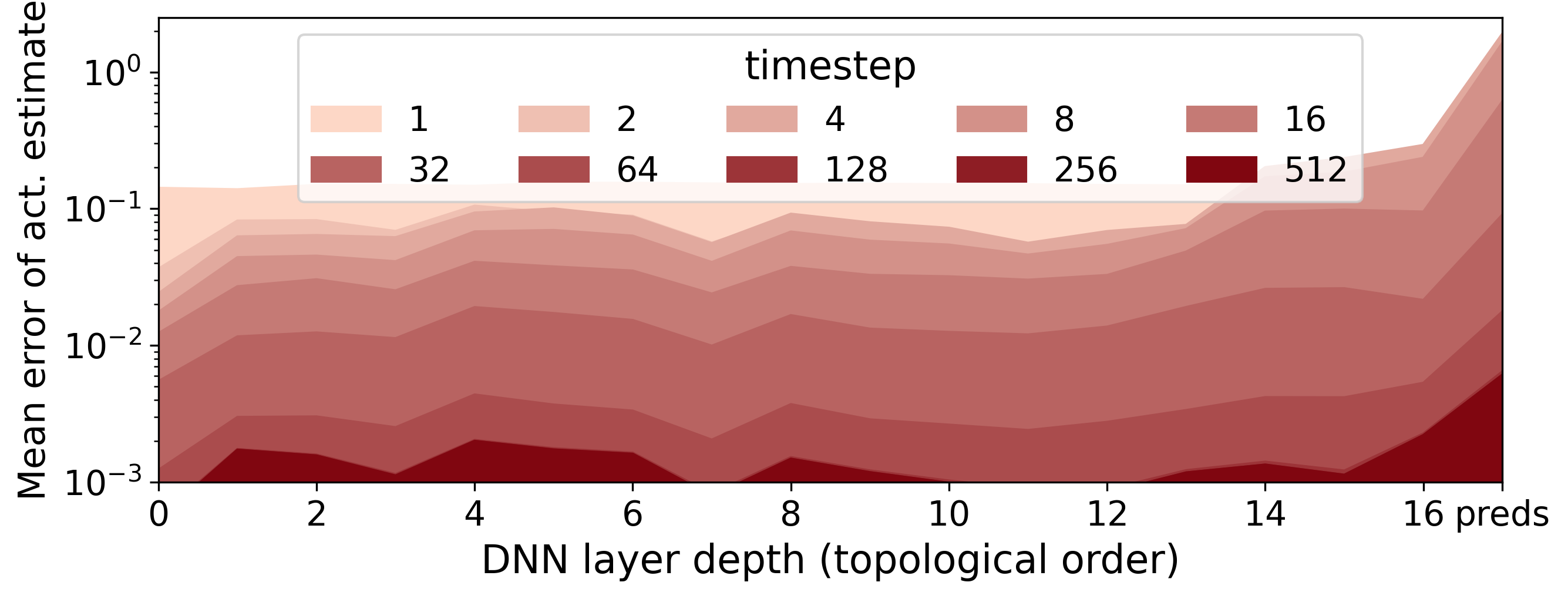}
        \label{fig:visualize_signgd_exp}
    }
    \hfill
    \subfigure[Subgradient-based neuron, $\eta(t) = 0.15 \cdot (0.95)^t$.]{
        \includegraphics[width=0.45\linewidth]{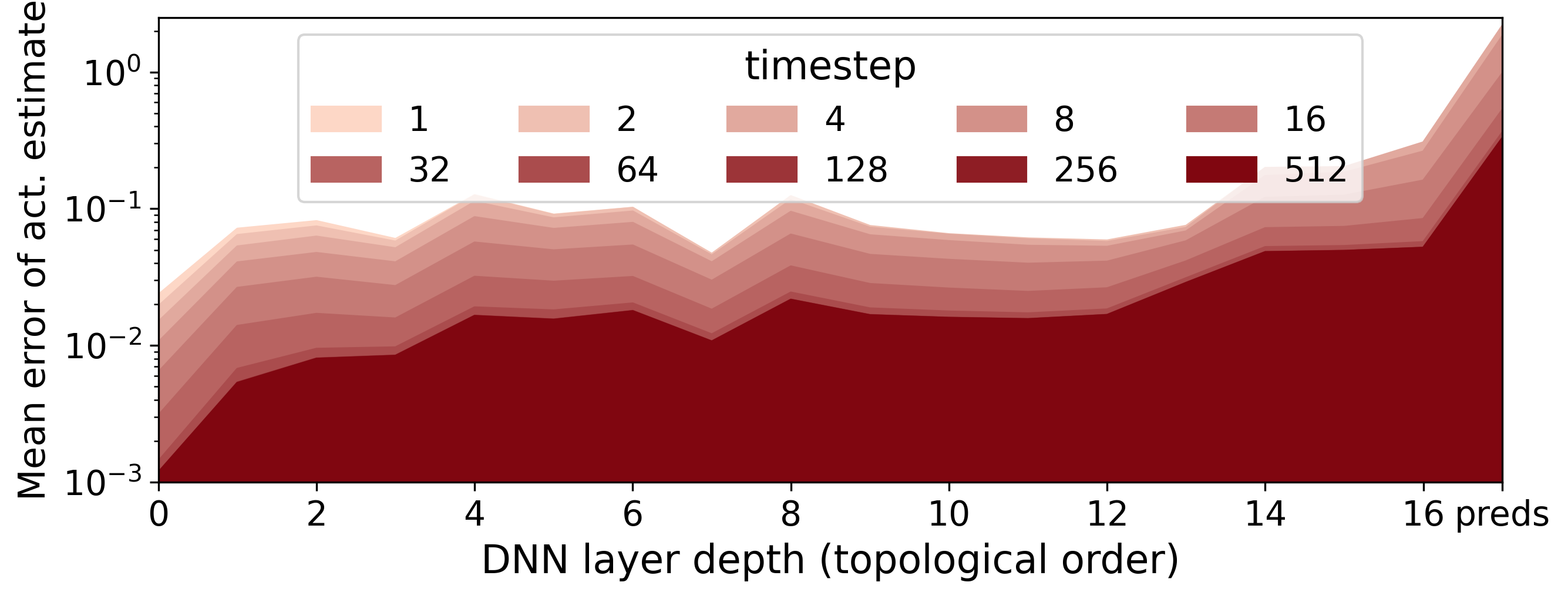}
        \label{fig:vis_exponential}
    }
\caption{Layer-wise time-evolution of error between the true ANN activation and the spike-decoded SNN activations. We qualitatively compare two neuronal dynamics side-by-side with the same learning rate schedules; inverse schedule (Fig.~\ref{fig:visualize_signgd_inv},~\ref{fig:visualize_if_rate}) and exponential schedule (Fig.~\ref{fig:visualize_signgd_exp},~\ref{fig:vis_exponential}). Measured with ResNet-18 on a single instance of CIFAR-10 dataset.}
\label{fig:error_propagation}
\end{center}
\vskip -0.2in
\end{figure}

\begin{algorithm}[tb]
   \caption{ANN-to-SNN Conversion with signGD-based Neuron (Definition~\ref{def:signgd_dynamics})}
   \label{alg:ann_to_snn_conversion}
\begin{algorithmic}[1]
   \STATE {\bfseries Input:} Target ANN model $F$, learning rate schedule $\eta: \mathbb{N} \to \mathbb{R}$, key-value mapping $\mathcal{D}$ of $\{\text{Nonlinearity}: \text{signGD-based neuronal dynamics}\}$, number of training batches for normalization $N \in \mathbb{N}$
   \STATE {\bfseries Output:} Converted SNN model $S$
   \STATE Extract the computational graph $G$ of tensor operators from ANN model $F$.
   \STATE Decompose max pooling and layer normalization operators of $G$ to generate a new graph $G'$ (See Section~\ref{sec:multioperand_nonlinearities}.).
    \IF{ReLU operator $\in$ G} 
\FOR{any ReLU operator $f \in G'$}
   \STATE Initialize the maximum ReLU output activation $M_{f} \leftarrow - \infty$.
   \ENDFOR 
   \STATE Register a callback for every ReLU operator $f$ to record its output activations $X_{f}$.
   \FOR{$t=1$ {\bfseries to} $N$}
   \STATE Sample a batch $x$ from training dataset.
   \STATE Feed-forward $x$ through ANN model $F$ and record $X_{f}$. 
   \FOR{any ReLU operator $f \in G'$}
   \STATE $M_{f} \leftarrow \max(M_{f}, X_{f})$
   \ENDFOR 
   \ENDFOR
   \FOR{any ReLU operator $f \in G'$}
   \STATE Replace the $f(\cdot)$ operator with a scaled ReLU operator $M_{f} \cdot f(\frac{\cdot}{M_{f}})$ to generate a new graph $G''$.
   \ENDFOR 
    \ELSE 
    \STATE $G'' \leftarrow G'$
    \ENDIF 
   \FOR{nonlinearity $h \in \text{domain of }\mathcal{D}$}
   \FOR{any operator $f \in G''$}
   \IF{$f$ is $h$}
   \STATE Initialize signGD-based neuron $n \leftarrow \mathcal{D}(f; \eta)$.
   \STATE Replace the operator $f$ of $G''$ with the neuron $n$ to generate a computational graph $G^S$ of SNN. 
   \ENDIF
   \ENDFOR 
   \ENDFOR 
   \STATE Create a SNN model $S$ from the computational graph $G^S$.
   \STATE Reset membrane potentials of SNN model $S$.
   \STATE Register a callback for every neuron $h$ of SNN $S$  to record its influx current $I_{h}$.
   \STATE Stimulate every neuron $h$ of SNN $S$  to emit a spike $1$ for a single time-step.
   \STATE Record the influx current $I_{h}^+ = I_{h}$ for every neuron $h$ of SNN $S$.
   \STATE Depress every neuron $h$ of SNN $S$ to not emit a spike for a single time-step.
   \STATE Record the idle current $I_{h}^- = I_{h}$ for every neuron $h$ of SNN $S$.
   \STATE Store the sum of weights $W \leftarrow I_{h}^+ - I_{h}^-$ and bias $b \leftarrow I_{h}^-$ for every neuron.
\end{algorithmic}
\end{algorithm}

\begin{table*}[t]
\caption{Comparing ANN-to-SNN conversion performance on CIFAR-10 ~\cite{krizhevsky2009cifar} dataset. 
\emph{No Spike-aware Activation Func} means ReLU functions in ANN architecture are not replaced with spike-aware functions before training, e.g., QCFS~\cite{bu2021optimalannsnn}, SlipReLU~\cite{Jiang2023SlipReLU}. For our signGD-based neuron, we used exponential schedule with initial LR $0.135$ and decay factor $0.95$. Results of RTS~\cite{Deng2021OptimalConversionTheory} are from \cite{li2021freelunch}.}
\label{tab:cifar10}
\vskip 0.15in
\begin{center}
\begin{small}
\begin{tabular}{lccccccc}
\toprule
\multirow{2}{*}{Methods} & No Spike-aware &ANN & \multicolumn{4}{c}{Simulation time-steps} \\
 & Activation Func. & Acc.& T = 16 &  T = 32 &T = 64 & T = 128 &T = 256 \\
\midrule
\multicolumn{8}{c}{ResNet-18~\cite{He2015ResNet} CIFAR-10}\\
\midrule
TSC~\cite{han2020deeptsc}    
& \Greencheck & 91.47  &- & - & 69.38 & 88.57 & 90.10  \\
RMP~\cite{Han2020RMPSNNRM}    
& \Greencheck & 91.47  &- & - & - &  87.60 & 89.37   \\
RTS~\cite{Deng2021OptimalConversionTheory}
& \Redlargex & 95.46  &- & 84.06 & 92.48 & 94.68 & 95.30  \\
SNNC-AP~\cite{li2021freelunch}      
& \Greencheck & 95.46  &- & 94.78 & 95.30 & 95.42 & 95.41   \\
SNM~\cite{Wang2022SignedNW}      
& \Greencheck & 95.39  & - & 94.03 & 94.03  & 95.19 &  -  \\
QCFS~\cite{bu2021optimalannsnn}      
&\Redlargex & 96.04  &  95.92 &  96.08  & 96.06 &  -  & - \\
SlipReLU~\cite{Jiang2023SlipReLU}  
& \Redlargex &  96.15   & \textbf{96.10}  & 96.12  & 96.22 &  -  & -  \\
SRP~\cite{Hao2023ResidualMembranePotential}       
& \Redlargex & 95.64   & 95.55  & 95.55  & 95.58 & - & - \\
Subgradient-based neuron (Thm.~\ref{thm:lr_schedule})& \Greencheck & 96.82 & 53.40 & 88.34 & 94.20 & 94.84& 94.87  \\
Ours (signGD-based neuron, Def.~\ref{def:signgd_dynamics})& \Greencheck & 96.82  & 80.74 & \textbf{96.29} & \textbf{96.78} & \textbf{96.79}& \textbf{96.79} \\
\midrule
\multicolumn{8}{c}{VGG-16~\cite{Simonyan2014VGG} CIFAR-10}\\
\midrule
TSC~\cite{han2020deeptsc}    
&  \Greencheck & 93.63  & - & - & 92.79 & 93.27 & 93.45 \\
RMP~\cite{Han2020RMPSNNRM}    
&  \Greencheck & 93.63 & -& 60.30 & 90.35 & 92.41 & 93.04   \\
RTS~\cite{Deng2021OptimalConversionTheory}
&  \Redlargex & 95.72  & - & 76.24 &  90.64 & 94.11 & 95.33 \\
SNNC-AP~\cite{li2021freelunch}      
& \Greencheck & 95.72   & - & 93.71 & 95.14 & 95.65 & 95.79   \\
SNM~\cite{Wang2022SignedNW}      
&  \Greencheck &94.09 & - & 93.43  & 94.07  & 94.07 &  -  \\
QCFS~\cite{bu2021optimalannsnn}      
&  \Redlargex & 95.52   & 95.40  & \textbf{95.54}  & 95.55  & - & -  \\
SlipReLU~\cite{Jiang2023SlipReLU}  
&  \Redlargex & 95.60   & 95.20  & 95.66  & 95.65  & - & -  \\
SRP~\cite{Hao2023ResidualMembranePotential}       
& \Redlargex &  95.52  & \textbf{95.44} &  95.42 & 95.40 & - & -\\
Subgradient-based neuron (Thm.~\ref{thm:lr_schedule})& \Greencheck & 95.96  & 50.98 & 82.84 & 90.66 & 91.75 &  91.80 \\
Ours (signGD-based neuron, Def.~\ref{def:signgd_dynamics})& \Greencheck & 95.96 & 81.06 & 95.53 & \textbf{95.96} & \textbf{95.97} & \textbf{95.97}  \\ 
\bottomrule
\end{tabular}
\end{small}
\end{center}
\vskip -0.1in
\end{table*}

\begin{table*}[t]
\caption{Comparing ANN-to-SNN conversion performance on CIFAR-100 ~\cite{krizhevsky2009cifar} dataset. 
\emph{No Spike-aware Activation Func} means ReLU functions in ANN architecture are not replaced with spike-aware functions before training, e.g., QCFS~\cite{bu2021optimalannsnn}, SlipReLU~\cite{Jiang2023SlipReLU}. For our signGD-based neuron, we used exponential schedule with initial LR $0.135$ and decay factor $0.95$. Results of RTS~\cite{Deng2021OptimalConversionTheory} are from \cite{li2021freelunch}.}
\label{tab:cifar100}
\vskip 0.15in
\begin{center}
\begin{small}
\begin{tabular}{lcccccccc}
\toprule
\multirow{2}{*}{Methods} & No Spike-aware &ANN & \multicolumn{4}{c}{Simulation time-steps} \\
 & Activation Func. & Acc.& T = 16 &  T = 32 &T = 64 & T = 128 &T = 256 \\
\midrule
\multicolumn{7}{c}{ResNet-20~\cite{He2015ResNet} CIFAR-100}\\
\midrule
TSC~\cite{han2020deeptsc}    &\Greencheck & 68.72 & - & - & - & 58.42 & 65.27  \\
RMP~\cite{Han2020RMPSNNRM}    &\Greencheck & 68.72 & - &  27.64 & 46.91 & 57.69 & 64.06   \\
RTS~\cite{Deng2021OptimalConversionTheory}
&\Redlargex & 77.16  & - &  51.27 & 70.12 & 75.81 & 77.22 \\
SNNC-AP~\cite{li2021freelunch}      
&\Greencheck & 77.16 & - &76.32 & 77.29 & 77.73 & 77.63 \\
SNM~\cite{Wang2022SignedNW}      
&\Greencheck &78.26  & -  & 74.48 & 77.59 &  77.97 &  -   \\
QCFS~\cite{bu2021optimalannsnn}      
&\Redlargex & 78.80  & \textbf{79.48} & \textbf{79.62} & 79.54 & - & 79.61 \\
SlipReLU~\cite{Jiang2023SlipReLU}  
&\Redlargex &  77.08  & 77.29 & 78.04 & 77.97 & - & 77.99 \\
SRP~\cite{Hao2023ResidualMembranePotential}       
&\Redlargex & 69.94 & 64.71 & 65.50 & 65.82 & - & - \\
Subgradient-based neuron (Thm.~\ref{thm:lr_schedule})&\Greencheck & 81.19  & 22.39 & 57.79 & 71.22 & 73.08 & 73.14 \\
Ours (signGD-based neuron, Def.~\ref{def:signgd_dynamics})&\Greencheck & 81.19  & 36.78 & 79.13 & \textbf{81.10} & \textbf{81.22} & \textbf{81.23} \\
\midrule
\multicolumn{7}{c}{VGG-16~\cite{Simonyan2014VGG} CIFAR-100}\\
\midrule
TSC~\cite{han2020deeptsc}    
&\Greencheck & 71.22 & - & - & - & 69.86 &  70.65  \\
RMP~\cite{Han2020RMPSNNRM}    
&\Greencheck & 71.22 & - & - & - & 63.76  & 68.34  \\
RTS~\cite{Deng2021OptimalConversionTheory}
&\Redlargex & 77.89 & - & 7.64  & 21.84  & 55.04  & 73.54  \\
SNNC-AP~\cite{li2021freelunch}      
&\Greencheck & 77.89  & - & 73.55  & 76.64  & 77.40 &  77.68    \\
SNM~\cite{Wang2022SignedNW}      
&\Greencheck &74.13  & -  & 71.8 & 73.69 & 73.95 &  -  \\
QCFS~\cite{bu2021optimalannsnn}      
&\Redlargex & 76.28   & 76.24  & \textbf{77.01}  & 77.10  &  -  &  77.08\\
SlipReLU~\cite{Jiang2023SlipReLU}  
&\Redlargex & 70.03  & 69.35 & 70.65  & 71.23  & - & - \\
SRP~\cite{Hao2023ResidualMembranePotential}       
&\Redlargex & 76.28  & \textbf{76.42} & 76.45 & 76.37 & - & -\\
Subgradient-based neuron (Thm.~\ref{thm:lr_schedule})&\Greencheck & 78.28  & 13.12 & 42.05 & 60.61 & 64.03 & 64.15 \\
Ours (signGD-based neuron, Def.~\ref{def:signgd_dynamics})&\Greencheck & 78.28  & 39.42 & 76.33 & \textbf{78.17} & \textbf{78.33} & \textbf{78.23} \\
\bottomrule
\end{tabular}
\end{small}
\end{center}
\vskip -0.1in
\end{table*}


\end{document}